%% file: main_arxiv.tex
\pdfoutput=1

\documentclass{article}[11pt]
\usepackage{fullpage}

\usepackage[utf8]{inputenc}
\usepackage{hyperref}
\usepackage{amsfonts}
\usepackage{amsmath}
\usepackage{amssymb}
\usepackage{amsthm}
\usepackage{qtree}
\usepackage{graphicx}
\usepackage{booktabs}
\usepackage{subcaption}
\usepackage[section]{placeins}
\graphicspath{ {./figures/} }
\usepackage{xcolor}
\usepackage[english]{babel}
\usepackage[ruled,vlined]{algorithm2e}
\usepackage{bbm}
\usepackage{footmisc}

\definecolor{DarkGreen}{rgb}{0.1,0.5,0.1}
\definecolor{DarkRed}{rgb}{0.5,0.1,0.1}
\definecolor{DarkBlue}{rgb}{0.1,0.1,0.5}
\hypersetup{
    unicode=false,          
    pdftoolbar=true,        
    pdfmenubar=true,        
    pdffitwindow=false,     
    pdfnewwindow=true,      
    colorlinks=true,       
    linkcolor=DarkGreen,    
    citecolor=DarkGreen,    
    filecolor=DarkGreen,    
    urlcolor=DarkBlue,      
}

\usepackage{natbib}
\setcitestyle{authoryear,round,citesep={;},aysep={,},yysep={;}}

\input{math_commands}

\input{macros}

\title{Contrasting the landscape of contrastive and non-contrastive learning} 

\author{
    Ashwini Pokle\footnotemark[1] \\
    Carnegie Mellon University \\
    \texttt{apokle@andrew.cmu.edu} \\
    \and
    Jinjin Tian\footnotemark[1] \\
    Carnegie Mellon University \\
    \texttt{jinjint@andrew.cmu.edu} \\
    \and
    Yuchen Li\footnotemark[1] \\
    Carnegie Mellon University \\
    \texttt{yuchenl4@andrew.cmu.edu} \\
    \and
    Andrej Risteski \\
    Carnegie Mellon University \\
    \texttt{aristesk@andrew.cmu.edu} \\}

\date{*Equal contribution}

\begin{document}

\maketitle

\input{sections/abstract}

\input{sections/introduction}

\pagestyle{plain}

\input{sections/overview}

\input{sections/related-work}

\input{sections/problem-setup}

\input{sections/empirical-setup}

\input{sections/results}

\input{sections/conclusion}

\subsubsection*{ACKNOWLEDGEMENTS}
The authors would like to thank Bingbin Liu for helpful comments and a careful proofread. 

\bibliography{references}
\bibliographystyle{ref_style}


\clearpage
\appendix

\thispagestyle{empty}

\def\toptitlebar{
\hrule height4pt
\vskip .25in}

\def\bottomtitlebar{
\vskip .25in
\hrule height1pt
\vskip .25in}

\newcommand{\makesupplementtitle}{\hsize\textwidth
    \linewidth\hsize \toptitlebar {\centering
        {\Large\bfseries Supplementary Material \par}}
    \bottomtitlebar}

\makesupplementtitle


\renewcommand{\theequation}{\thesection.\arabic{equation}}

\input{sections/appendices/landscape_analysis}
\input{sections/appendices/training_dynamics}

\end{document}

%% file: math_commands.tex
\pdfoutput=1

\usepackage{amsmath}
\usepackage{amsfonts}
\usepackage{bm}
\usepackage{amsthm}

\newcommand{\R}{{\mathbb{R}}}
\newcommand{\N}{{\cal{N}}}
\newcommand{\bx}{{\bf{x}}}

\newcommand{\bz}{{\bf{z}}}

\newcommand{\bM}{{\bf{M}}}

\newcommand{\bI}{{\bf{I}}}
\newcommand{\bW}{{\bf{W}}}

\newcommand{\ba}{{\bf{a}}}
\newcommand{\bb}{{\bf{b}}}
\newcommand{\Z}{{\cal{Z}}}

\newtheorem*{theorem*}{Theorem}
\newtheorem{theorem}{Theorem}

\newtheorem{lemma}{Lemma}

\newtheorem{assumption}{Assumption}
\newtheorem*{remark*}{Remark}
\newtheorem{remark}{Remark}

\makeatletter

\newcommand{\argmin}{\mathop{\mathrm{argmin}}}

\def\1{\bm{1}}

\def\vb{{\bm{b}}}

\def\vv{{\bm{v}}}

\def\vx{{\bm{x}}}
\def\vy{{\bm{y}}}
\def\vz{{\bm{z}}}

\def\evb{{b}}

\def\evx{{x}}

\def\evz{{z}}

\def\mD{{\bm{D}}}

\def\mI{{\bm{I}}}

\def\mM{{\bm{M}}}

\def\mP{{\bm{P}}}

\def\mW{{\bm{W}}}

\def\emD{{D}}

\def\emI{{I}}

\def\emW{{W}}

%% file: macros.tex
\newif\ifcomments

\commentstrue

\ifcomments
  \newcommand{\colornote}[3]{{\color{#1}\bf{#2: #3}\normalfont}}
\else
  \newcommand{\colornote}[3]{}
\fi


%% file: sections/abstract.tex
\begin{abstract}
A lot of recent advances in unsupervised feature learning are based on designing features which are invariant under semantic data augmentations. A common way to do this is \emph{contrastive learning}, which uses positive and negative samples. Some recent works however have shown promising results for \emph{non-contrastive learning}, which does not require negative samples. 
 However, the non-contrastive losses have obvious ``collapsed'' minima, in which the encoders output a constant feature embedding, independent of the input. A folk conjecture is that so long as these collapsed solutions are avoided, the produced feature representations should be good. In our paper, we cast doubt on this story: we show through theoretical results and controlled experiments that even on simple data models, non-contrastive losses have a preponderance of \emph{non-collapsed} bad minima. Moreover, we show that the training process does not avoid these minima.\footnote{Code for this work can be found at \url{https://github.com/ashwinipokle/contrastive\_landscape}}
\end{abstract}

%% file: sections/introduction.tex
\section{Introduction}

Recent improvements in representation learning without supervision were driven by self-supervised learning approaches, in particular contrastive learning (CL), which constructs positive and negative samples out of unlabeled dataset via data augmentation \citep{chen2020simple, he2020momentum, caron2020unsupervised, ye2019unsupervised, oord2018representation, wu2018unsupervised}. 
Subsequent works based on data augmentation also showed promising results for methods based on non-contrastive learning (non-CL), which do not require explicit negative samples  \citep{grill2020bootstrap, richemond2020byol, chen2021exploring, zbontar2021barlow, tian2021understanding}.

However, understanding of how these approaches work, especially of how the learned representations compare---qualitatively and quantitatively---is lagging behind. 
In this paper, via a combination of empirical and theoretical results, we provide evidence that non-contrastive methods based on data augmentation can lead to substantially worse representations. 

 Most notably, avoiding the collapsed representations has been the key ingredient in prior successes in non-contrastive learning. The collapses was first referred as the \emph{complete collapse}, that is, all representation vectors shrink into a single point; later a new type of collapses, \emph{dimension collapse} \citep{hua2021feature,jing2022understanding} caught attention as well, that is the embedding vectors only span a lower-dimensional subspace. It is naturally to conjunct that avoiding those kinds of collapsed representations suffices for non-contrastive learning to succeed. We provide strong evidence in contrast to this: namely, we show that even under a very simple, but natural data model, the non-contrastive loss has a prevalence of bad optima that are not collapsed (in neither way, complete or dimension collapse). Moreover, we supplement this with a training dynamics analysis: we prove that the training dynamics can remedy this situation---however, this is crucially tied to a careful choice of a predictor network in the model architecture.

Our methodology is largely a departure from other works on self-supervised learning in general: rather than comparing the performance of different feature learning methods on specific datasets, we provide extensive theoretical and experimental results for \emph{synthetic} data generated by a natural data-generative process in which there is a ``ground-truth'' representation. The more structured form of the data allows us to make much more precise statements on the quality of the representations produced by the methods we consider.

%% file: sections/overview.tex
\section{Overview of Results} \label{sec:overview}

We will study two aspects of contrastive learning (CL) and non-contrastive learning (non-CL): the set of optima of the objective, and the gradient descent training dynamics. 

In order to have a well-defined notion of a ``ground truth'' we will consider data generated by a sparse-coding inspired model \cite{olshausen1997sparse, arora2015simple, wen2021toward}:  
$$\bm{x} = \bm{M}\bm{z} + \bm{\epsilon}, \bm{\epsilon} \sim  \mathcal{N}(0, \sigma_0^2 \mathbf{I}_p).$$ 

Moreover, the encoder(s) used in the various settings will be simple one-layer ReLU networks. The motivation comes from classical theoretical work on sparse coding \cite{arora2015simple}, which shows that such an encoder can recover a representation $\tilde{z}$ which has the same support as the $z$.

\subsection{Landscape analysis}

The first aspect we will turn to is an analysis of the minima for both contrastive and non-contrastive losses. More precisely, we will show that even for simple special cases of the sparse coding data generative model, non-contrastive losses exhibit abundant  \emph{non-collapsed bad global optima}, whereas the contrastive loss is minimized \emph{only at the ground truth solution}:

\begin{theorem*}[informal, non-collapse bad global optima]
On a noise-less sparse coding data distribution augmented by random masking,
with encoders given by single-layer ReLU networks, 
\begin{itemize} 
    \item The non-contrastive loss function has infinitely many non-collapsed global optima that are far away from the ground truth.
    \item Any global optima of the contrastive loss function is (up to a permutation) equal to the ground truth solution. 
\end{itemize}
\end{theorem*}

In this case, the ground truth refers to learning the features $z$ underlying the sparse coding data distribution, which is formally defined in Section~\ref{sec:setup}.
For the formal theorem statement, see Section~\ref{sec:land}.
The proof is deferred to Appendix \ref{sec:appendix-landscape}.

\subsection{Training dynamics analysis} 

In the previous section, we have shown that bad optima exists for non-contrastive loss, the next reasonable question to ask is whether the training dynamics of the algorithm is capable of circumventing them. 
We answer this question using both theoretical analyses and experimental results.

Theoretically, we show that for non-contrastive loss, under certain conditions that we will specify,
\begin{itemize}
    \item (Negative result) A one-layer linear dual network (\eqref{eq:linear_network}) optimized via gradient descent cannot improve beyond a linear combination of its initialization. (Theorem~\ref{thm:linear_network} and Corollary~\ref{remark:linear_network})
    \item (Positive result) However, a one-layer ReLU network (\eqref{eq:relu_network}) provably converges to the groundtruth using alternating optimization with warm-start initialization and appropriate normalization. (Theorem~\ref{thm:relu_no_pred})
\end{itemize}
These results show that the results of non-contrastive learning is very sensitive to the interaction between training algorithm, non-linearity, and normalization scheme.

\begin{remark}
    In comparison, for contrastive loss, \cite{wen2021toward} proves that training dynamics learns a good representation under a similar sparse-coding inspired data generating process, even without the warm-start assumption.
\end{remark}

Moreover, we provide detailed empirical evidence in support of the following conclusions: 
\begin{enumerate}
    \item The gradient descent dynamics on the contrastive loss is capable of recovering the ground truth representation $z$. These results are robust to changes in architecture of the encoder, as well as the parameters for the data generative model. 
    \item The gradient descent dynamics on the two most popular non-contrastive losses, SimSiam and BYOL \citep{grill2020bootstrap, chen2021exploring}, are not able to avoid the poor minima, if the architecture does not include a linear predictor on top of the encoder(s). However, if a predictor is included, both SimSiam and BYOL tend to converge to a solution  close to the ground truth. We also find that this predictor is optional if weights of the encoder are initialized close to the ground truth minimum, and these weights are row-normalized or column normalized after every gradient descent update. 
\end{enumerate}

%% file: sections/related-work.tex
\section{Related Work}

\paragraph{Contrastive learning (CL)} The idea of learning representations so that a pair of similar samples (also known as "positive pairs") are closer and dissimilar pairs (known as "negative samples") are farther is widespread both in NLP (\cite{gao2021simcse, giorgi2020declutr}) and Vision (\cite{chen2020simple, he2020momentum, caron2020unsupervised, ye2019unsupervised, oord2018representation, wu2018unsupervised, tian2020contrastive, li2020prototypical, henaff2020data}). One of the earliest works based on this principle was proposed by \cite{hadsell2006dimensionality} which used Euclidean distance based contrastive loss. Recently, normalized temperature-scaled cross entropy loss or \emph{NT-Xent} loss (\cite{chen2020simple, wu2018unsupervised, sohn2016improved}) has gained more popularity. \cite{wang2020understanding} suggest that CL should optimize for both the alignment and uniformity features on the hypersphere. Most of these CL approaches often require additional tricks like maintaining a large memory bank (\cite{wu2018unsupervised, misra2020self}), momentum encoder  (\cite{he2020momentum}) or using a large batch sizes (\cite{chen2020simple}) to learn useful representations and prevent collapse, which makes them computationally intensive. There has been some recent work to overcome these drawbacks.  \cite{caron2020unsupervised} propose contrasting soft cluster assignments for augmented views of the image instead of directly comparing features which helps them to avoid most of these tricks.

\paragraph{Non-contrastive learning (Non-CL)}
Another line of proposed SSL-approaches, such as BYOL \citep{grill2020bootstrap}, question whether these negative examples are indispensable to prevent collapsing while preserving high performance. The authors instead propose a framework that outperforms previous state-of-the-art approaches (\cite{chen2020simple}, \cite{he2020momentum}, \cite{tian2020contrastive}) without any use of negative examples.
To understand how non-contrastive learning works, existing works mostly focus on analyzing what elements help it avoid learning collapsed representations \citep{grill2020bootstrap, richemond2020byol, chen2021exploring, zbontar2021barlow, tian2021understanding,wang2021towards}. 
Our work contributes to these prior efforts by pointing out another subtle difference between contrastive and non-contrastive learning, which has not been formalized in prior theoretical or empirical works. Specifically, as we show both theoretically and empirically, the non-contrastive loss has abundant \emph{non-collapse bad global optima}.

\paragraph{Theoretical understanding}
Methodologically, our work is closely related to a long line of works aiming to understand algorithmic behaviors by analyzing a simple model in controlled settings. 
Our model architecture gets inspirations from \cite{tian2021understanding} and \cite{wen2021toward}. Namely, \cite{tian2021understanding} adopts a linear network and calculates the optimization dynamics of non-contrastive learning, while \cite{wen2021toward} analyzes the feature learning process of contrastive learning on a single-layer linear model with ReLU activation. Another relevant work by \cite{tian2020understanding} calculates the optimization dynamics for a more complex N-layered dual-deep linear networks with ReLU activation for SimCLR architecture and different variants of contrastive loss like InfoNCE, soft-triplet loss etc. In a similar spirit, we base our comparison of contrastive and non-contrastive learning on single-layer dual networks, but instead of discussing the optimization process, we focus on the final features learned by these different training approaches. 
Our work is also related to \cite{arora2019theoretical} in which the authors assess the ability of self-supervised learning methods to encode the latent structure of the data. 
To theoretically evaluate the learned features, we use a variant of the sparse coding data generating model presented in \cite{olshausen1997sparse, lewicki2000learning, arora2015simple, wen2021toward}, so that the ground truth dictionary defines the evaluation metrics that characterize the quality of features learned by our simple model. Namely, in order for a learned model to represent the latent structure of the data, each feature in the ground truth dictionary is expected to be picked up by a set of learned neurons. \citep{ma2018noise,zimmermann2021contrastive} are on the related topics as well but use different data models or setups.

%% file: sections/problem-setup.tex
\section{Problem Setup}
\label{sec:setup}

\paragraph{Notation} For a matrix $\bm{M} \in \mathcal{R}^{m \times n}$, we use $\bm{M}_{i*}$ to represent its $i$-th row and $\bm{M}_{*j}$ to represent its $j$-th column. 
We denote $[d]:=\{1, \hdots, d\}$. We use the notation \textit{poly(d)} to represent a polynomial in $d$. We denote the set of matrices with unit column norm as $\mathcal{U}$ --- we will omit the implied dimensions if clear from context. The subset of $\mathcal{U}$ with positive entries and non-negative entries is denoted by $\mathcal{U}_{>}$ and $\mathcal{U}_{\geq}$. We use $\mathbb{O}$ to denote the set of orthogonal matrices (again, we will omit the dimension if clear from context), and the subset of orthogonal matrices with positive entries and non-negative entries is denoted by $\mathbb{O}_>$ and $\mathbb{O}_\geq$.

\subsection{A sparse coding model set up} \label{sec:setup:data-generation}
\paragraph{Data generating process} The data samples $\bm{x}$ in our model are generated i.i.d. through a sparse coding generative model as:
\begin{align}
    \label{eq:data-generation}
    \bm{x} = \bm{M}\bm{z} + \bm{\epsilon}, \quad 
    \bm{\epsilon} \sim  \mathcal{N}(0, \sigma_0^2 \mathbf{I}_p)
\end{align}
where, the \textit{sparse} latent variable $\bz = (\bm{z}_1, \hdots \bm{z}_d)$ and $\bm{\epsilon} \in \mathbb{R}^{p \times 1}$. We assume that  $\bz_1,\dots,\bz_d$ are i.i.d. with $\beta:=Pr(\bz_i \neq 0) \in (0, 1)$ for all $i \in [d]$. 
We consider two cases, namely the case of \emph{zero-one} latent ($\bz_i \in \{0, 1\}$) and the case of \emph{symmetric} latent ($\bz_i \in \{-1, 0, 1\}$).
The \textit{dictionary matrix} $\bM \in \mathbb{O}^{p \times d}$, is a column-orthonormal matrix. Further, we assume $p = poly(d)$. We note that the noise $\bm{\epsilon}$ is optional in our setting and there are some theoretical results in section \ref{sec:all-results} that assume that $\bm{\epsilon} = 0$. All empirical results assume a small Gaussian noise with $\sigma_0 = \Theta(\frac{\log d}{d})$. 

\paragraph{Augmentation} \label{sec:augmentation}
We augment an input sample $\bm{x}$ through random masking. These random masks are generated as follows: Consider two independent diagonal matrices $\bm{D_1}, \bm{D_2}  \in \mathbb{R}^{p \times p}$ with $\{0, 1\}$  as their diagonal entries. Each of the diagonal entries are sampled i.i.d from $Bernoulli(\alpha)$. For an input $\bm{x}$, the  augmented views $\ba_1, \ba_2 \in \mathbb{R}^{p\times 1}$ are generated as \footnote{Prior work \citep{wen2021toward} motivates another masking scheme which we refer to as ``dependent masking" in which the diagonal entries of a diagonal matrix $\bm{D}$ are sampled i.i.d. from from $Bernoulli(\alpha)$. The augmentation views are computed as \begin{equation}\bm{a}_1 := 2\bm{D} \bm{x}; \; \; \bm{a}_2 := 2(\bm{I} - \bm{D})\bm{x} \label{eq:dependent-masking}\end{equation}}:

\begin{equation}
    \label{eq:augmentation}
    \ba_1 := \bm{D_1} \bx; \;\; \ba_2 := \bm{D_2} \bx
\end{equation}

\paragraph{Network architecture}
We use a dual network architecture inline with prior work \cite{arora2015simple, tian2021understanding, wen2021toward}. 
In contrastive learning setting, we assume that the encoder shares weights between the two input views. However, in non-contrastive learning setting, we assume two separate independently initialized networks, called online network and target network. Both the networks are a single-layer neural network with activation function $\phi$ (we use ReLU activation or its symmetrized version, for the cases of zero-one latent $\vz$ or symmetric latent $\vz$ in \eqref{eq:data-generation}, respectively), namely:
\begin{equation*}
    h_{\bm{W}, \bm{b}}(\bm{a}) := \phi(\bm{W}\bm{a} + \bm{b}) \in \mathbb{R}^{m \times 1}.
\end{equation*}
For brevity, we will skip the bias $\bm{b}$ in the notation and instead refer to  $h_{\bm{W}, \bm{b}}(\bm{a})$ simply as  $h_{\bm{W}}(\bm{a})$.

In part of the theoretical analysis in Section~\ref{sec:dyn:theory}, we further simplify this architecture to dual \emph{linear} network
\begin{equation*}
    h_{\bm{W}}^{linear}(\bm{a}) := \bm{W}\bm{a} \in \mathbb{R}^{m \times 1}.
\end{equation*}

For results in Section~\ref{sec:dyn:experimental}, in addition to the linear encoder, we assume a linear prediction head $\bm{W}^p \in \mathbb{R}^{m \times m}$ which transforms the output of one of the encoders to match it to the output of the other encoder.
\begin{equation*}
    g_{\bm{W}^p, \bm{b}^p, \bm{W}, \bm{b}}(\bm{a}) := \bm{W}^p h_{\bm{W}, \bm{b}}(\bm{a)} + \bm{b}^p \in \mathbb{R}^{m \times 1}.
\end{equation*}

For brevity, we will again skip the biases $\bm{b}$ and $\bm{b}^p$ in the notation and refer to  $g_{\bm{W}^p, \bm{b}^p, \bm{W}, \bm{b}}(\bm{a})$ simply as  $g_{\bm{W}^p, \bm{W}}(\bm{a})$.

\paragraph{Non-contrastive learning algorithm and loss function}
Our non-contrastive learning algorithm is motivated from the recent works in non-contrastive SSL methods  like BYOL (\cite{grill2020bootstrap}) and SimSiam (\cite{chen2021exploring}). 
Given an input data sample $\bm{x}$, we first augment it through random masking to obtain two different views of the input $\bm{x}$. We then extract representation vectors of these views with two different randomly initialized encoders (described above), called the online network and the target network, with weights $\bm{W}^o \in \mathbb{R}^{m\times p}$ and $\bm{W}^t \in \mathbb{R}^{m\times p}$, respectively. The output of these encoders are normalized and loss on the resulting representations is computed as: 
\begin{align}
    \label{eq:loss_non-CL-l2}
    &L_{\text{non-CL-l2}}(\bW^o, \bW^t) := 
  \mathbb{E}\left\Vert \dfrac{\phi(h_{\bm{W}^o, \bb^o}(\bm{D}_1 \bm{x}))}{ \| \phi(h_{\bm{W}^o, \bb^o}(\bm{D}_1 \bm{x}))\|_2} -  SG\left(\dfrac{\phi(h_{\bm{W}^t,\bb^t}(\bm{D}_2 \bm{x}))}{\|\phi(h_{\bm{W}^t, \bb^t}(\bm{D}_2 \bm{x}))\|_2}\right)\right \Vert_2^2 
\end{align} 
where, $\phi(\cdot)$ refers to the activation function, and the stop gradient operator $SG$ indicates that we do not compute the gradients of the operand during optimization.  
In our experiments, we use ReLU activation, and during the training process, we alternate between $L_{\text{non-CL-l2}}(\bW^o, \bW^t)$ and $L_{\text{non-CL-l2}}(\bW^t, \bW^o)$ to update the online and target networks. 
This alternating optimization can be viewed as a simple approximation of the combination of stop gradient and exponential moving average techniques used in prior works \cite{grill2020bootstrap}, so that some theoretical analysis on the training dynamics can be carried out conveniently.
 
 For the architecture with a linear prediction head, we minimize the following loss:
\begin{align}
    \label{eq:loss_non-CL-l2-predictor}
    &L_{\text{non-CL-l2-pred}}(\bW^o, \bW^t, \bW^p):= 
  \mathbb{E}\left\Vert \dfrac{g_{\bm{W}^t, \bm{W}^p}(\bm{D}_1 \bm{x})}{ \| g_{\bm{W}^t,\bm{W}^p}(\bm{D}_1 \bm{x})\|_2} -  SG\left(\dfrac{\phi(h_{\bm{W}^t}(\bm{D}_2 \bm{x})}{\| \phi(h_{\bm{W}^t}(\bm{D}_2 \bm{x})) \|_2}\right)\right \Vert_2^2 
\end{align} 

For our theoretical results, we use a closely related loss function defined as:
\begin{align}
    \label{eq:loss_non-CL}
    &L_{\text{non-CL}}(\bW^o, \bW^t)\footnotemark :=   
    \left(2 - 2 \mathbb{E}\langle ReLU(\bm{W}^o \bb^o\bm{D}_1 \bm{x}) ,  SG( ReLU(\bm{W}^t, \bb^t \bm{D}_2 \bm{x})) \rangle\right) 
\end{align} 

In particular, for dual \emph{linear} network 
, the loss function is 
\begin{align}
    \label{eq:linear_network}
    &L_{\text{linear-non-CL}}(\bW^o, \bW^t) := 2 - 2 \mathbb{E}_{\vx, \mD_1, \mD_2} \langle \mW^o \mD_1 \vx, \mW^t \mD_2 \vx\rangle 
\end{align} 

\footnotetext{\label{footnote-unnormalized-loss} The contrastive loss in Eq. (\ref{eq:loss_CL}) and non-contrastive loss Eq. (\ref{eq:loss_non-CL-l2-predictor}, \ref{eq:loss_non-CL-pred-cosine-sim}) use the unnormalized representations instead of the normalized ones as it is simpler to analyze theoretically. For empirical experiments, we use normalized representations.}

For the architectures with a prediction head, we define the non-contrastive loss function as:
\begin{align}
    \label{eq:loss_non-CL-pred-cosine-sim}
    &L_{\text{non-CL-pred}}(\bW^o, \bW^t, \bW^p)^{\ref{footnote-unnormalized-loss}} :=  
    \left(2 - 2 \mathbb{E}\langle \bW^p ReLU(\bm{W}^o \bm{D}_1 \bm{x}) ,  SG(ReLU(\bm{W}^t \bm{D}_2 \bm{x}))\rangle\right) 
\end{align} 

\paragraph{Contrastive learning algorithm and loss function}
For contrastive learning, we use a simplified version of SimCLR algorithm (\cite{chen2020simple}). We assume a linear encoder with weight $\bm{W} \in \mathbb{R}^{m\times p}$ that extracts representation vectors of an augmented input. Given positive augmented samples $\bm{D}_1 \bm{x}$ and $\bm{D}_2 \bm{x}$ and a batch of augmented negative data samples $\mathbb{B} = \{\bm{D}' \bm{x}'\}$ with $\bm{x}'\neq\bm{x}$, the contrastive loss is defined as:

\begin{align}
    \label{eq:loss_CL}
    &L_{CL}(\bW) := 
    - \mathbb{E}\left[ \log \frac{\exp\{\tau S^{+} \} }{ \exp\{\tau   S^{+} \} + \sum_{\bm{x}' \in \mathbb{B}} \exp\{\tau S^{-} \}} \right].
\end{align}
where
\[
S^{+}:= \text{sim}(h_{\bW, \bb}(\bm{D}_1 \vx), h_{\bW, \bb}(\bm{D}_2 \vx)),
\]
\[
S^{-}:=\text{sim}(h_{\bW, \bb}(\bm{D}_1 \vx), h_{\bW, \bb}(\bm{D}_3' \bm{x}')), 
\]
and
\[ \text{sim}(\vy_1, \vy_2)\footref{footnote-unnormalized-loss} := \langle \vy_1,  \vy_2 \rangle = \vy_1^\top \vy_2 \]
representing the similarity of two vectors $\vy_1$ and $\vy_2$ 
and $\tau$ is the temperature parameter and $\bm{D}_3'$ is an augmentation matrix for the negative sample $\bm{x}'$. In our theoretical results, we will assume that $|\mathbb{B}| \to \infty$.

%% file: sections/empirical-setup.tex
\subsection{Experimental set up} \label{empirical-results-set-up}
We provide details of our experimental setup here.

\paragraph{Dataset} We use the data generating process described in section \ref{sec:setup:data-generation} to generate a synthetic dataset with 1000 data samples. The dictionary $\bM$ is generated by applying QR decomposition on a randomly generated matrix $\bm{G} \in \mathbb{R}^{p \times d}$ whose entries are i.i.d standard Gaussian. The sparse coding latent variable $\bm{z} \in \{0, \pm{1}\}^d$ is set to $\{-1,+1\}$ with an equal probability of $\beta/2${\footnote{In addition, we also ensure that the sparse coding vectors $\bz$ are generated such that at least one of the entries is non-zero.}}. The noise $\bm{\epsilon} \in \mathbb{R}^{p}$ is sampled i.i.d. from $\mathcal{N}\left(0, \frac{\log d}{d}\right)$.

\paragraph{Augmentation} We try both augmentation schemes described in in Eq. \ref{eq:dependent-masking} and in Eq. \ref{eq:augmentation}. Empirically, we find that the second augmentation scheme outperforms the first augmentation scheme. Therefore we report empirical results only on the second augmentation scheme.

\paragraph{Training setting}  We randomly initialize the weights of encoder and predictor networks for all the experiments (except for warm start). We sample each of the entries of these weight matrices from a Gaussian distribution $\mathcal{N}(0, \Theta(\frac{1}{pd}))$; and we always initialize the bias to be zero. For optimization, we use stochastic gradient descent (SGD) with a learning rate of 0.025 to train the model for 8000 epochs with batch size of 512. By default, the masking probability of random masks is 0.5, unless specified otherwise. In general, we use the dimensions m=50, p=50 and d=10, unless specified otherwise. As the  latents $\bm{z} \in \{-1, 0, 1\}$, we use symmetric ReLU activation (\cite{wen2021toward}) instead of the standard ReLU activation.

%% file: sections/results.tex
\section{Results} 
\label{sec:all-results}

In this section we formalize the differences between contrastive loss and non-contrastive loss under our simple model. 

We first formalize the notion of \emph{ground truth features} associated with our data generating process in Section~\ref{sec:setup:data-generation}.
Specifically, we say that an encoder successfully encodes our data distribution if given the observed data point $\vx = \mM \vz + \bm\epsilon$, the encoder is able recover the latent variable $\vz$ up to permutation, i.e. 
\begin{equation}
    \label{eq:encoder_recovers_latent}
    \text{ReLU}(\mW^o x) \approx c \mP \vz
\end{equation}
for some constant $c$ and permutation matrix $P$. We will sometimes assume that $\bW^o$ is normalized (i.e. has unit row norm), in which case we have $c = 1$. Note, the indeterminacy up to permutation is unavoidable: one can easily see that permuting both the latent coordinates and the matrix $M$ correspondingly results in the same distribution for $\bm{x}$.

In the following sections, we present two lines of results that highlight some fundamental differences between contrastive and non-contrastive loss. Specifically, we conduct \emph{landscape analysis} (Section~\ref{sec:land}), which focuses on the properties of global optima; and \emph{training dynamics analysis} (Section~\ref{sec:dyn}), which focuses on the properties of the points that the training dynamics converge to. Both lines of analysis reveal the fragility of non-contrastive loss: its global optima contain bad optima, from which the training process cannot easily navigate through without careful architectural engineering --- in particular, the inclusion of a predictor $\bm{W}^p$. 
We provide a combination of theoretical results with extensive experiments. The network architecture and the values of hyper-parameters cover a large spread for thoroughness, and we specify these wherever appropriate.

\paragraph{Evaluation metric} Traditionally, the success of a self-supervised learning (SSL) algorithm is determined by evaluating the learnt representations of the encoder on a downstream auxiliary task. For example, the representations learnt on images are typically evaluated through a linear evaluation protocol (\cite{kolesnikov2019revisiting, bachman2019learning, chen2020simple, grill2020bootstrap}) on a standard image classification datasets like ImageNet. A higher accuracy on the downstream tasks is indicative of the better quality of the learnt representations.

In our experimental setup, a self-supervised learning algorithm will be considered successful if we are able to find the following separation through the learnt weights $\bm{W}$ of the encoder.
\begin{align}
    \langle \bW_{i*} , \bm{M}_{*j} \rangle_{j \in \N_i} \gg \langle \bW_{i*} , \bm{M}_{*j} \rangle_{j \in [d]\setminus \N_i}
\end{align}
where $\N_i\subseteq [d]$ is the subset of dictionary bases (i.e. $\{M_{*1}, \dots, M_{*d}\}$) that neuron $i$ (approximately) lies in. 

Motivated by this goal, we propose an alternate approach to evaluate the success of a SSL algorithm under the sparse coding setup. Specifically, we assess the quality of learnt representations by computing the maximum, median and minimum values of the following expression,
\begin{equation}
    \label{eq:max-cosine}
    \max_i \Big | \Big \langle \dfrac{\bm{W_{i*}}}{\|\bm{W_{i*}}\|_2 }, \dfrac{\bm{M_{*j}}}{\|\bm{M_{*j}}\|_2} \Big \rangle \Big | \;\;\; \forall j \in [m],
\end{equation}
referred to as \textit{Maximum max-cosine}, \textit{Medium max-cosine} and \textit{Minimum max-cosine} respectively, for each $j \in [d]$ and use these values to determine if the SSL algorithm has correctly recovered the ground truth dictionary matrix $\bm{M}$. Ideally, we want the SSL algorithm to learn the correct weights $\bm{W}$ of its linear encoder such that the values of the above dot product are close to $1$ for all the three metrics as this would indicate near-perfect recovery of the ground truth support. Of these three metrics, intuitively high \textit{Minimum max-cosine} is the most indicative of the success of a SSL algorithm, as it suggests that even the worst alignment of dictionary and neurons is good. We therefore include plots corresponding to \textit{Minimum max-cosine} and training loss in the main paper, and defer the plots of the other metrics to Appendix \ref{sec:appendix-landscape} (for additional empirical results of Landscape analysis) and Appendix \ref{sec:appendix-training-dynamics} (for additional empirical results of training dynamics analysis).

\subsection{Landscape analysis}\label{sec:land}

In this section, we provide results which show that non-contrastive loss has infinitely many non-collapsed global optima that are far from the ground truth. By contrast, contrastive learning loss guarantees recovery of the correct ground truth support. In particular, this happens even in the extremely simple setting in which there is no noise in the data generating process (i.e. $\sigma_0^2 = 0$) and $M = I$. 

Precisely, we show:

\begin{theorem}[Landscape of contrastive and non-contrastive loss]\label{thm:byolbad}
Let the data generating process and network architecture be specified as in Section \ref{sec:setup}, 
and consider the setting $d=p=m , \bm{M} = \mathbf{I}$, and $\sigma^2_0 = 0$.
Moreover, let the latent vectors $\{\bm{z}_j \in \mathbb{R}^{d \times 1}\}_{j=1}^d$ be chosen by a uniform distribution over 1-sparse vectors. 

Then, we have:
\begin{itemize}
    \item[(a)] 
    $\mathcal{U}_\ge \subseteq \argmin_{W \in \mathcal{U}} L_{\text{non-CL}}(W, W)$
    \item[(b)] $\argmin_{W \in \mathcal{U}} L_{\text{CL}}(W, W)$ 
    is the set of permutation matrices. 
\end{itemize}
\end{theorem}

Part (a) of Theorem~\ref{thm:byolbad} shows,  since any matrix $\mW \in \mathcal{U}_{>}$ is a global optimum of the non-contrastive loss, optimizing it does not necessarily lead to learning of the groundtruth features described in \eqref{eq:encoder_recovers_latent}.
Indeed, there are clearly abundant elements $\mW \in \mathcal{U}_{>}$ such that $\text{ReLU}(\mW x)$ and $c \mP \vz$ are very different by any measure. On the other hand, part (b) of Theorem~\ref{thm:byolbad} shows optimizing the contrastive loss objective guarantees learning the groundtruth features up to a permutation.  

The proof of this theorem is deferred to Appendix~\ref{sec:appendix-landscape}, though one core idea of the proof is relatively simple. For 1a), it's easy to see that so long as all the inputs to the ReLU are non-negative, the objective is minimized --- from which the result easily follows. For 1b) the result follows by noting that in order to minimize the contrastive loss, it suffices that the columns of $W$ are non-negative and orthogonal --- which is only satisfied when $W$ is a permutation matrix. 

\begin{remark}
For landscape analysis, it suffices to show the abundance of non-collapsed bad global optima.
Therefore, for simplicity, we remove the bias and use the (asymmetric) ReLU activation.
Note that without the bias, a symmetric ReLU would be the same as the identity function, which would lose the non-linearity.
\end{remark}

\begin{figure}[!htbp]
  \centering
  \begin{minipage}[b]{0.45\textwidth}
    \centering
    \includegraphics[width=0.75\textwidth]{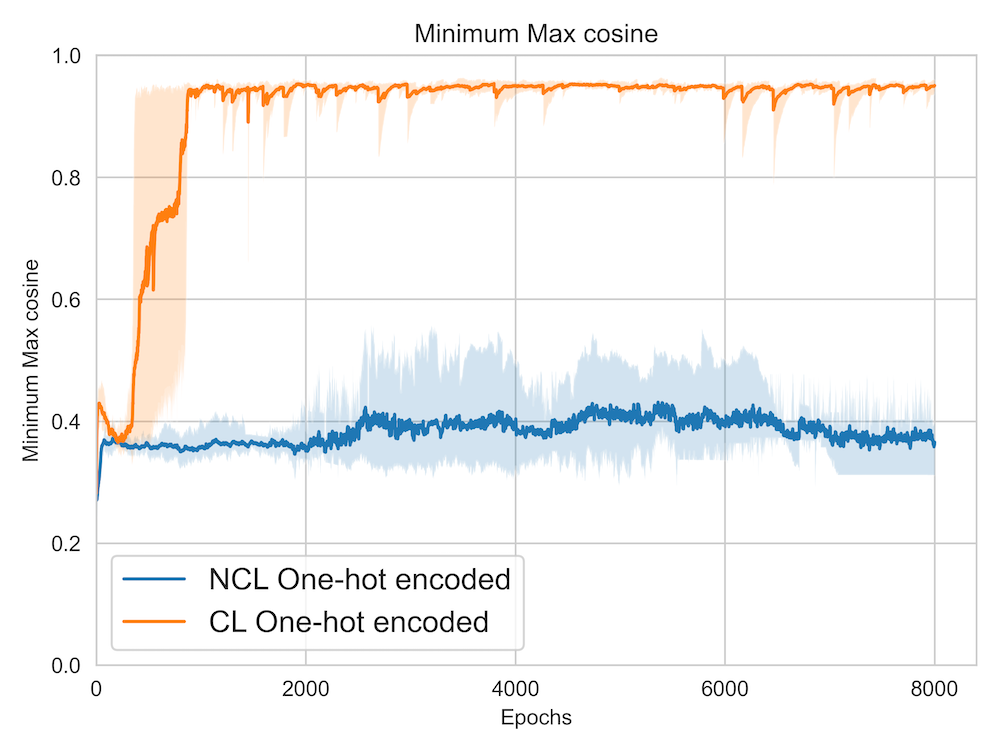}
  \end{minipage}
  \hfill
  \begin{minipage}[b]{0.45\textwidth}
    \centering
    \includegraphics[width=0.75\textwidth]{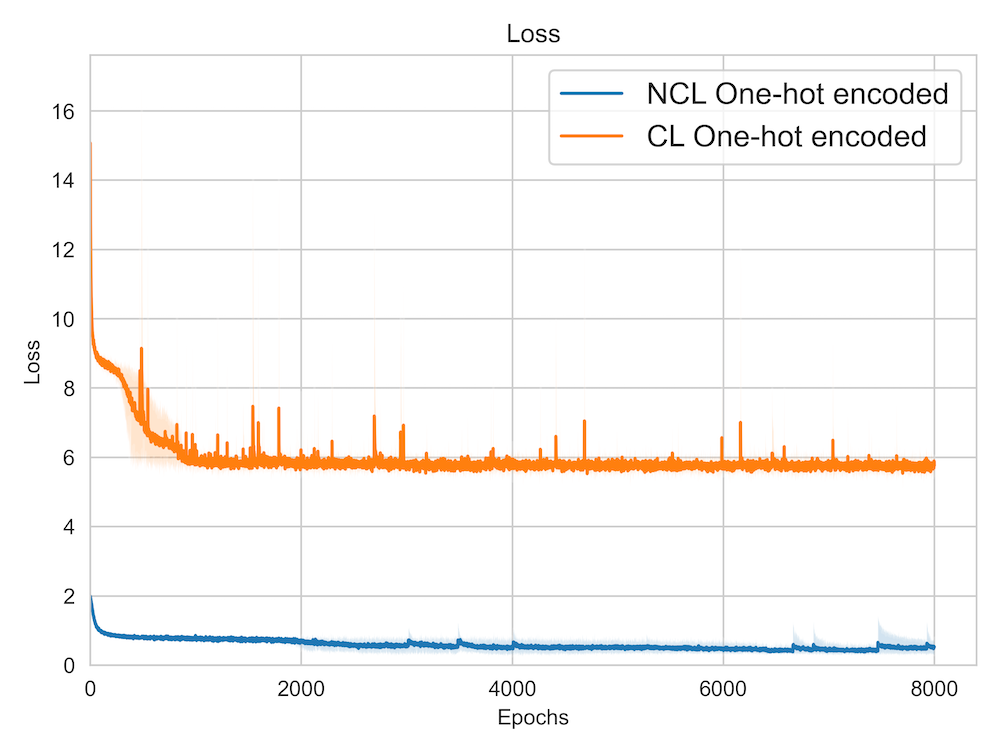}
  \end{minipage}
  \caption{(NCL vs CL with one-hot latents) Minimum Max-Cosine (left) and Loss (right) curves for non-contrastive loss (NCL) and contrastive loss (CL) on an architecture with a randomly initialized linear encoder. We normalize the representations before computing the loss and use a symmetric ReLU after the linear encoder. The latent $\bm{z}$ is one-hot encoded. Reported numbers are averaged over 5 different runs.}  \label{fig:landscape-one-hot}
\end{figure}

\begin{figure}[tbp]
  \centering
  \begin{minipage}[b]{0.45\textwidth}
    \centering
    \includegraphics[width=0.75\textwidth]{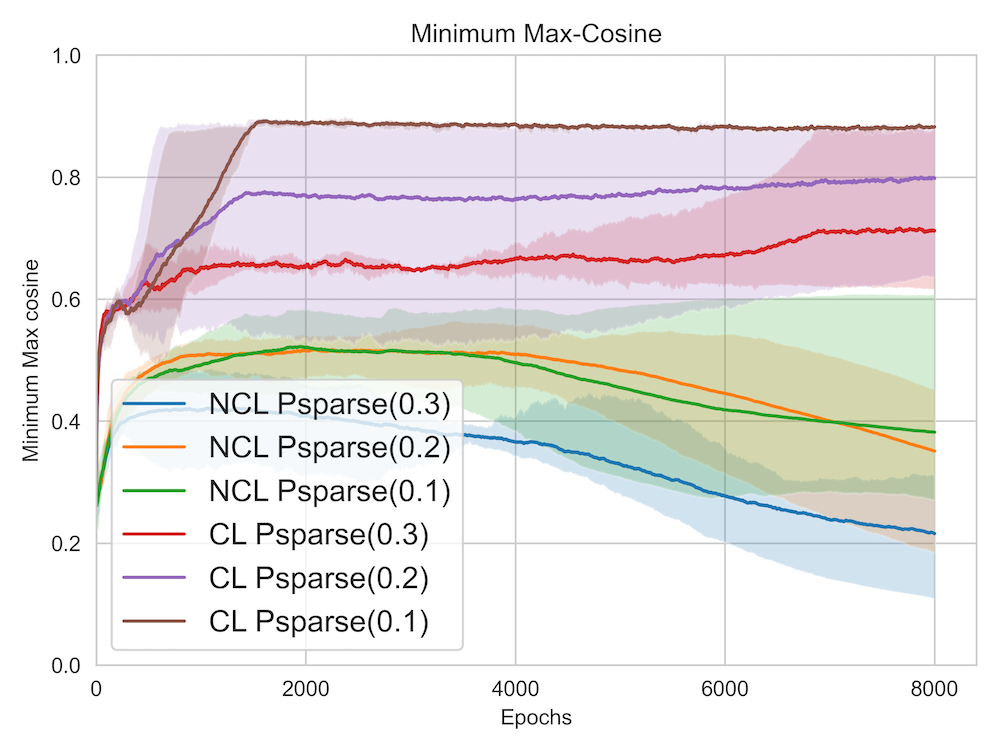}
  \end{minipage}
  \hfill
  \begin{minipage}[b]{0.45\textwidth}
    \centering
    \includegraphics[width=0.75\textwidth]{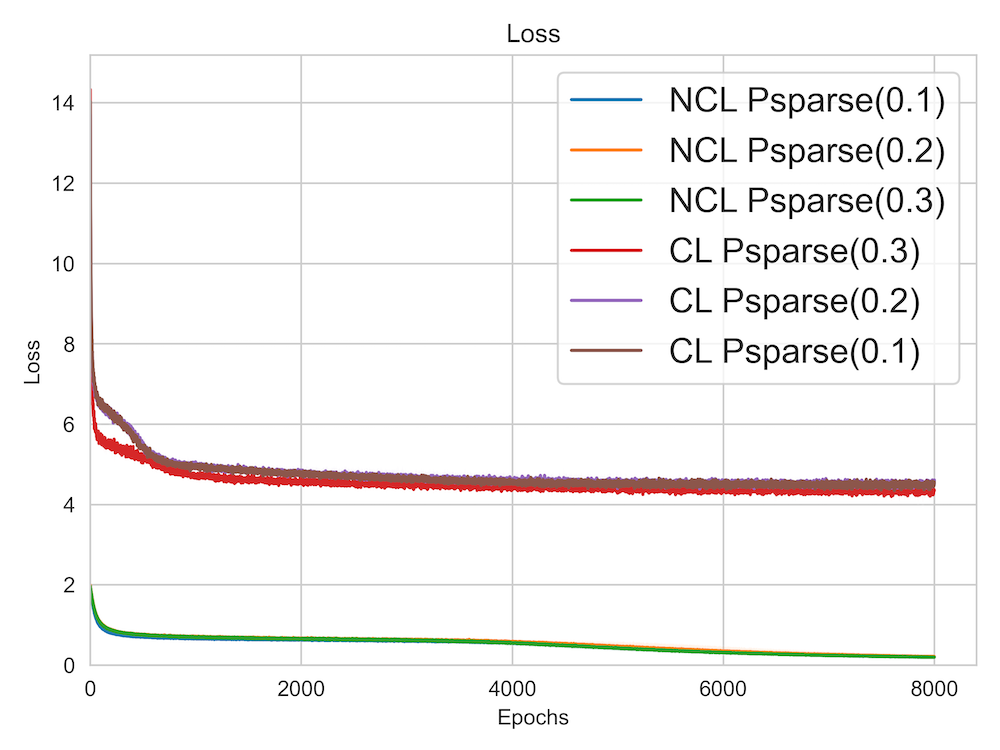}
  \end{minipage}
  \caption{(NCL vs CL loss with k-sparse latents) Minimum Max-Cosine (left) and Loss (right) curves for non-contrastive loss (NCL) and contrastive loss (CL) on an architecture with a randomly initialized linear encoder. We include Batch Normalization  (\cite{ioffe2015batch}) layer and symmetric ReLU activation after the linear encoder. Psparse indicates $Pr(\bm{z}_i = \pm 1), i \in [d]$ in the sparse coding vector $\bm{z}$. Reported numbers are averaged over 5 different runs.}
\label{fig:landscape-k-sparse}
\end{figure}

\subsection{Training dynamics analysis}\label{sec:dyn}

\subsubsection{Theoretical results on training dynamics}
\label{sec:dyn:theory}

We analyze the training dynamics for two simple model architectures, namely dual linear networks and dual ReLU networks, 
under the data generating process specified by \eqref{eq:data-generation}, in which the latent vector $z$ follows a symmetric Bernoulli distribution which we define in the following Assumption~\ref{assumption:symmetric_bernoulli_latent}:
\begin{assumption}[symmetric Bernoulli latent]
    \label{assumption:symmetric_bernoulli_latent}
    Let the latent vector $\vz$ be such that for each index $i \subset [d]$, $1$ and $\forall l \notin K$, $\evz_l = 0$.
    \[ \evz_i = \begin{cases}
        -1, \quad &\text{with probability } \frac{\kappa}{2} \\ 
        1, \quad &\text{with probability } \frac{\kappa}{2} \\ 
        0, \quad &\text{with probability } 1-\kappa
        \end{cases} \]
\end{assumption}

\begin{remark}
    Our theoretical result for training dynamics uses the symmetric (i.e. $0, \pm1$) instead of the binary (i.e. $0, 1$) latent, 
    because our experiments (reported in Section~\ref{sec:dyn:experimental}) use the symmetric version. \\
\end{remark}

\noindent\textbf{Scenario 1: non-warm-start regime: dual linear networks} \newline

The motivation behind looking at linear enocders is that if the initialization is far from the groundtruth, 
then ReLU will be in the linear regime most of the time, 
and we characterize the training dynamics in this regime.
First, building upon the loss and architecture defined in \eqref{eq:linear_network}, we use an additional L2 regularization (with parameter $\lambda \in (0, 1)$), which is essential in this fully-linear architecture to prevent the weights from exponentially growing.
\begin{align}
    \label{eq:linear_network_weight_decay}
    &L_{\text{linear-non-CL}}(\bW^o, \bW^t) := 2 - 2 \mathbb{E}_{\vx, \mD_1, \mD_2} \langle \mW^o \mD_1 \vx, \mW^t \mD_2 \vx\rangle + \frac{1-\lambda}{2} \| \bW^o \|_2^2 + \frac{1-\lambda}{2} \| \bW^t \|_2^2
\end{align} 

\begin{remark}[weight decay schedule and learning rate]
    \label{remark:weight_decay}
    \eqref{eq:linear_network_weight_decay} implicitly assumes that the weight is multiplied by the same decay factor $\lambda \in (0, 1)$ in each iteration. 
    This could be generalized to using a varying weight decay schedule, i.e. the weight-decay factor in step $t$ is $\lambda_t \in (0, 1)$ and can change during the optimization process. 
    With the latter setting, the calculation is essentially the same.
    
    Besides, we assume that the learning rate in step $t$, denoted as $\eta_t$, is small compared with the weight decay factor $\lambda$.
    This assumption is formally stated in the following Theorem~\ref{thm:linear_network}, and is consistent with common practice of setting $\lambda$ close to 1, and setting a small learning rate $\eta_t$ close to 0.
\end{remark}

We explicitly write the optimization process and characterize the limitation of non-contrastive training process in learning the groundtruth even in our simple setting:

\begin{theorem}[non-contrastive loss on linear network]
    \label{thm:linear_network}
    Suppose that the data generating process is specified as in Section~\ref{sec:setup} with the latent distribution changed to symmetric (Assumption~\ref{assumption:symmetric_bernoulli_latent}), 
    and the weights are initialized to $\mW^o_0$ and $\mW^t_0$, respectively.
    In step $t$, denote the learning rate as $\eta_t < \frac{\lambda}{2 \alpha^2 (2 \frac{\kappa}{2} + \sigma_0^2)}, \forall t$.
    Then, running gradient descent on the loss $L_{\text{linear-non-CL}}(\mW^o_t, \mW^t_t)$ in \eqref{eq:linear_network_weight_decay}
    will lead to
    \begin{align*}
        \mW^o_t \mM &= C_{1,t} \mW^o_0 \mM + C_{2,t} (\mW^o_0 \mM + \mW^t_0 \mM) \\
        \mW^t_t \mM &= C_{1,t} \mW^t_0 \mM + C_{2,t} (\mW^o_0 \mM + \mW^t_0 \mM)
    \end{align*}
    for some scalars $C_{1,t} := \prod_{i=0}^{t-1}(\lambda-c_i) \in (0, 1)$, and $C_{2,t} := \sum_{j=0}^{t-1} \left( c_i \prod_{i=j}^{t-1}(\lambda-c_i) \prod_{i=0}^{j-1}(\lambda+c_i) \right) > 0$,
    in which $c_i = 2 \eta_i \alpha^2 (2 p_z + \sigma_0^2) \in (0, \lambda)$.
\end{theorem}

\begin{remark}[limitation of dual linear networks]
    \label{remark:linear_network}
    By \eqref{eq:encoder_recovers_latent} with ReLU changed to identity, a good encoder $\mW^o$ should satisfy $\mW^o \mM \approx c \mP$.
    However, by the above Theorem~\ref{thm:linear_network}, the learned encoder $\mW^o_t \mM$  will converge to a linear combination of $\mW^o_0 \mM$ and $\mW^t_0 \mM$, which is a one-dimensional subspace.
    In particular, if this subspace is far from any permutation matrix $\mP$, then the learned encoder will stay away from the good encoders that can approximately recover the latent vector $\vz$.
\end{remark}

The proof of this Theorem~\ref{thm:linear_network} is deferred to Appendix~\ref{sec:proof:linear}. \\

\noindent\textbf{Scenario 2: warm-start regime: dual ReLU networks} \newline

In the warm-start regime, we take into consideration the ReLU non-linearity,
in which the ReLU activation is also the symmetric version, corresponding to the symmetric latent described in Assumption~\ref{assumption:symmetric_bernoulli_latent}. Namely, $\text{SReLU}_{\vb}(\vx) := ReLU(\vx - \vb) - ReLU(-\vx - \vb)$, for some positive bias vector $\vb$. 
Equivalently: 
\begin{equation}
    \label{eq:srelu}
    \begin{split}
        (\text{SReLU}_{\vb}(\vx))_i = \begin{cases}
        \evx_i - \evb_i, \quad &\text{if } \evx_i > \evb_i \\ 
        0, \quad &\text{if } -\evb_i \le \evx_i \le \evb_i \\ 
        \evx_i + \evb_i, \quad &\text{if } \evx_i < -\evb_i \\ 
        \end{cases}
    \end{split}
\end{equation}

Hence the loss function is:
\begin{equation}
    \label{eq:relu_network}
    \begin{split}
        L(\mW^o, \mW^t) = 2 - 2 \mathbb{E}_{\vx, \mD_1, \mD_2} &\langle \text{SReLU}_{\vb^o}(\mW^o \mD_1 \vx) , \\
            &\text{SReLU}_{\vb^t}(\mW^t \mD_2 \vx) \rangle
    \end{split}
\end{equation}

We study the case $\mM = \mI$, under a series of assumptions that we state formally and discuss their significance below. 

\begin{assumption}[warm start]
    \label{assumption:warm_start_symmetric_bernoulli}
    We assume that $\mW^o$ and $\mW^t$ are both warm-started, i.e. let 
    \begin{align*}
        \mW^o &= \mM + \bm\Delta^o = \mI + \bm\Delta^o \\
        \mW^t &= \mM + \bm\Delta^t = \mI + \bm\Delta^t
    \end{align*}
    for some initial error $\bm\Delta^o, \bm\Delta^t$ such that $\forall i, j, |\bm\Delta^o_{ij}| < \frac{1}{10d}, |\bm\Delta^t_{ij}| < \frac{1}{10d}$.
\end{assumption}

\begin{assumption}[small noise]
    \label{assumption:small_noise}
    We assume that the magnitude of the noise $\| \bm\epsilon \|_2$ is smaller than the initial warm start errors $\bm\Delta^o_{ij}, \bm\Delta^t_{ij}$ of the weights, i.e.
    \[ \forall i, j, \| \bm\epsilon \|_2 \le |\bm\Delta^o_{ij}|, \| \bm\epsilon \|_2 \le |\bm\Delta^t_{ij}| \]
\end{assumption}

\begin{remark}
    Note that Assumption~\ref{assumption:small_noise} assumes a smaller noise scale than what is used in our experiments (described in Section~\ref{sec:setup:data-generation}).
    In particular, the noise does not have to exist, i.e. the noiseless setting ($\bm\epsilon = 0$) satisfies Assumption~\ref{assumption:small_noise}.
\end{remark}

\begin{assumption}[bias]
    \label{assumption:bias_symmetric_bernoulli}
    We assume that the bias $\vb^o$ and $\vb^t$ are fixed throughout the optimization process and satisfy the following requirements:
    for each set $i \in [d]$, let
    \begin{align*}
        c_b^o &= \sum_{l \ne i} \bm\Delta^o_{il} \emD_{1,ll} \evz_l + \sum_{j=1}^d (\emI_{ij} + \bm\Delta^o_{ij}) \emD_{1,jj} \bm\epsilon_j \\
        c_b^t &= \sum_{l \ne i} \bm\Delta^t_{il} \emD_{1,ll} \evz_l + \sum_{j=1}^d (\emI_{ij} + \bm\Delta^t_{ij}) \emD_{1,jj} \bm\epsilon_j 
    \end{align*}
    which are 
    in $(0, 1)$ by Assumptions \ref{assumption:warm_start_symmetric_bernoulli} and \ref{assumption:small_noise}.
    Let $b_i^o, b_i^t \in (0, 1)$ satisfy:
    \begin{align*}
        \max\{-c_b^o, c_b^o, 0\} &< b_i^o < \min\{1 + c_b^o, 1 -  c_b^o\} \\
        \max\{-c_b^t, c_b^t, 0\} &< b_i^t < \min\{1 + c_b^t, 1 -  c_b^t\}
    \end{align*}
\end{assumption}

Our following theorem characterizes the convergence point of alternating optimization on the non-contrastive loss:

\begin{theorem}[ReLU network with normalization, warm start]
    \label{thm:relu_no_pred}
    Under Assumptions~\ref{assumption:symmetric_bernoulli_latent}, \ref{assumption:warm_start_symmetric_bernoulli}, \ref{assumption:bias_symmetric_bernoulli}, running:
    \item Repeat until both $\mW^o$ and $\mW^t$ converges:
    \begin{itemize}
        \item Repeat until $\mW^o$ converges: \\
            \[ \mW^o \leftarrow \text{normalize}(\mW^o - \eta \nabla_{\mW^o} L(\mW^o, \mW^t)) \]
        \item Repeat until $\mW^t$ converges: \\
            \[ \mW^t \leftarrow \text{normalize}(\mW^t - \eta \nabla_{\mW^t} L(\mW^o, \mW^t)) \]
    \end{itemize}
    will make $\mW^o$ and $\mW^t$ both converge to $I$
    in which \texttt{normalize} means row-normalization, i.e. 
    $\| \mW^o_{i*} \|_2 = 1$ and $\| \mW^t_{i*} \|_2 = 1$ for each $i$.
\end{theorem}

The proof is deferred to Section~\ref{sec:proof:relu_no_pred}.

\subsubsection{Experimental results on training dynamics}
\label{sec:dyn:experimental}

We analyze properties of the training dynamics via extensive experimental results. 
 We follow the set up described in section \ref{empirical-results-set-up}.
 
\paragraph{NCL from random initialization does not work} First, we empirically show that the training dynamics for the non-contrastive loss (Eq. \ref{eq:loss_non-CL}) fails to learn the correct ground truth representations whereas contrastive loss (Eq. \ref{eq:loss_CL}) is able to successfully learn the correct representations. We include plots of \textit{Minimum Max-Cosine} and training loss in Figure \ref{fig:landscape-one-hot} and \ref{fig:landscape-k-sparse}. Both of these figures show that when encoder is a one layer ReLU network, NCL has lower values of \textit{Minimum Max-Cosine} which indicates that it has failed to recover the correct support of ground truth dictionary $\bm{M}$, whereas CL is able to learn good representations across different levels of sparsity as shown by its high values of \textit{Minimum Max-Cosine}. We include additional plots for \textit{Maximum Max-Cosine} and \textit{Median Max-Cosine} in Appendix \ref{sec:appendix-landscape}. Finally, from Figure \ref{fig:landscape-ncl-k-sparse-ablation}, we note that introducing additional linear layers in the base encoder does not aid non-contrastive loss to learn the correct ground truth representations.

\begin{figure}[!tbp]
  \centering
  \begin{minipage}[b]{0.45\textwidth}
    \centering
    \includegraphics[width=0.75\textwidth]{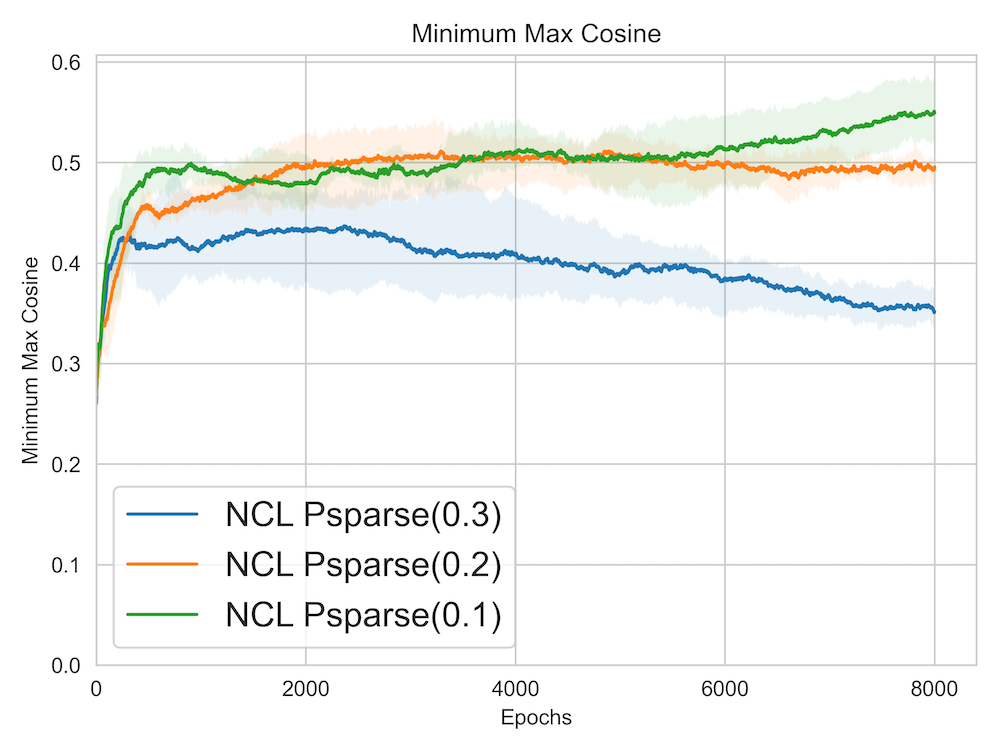}
  \end{minipage}
\hfill
  \begin{minipage}[b]{0.45\textwidth}
    \centering
    \includegraphics[width=0.75\textwidth]{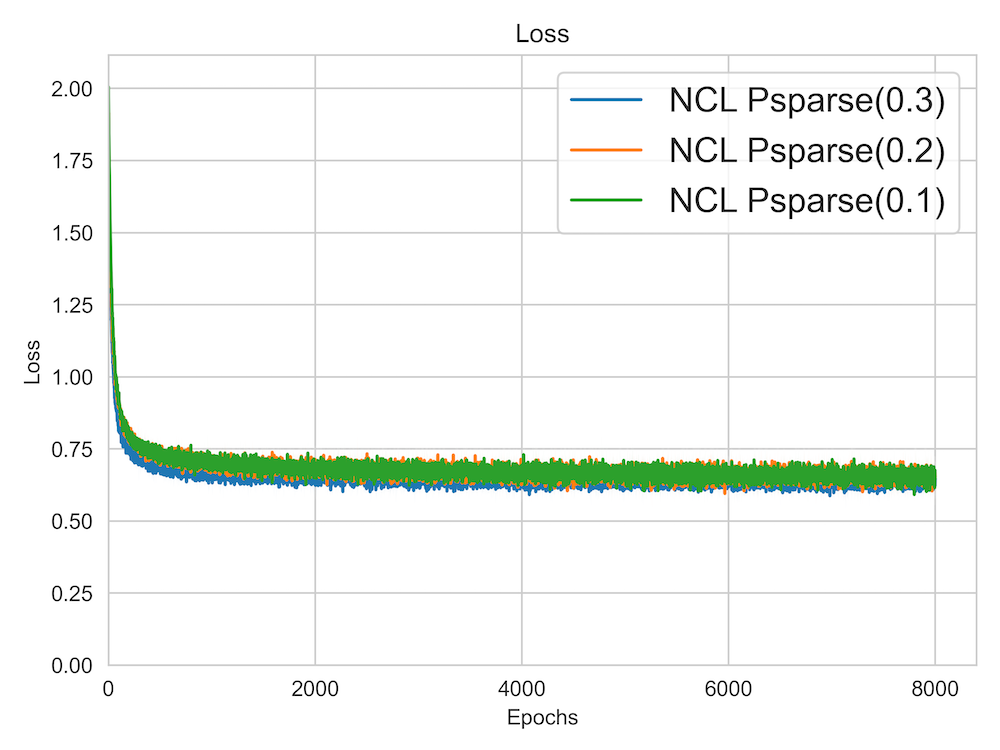}
  \end{minipage}
  \caption{(NCL with 2 layered encoder does not work) Minimum Max-Cosine (left) and Loss (right) curves for non-contrastive loss (NCL) with a two-layered linear encoder with batch-normalization and symmetric ReLU activation. Psparse indicates $Pr(\bm{z}_i = \pm 1), i \in [d]$ in the sparse coding vector $\bm{z}$. Reported numbers are averaged over 5 different runs.}
\label{fig:landscape-ncl-k-sparse-ablation}
\end{figure}

\begin{figure}[!tbp]
  \centering
  \begin{minipage}[b]{0.45\textwidth}
     \centering
    \includegraphics[width=0.75\textwidth]{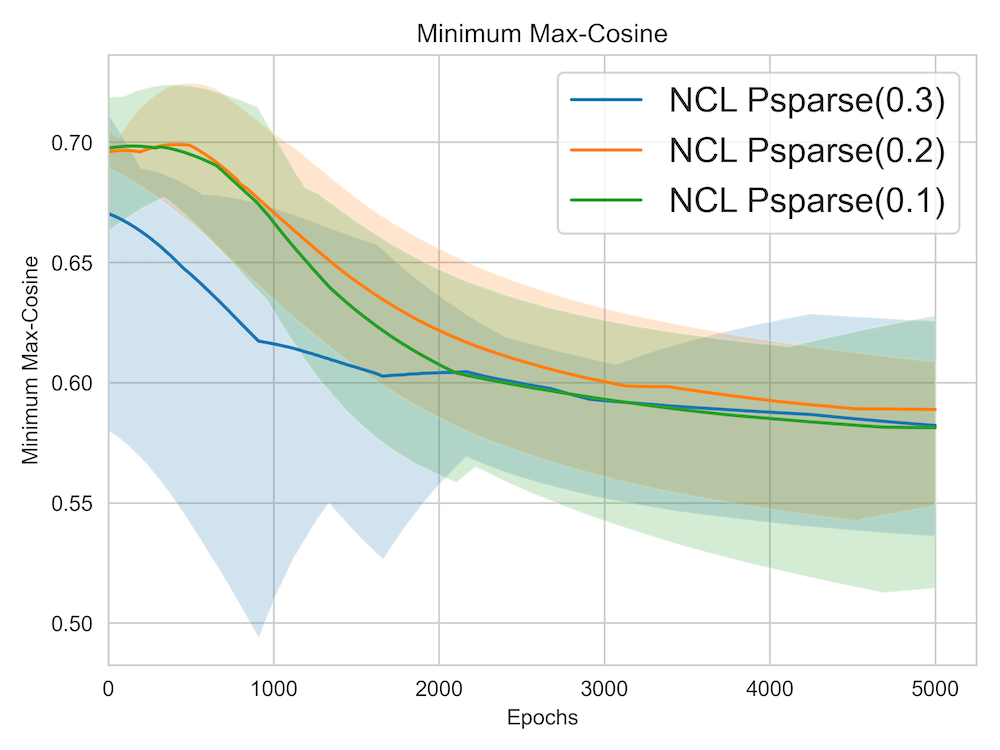}
  \end{minipage}
    \hfill
  \begin{minipage}[b]{0.45\textwidth}
     \centering
    \includegraphics[width=0.75\textwidth]{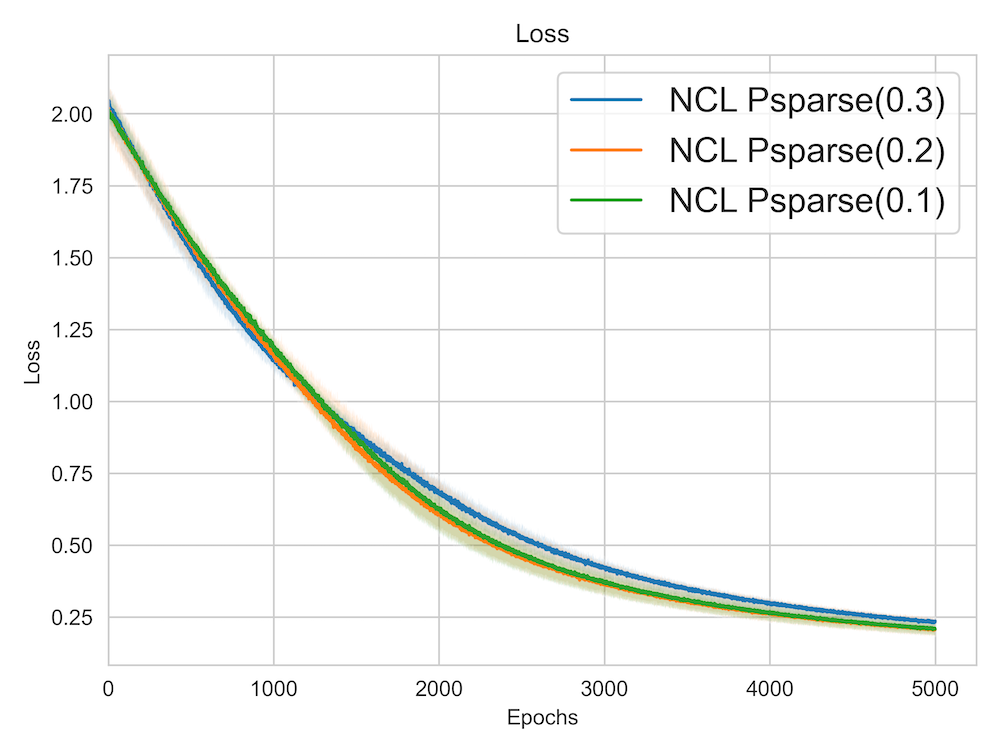}
  \end{minipage}
  \caption{(NCL with warm start does not work) Minimum Max-Cosine (left) and Loss (right) curves for non-contrastive loss (NCL) with a warm-started linear encoder. 
  Reported numbers are averaged over 5 different runs.}
\label{fig:theorem-3-warm-start-ncl}
\end{figure}

\begin{figure}[!tbp]
  \centering
  \begin{minipage}[b]{0.45\textwidth}
    \centering
    \includegraphics[width=0.75\textwidth]{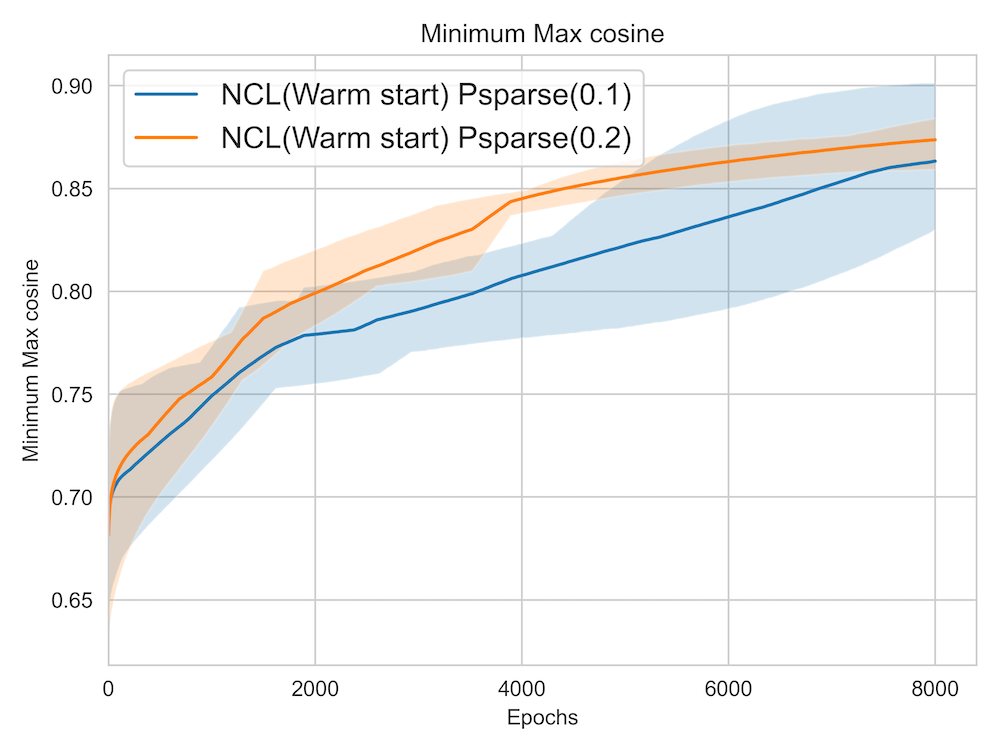}
  \end{minipage}
  \hfill
  \begin{minipage}[b]{0.45\textwidth}
    \centering
    \includegraphics[width=0.75\textwidth]{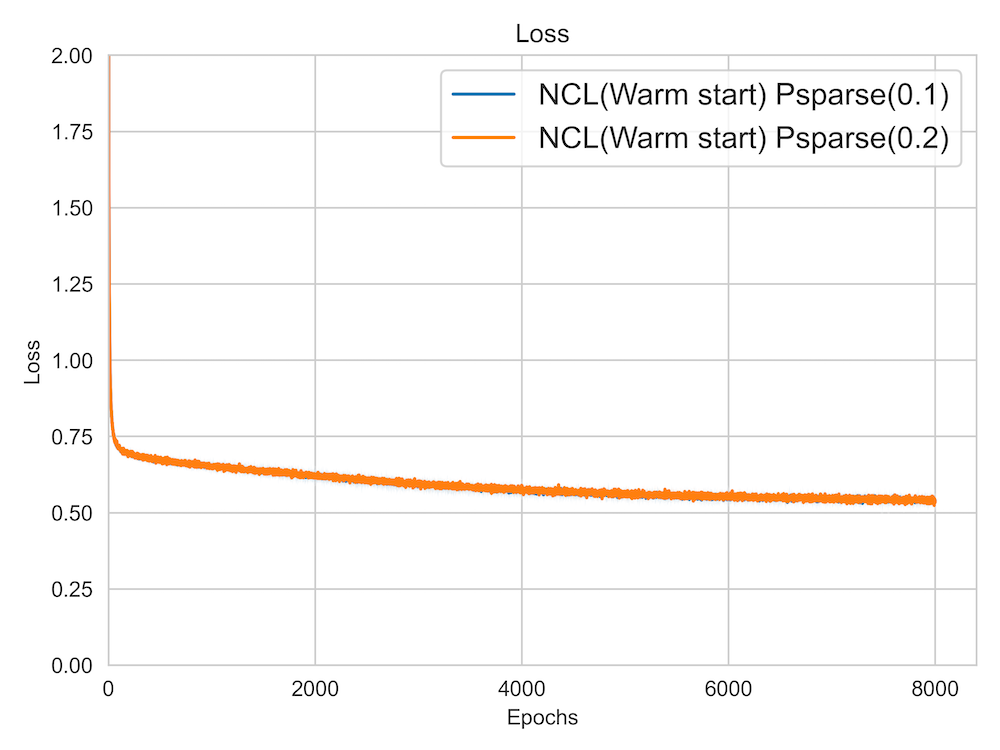}
  \end{minipage}
  \caption{(NCL with a linear prediction head and warm-start works) Minimum Max-Cosine (left) and Loss (right) curves for non-contrastive loss (NCL) with warm-started linear encoder and a linear predictor. Reported numbers are averaged over 3 different runs.}
 \label{fig:theorem-4-ncl-pred-warm-start}
\end{figure}

Next, we provide several experiments that demonstrate the importance of including a \emph{prediction head} $\bm{W}_p$. Namely, we will consider starting the training dynamics from a \emph{warm start} --- i.e. a point nearby the optimum. Precisely, we initialize the columns of the matrix $\bm{W}$ with random columns of the dictionary matrix $\bm{M}$ and add Gaussian noise $\mathcal{N}(0, \sigma^2\bI)$ with $\sigma = \frac{1}{p^{c/2}}$ where the choice of $\sigma$ determines the closeness of the matrix $\bm{W}$ to the dictionary matrix. For the results in this section, we use $c=1$. As we use $p=50$ for these experiments, we have $\sigma \approx 0.141$. We include additional results for other values of $\sigma$ in Appendix \ref{sec:appendix-training-dynamics}.

\paragraph{NCL with warm start only fails to learn good representations} We provide empirical evidence that even with warm-start, non-contrastive loss without a prediction head $\bm{W}_p$ fails to learn the correct ground truth representations. Figure \ref{fig:theorem-3-warm-start-ncl} shows that for non-contrastive loss with warm start, \textit{Minimum Max-Cosine} decreases as the training proceeds which points towards failure of this model to learn good representations.

\paragraph{NCL with warm start and prediction head learns good representations}
On the other hand, from Figure \ref{fig:theorem-4-ncl-pred-warm-start}, we can conclude that if we do include a linear predictor, from a warm start non-contrastive loss objective can learn the correct ground truth representations. We observe that \textit{Minimum Max-Cosine} increases on average from around $0.7$ to $0.87$ as the training proceeds. Finally, Figure \ref{fig:theorem-4-ncl-pred-no-warm-start} shows that in the absence of "warm-started" encoder, even with a predictor, NCL without a projection head fails to learn the ground truth representations. This experiment, together with Figure \ref{fig:theorem-4-ncl-pred-no-warm-start} indicate that the inclusion of a predictor head $W_p$ is of paramount importance to training; however even with the inclusion of the predictor head, the training dynamics seem unable to escape the plethora of bad minima of the optimization landscape. We provide some theoretical evidence for these phenomena as well in Appendix \ref{sec:appendix-training-dynamics}.

As a technical implementation detail, we note for all the experiments in this subsection the encoder includes a batch normalization layer. There is no symmetric ReLU activation or Batch Normalization layer after the predictor.

\begin{figure}[htp!]
  \centering
  \begin{minipage}[b]{0.45\textwidth}
  \centering
    \includegraphics[width=0.75\textwidth]{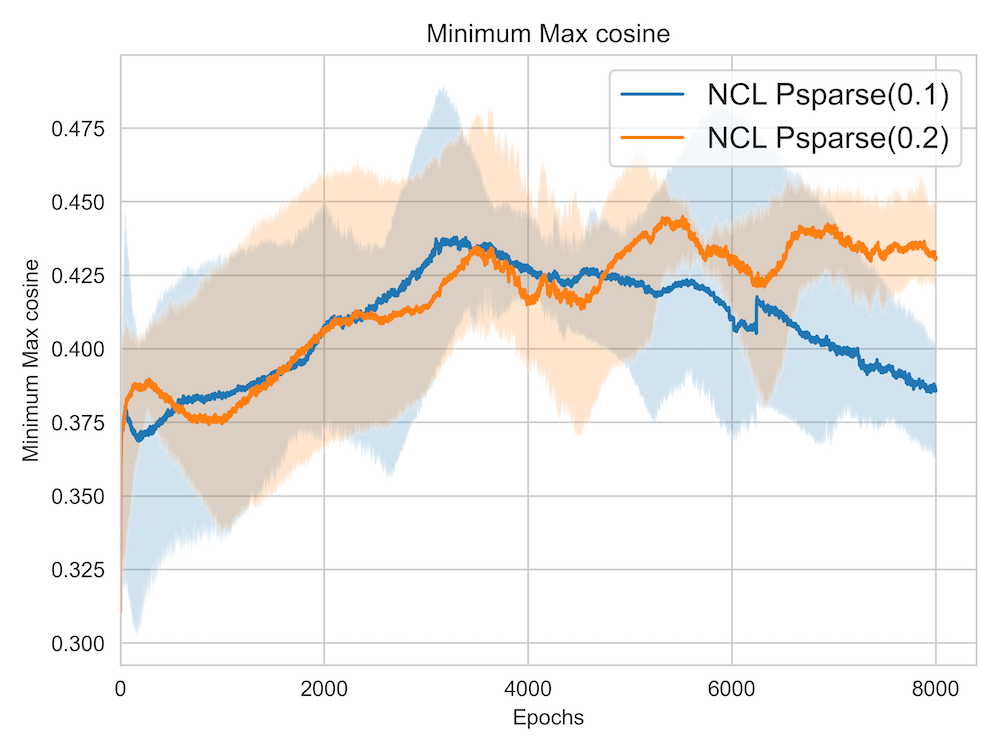}
  \end{minipage}
  \hfill
  \begin{minipage}[b]{0.45\textwidth}
  \centering
    \includegraphics[width=0.75\textwidth]{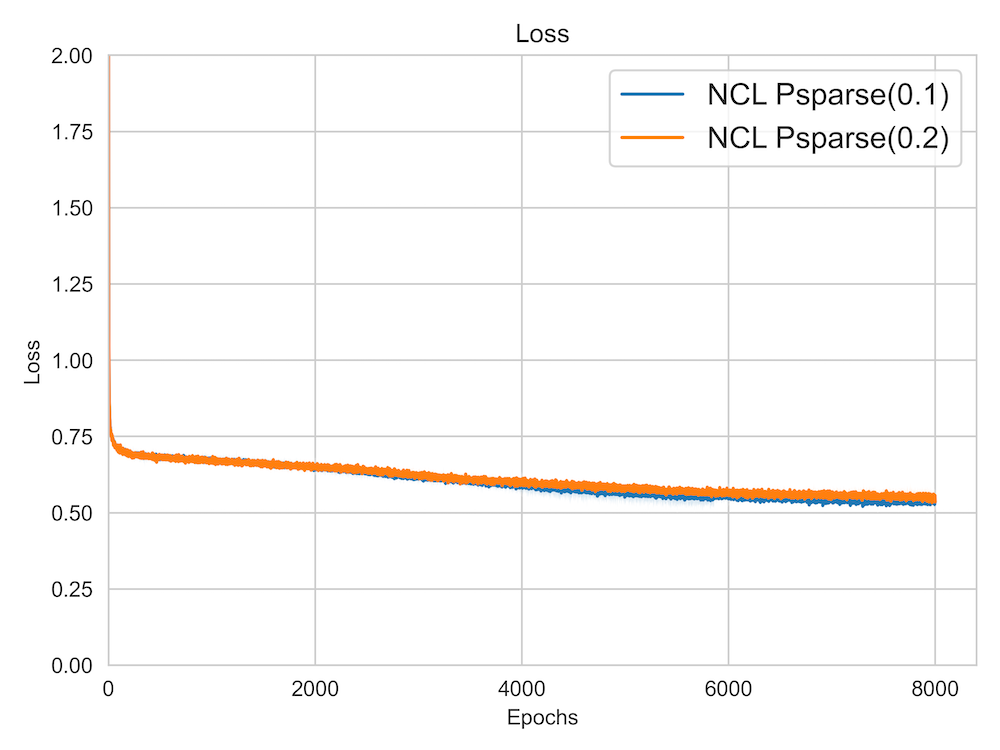}
  \end{minipage}
  \caption{(NCL with a linear prediction head does not work without warm-start ) Minimum Max-Cosine (left) and Loss (right) curves for non-contrastive loss (NCL) with non warm-started linear encoder and a linear predictor. Reported numbers are averaged over 3 different runs.}
 \label{fig:theorem-4-ncl-pred-no-warm-start}
\end{figure}

\begin{table*}[th!]
\centering
\begin{tabular}{ccccc}
\toprule
Model & Pr(sparse) & \multicolumn{1}{p{2.5cm}}{\centering Maximum \\ Max cosine $\uparrow$} & \multicolumn{1}{p{2.5cm}}{\centering Median \\ Max cosine $\uparrow$ }  & \multicolumn{1}{p{2.5cm}}{\centering Minimum \\ Max cosine $\uparrow$}  \\
\midrule
Simplified-SimCLR & 0.1 &  $0.94 \pm 0.005$ & $0.92 \pm 0.004$ & $0.83 \pm 0.16$ \\
Simplified-SimCLR & 0.2 & $0.94 \pm 0.004$ & $0.93 \pm 0.003$ & $0.90 \pm 0.01$ \\
Simplified-SimCLR & 0.3 & $0.94 \pm 0.006$ & $0.92 \pm 0.007 $ & $0.57 \pm 0.019$ \\
Simplified-SimSiam & 0.1 & $0.94 \pm 0.007$ & $0.87 \pm 0.08$ & $0.47 \pm 0.07$ \\
Simplified-SimSiam & 0.2 & $0.93 \pm 0.008$ & $0.82 \pm 0.01$ & $0.45 \pm 0.005$ \\
Simplified-SimSiam & 0.3 & $0.678 \pm 0.2$ & $0.45 \pm 0.01$ & $0.37 \pm 0.03$ \\
\bottomrule
\end{tabular}
\caption{\label{tab:table-simsiam-simclr} Comparison of the cosine values learnt by the simplified-SimCLR and simplified-SimSiam. Pr(sparse) indicates the probability $Pr(\bm{z}_i = \pm 1), i \in [d]$ in the sparse coding vector $\bm{z}$. We report mean $\pm$ std. deviation over 5 runs. $\uparrow$ symbol indicates that the higher value is better for the associated metric. Note that we sample the diagonal entries in random masks from Bernoulli($0.1$) for these experiments.} 
\end{table*}

\begin{figure}[!htbp]
  \centering
  \begin{minipage}[b]{0.45\textwidth}
    \centering
    \includegraphics[width=0.75\textwidth]{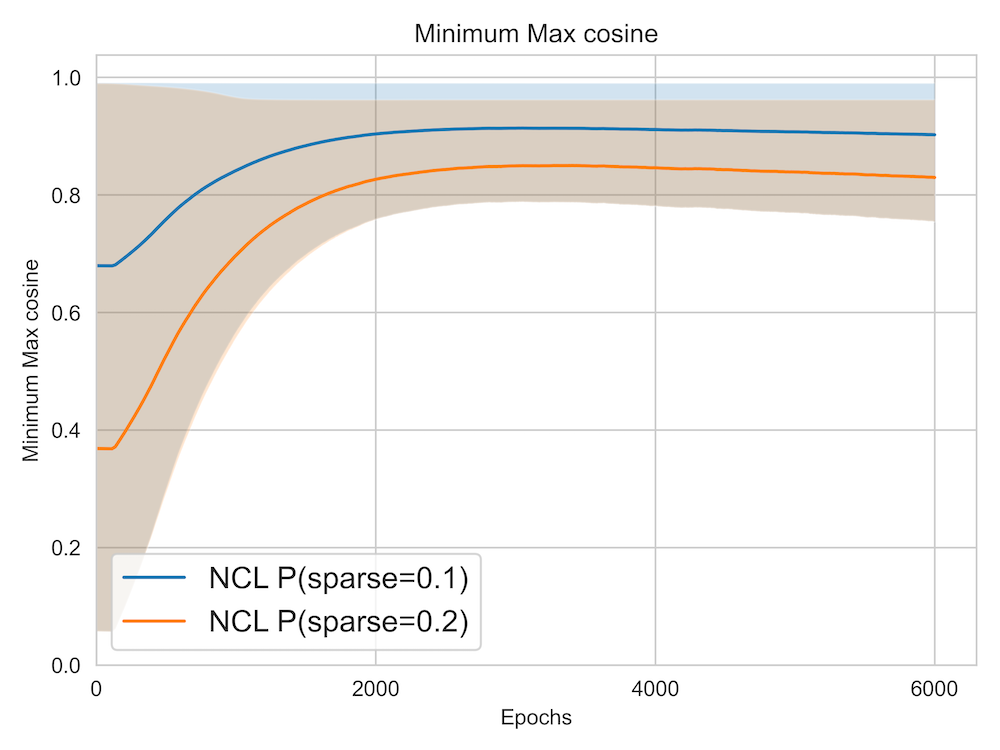}
  \end{minipage}
  \hfill
  \begin{minipage}[b]{0.45\textwidth}
    \centering
    \includegraphics[width=0.75\textwidth]{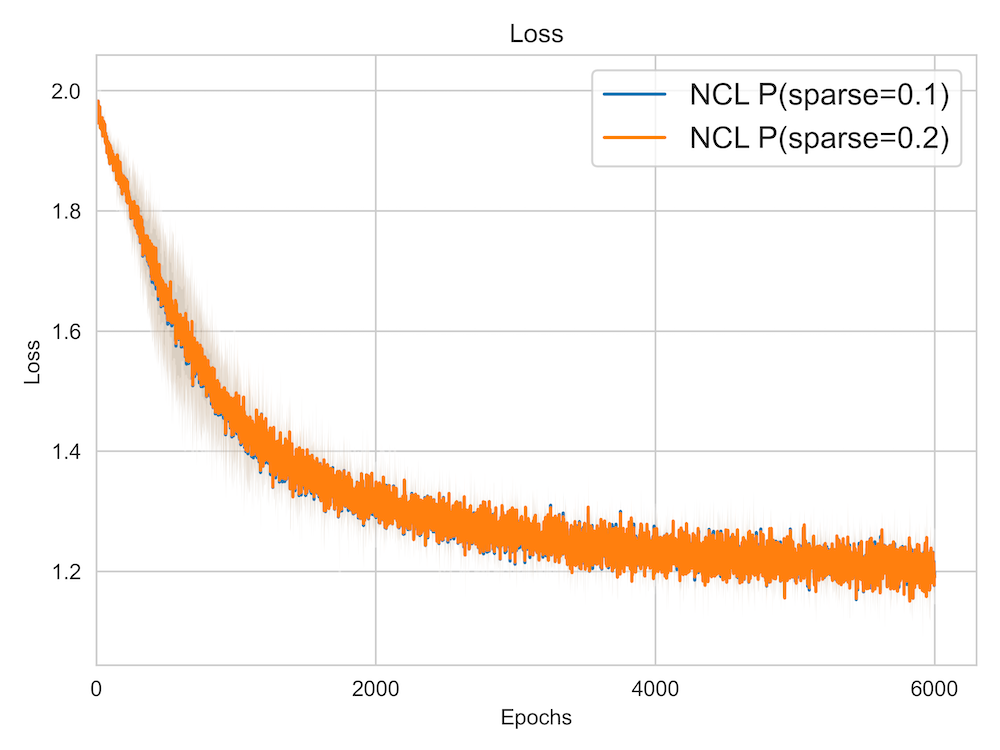}
  \end{minipage}
  \caption{(NCL [Eq. \ref{eq:loss_non-CL}] with warm-start and row normalized encoders $\bW^o$, $\bW^t$ works even in the absence of a prediction head) Minimum Max-Cosine (left) and Loss (right) curves for non-contrastive loss (NCL) with warm-started linear encoder. Reported numbers are averaged over 3 different runs. The warm start parameter $c=2$ and probability of random masking is $0.9$}
 \label{fig:ncl-row-norm-warm-start}
\end{figure}

\paragraph{NCL with warm start, and row or column normalized encoder learns good representations} Another interesting observation from our experiments is that we can skip a linear predictor if the encoder is sufficiently "warm-started", and the rows or columns of the encoder $\bW^o$ and $\bW^t$ are normalized after every gradient descent update. In Figure \ref{fig:ncl-row-norm-warm-start}, where we normalize the rows of encoder, we observe that \textit{Minimum Max-Cosine} increases on average from around $0.62$ to $0.85$ as the training proceeds. We include some additional results with column normalization of the weights of encoder in Appendix \ref{sec:appendix-training-dynamics}. Note that we also normalize the outputs of the encoder in these experiments.

\paragraph{Robustness across architectures and loss functions} In the previous subsections, we used a relatively simple architecture for contrastive and non-contrastive learning. Our empirical observations in the previous sections generalize to the cases where we use a slightly more complex architecture and a different variant of loss. Table \ref{tab:table-simsiam-simclr} lists the cosine values learnt by two architectures  \textit{Simplified-SimCLR} and \textit{Simplified-Simsiam}. \textit{Simplified-SimCLR} uses a linear projector in addition to a linear encoder, and optimizes the normalized temperature-scaled cross entropy loss (\cite{chen2020simple}). \textit{Simplified-Simsiam} architecture uses a linear predictor, and minimizes negative cosine similarity as proposed in (\cite{chen2021exploring}). The encoders of both the architectures are randomly initialized. Both the architectures use batch-normalization layer and ReLU activation after the encoder. We observe that \textit{Simplified-SimCLR} successfully recovers the  ground-truth support but \textit{Simplified-Simsiam} fails to do so, as evident by its low value of \textit{Minimum Max-Cosine}. This is consistent with our observations in previous subsections.

%% file: sections/conclusion.tex
\vspace{-2mm}
\section{Conclusion}
\vspace{-3mm}
In this work, we present some interesting theoretical results that highlight some fundamental differences in the representations learnt by contrastive learning and non-contrastive learning loss objectives in the sparse coding model setting, provided that the encoder architecture is fixed. We use a simple dual network architecture with ReLU activation. We theoretically prove that in this setting, non-contrastive  loss objective has an ample amount of non-collapsed global minima that might not learn the correct ground truth features and therefore fail to recover the correct ground truth dictionary matrix. In contrast, optimizing the contrastive loss objective guarantees recovery of the correct ground truth dictionary matrix. We provide additional empirical  results which show that even non-contrastive training process cannot avoid these bad non-collapsed global minima. We then empirically show that using  warm-start and a linear predictor aids non-contrastive loss to learn the correct ground truth representations. While we worked in a relatively simple setting, we unearthed some fundamental key differences in the quality of representations learnt by contrastive and non-contrastive loss objectives. We hope that these results will motivate further studies to understand the qualitative and quantitative differences in representations learnt by contrastive and non-contrastive loss objectives.

%% file: sections/appendices/landscape_analysis.tex
\section{LANDSCAPE ANALYSIS}
\label{sec:appendix-landscape}

\input{sections/appendices/ident-one-hot-encoded}
\input{sections/appendices/contrastive-k-sparse-gen-M}

\input{sections/appendices/empirical_results_landscape}

%% file: sections/appendices/ident-one-hot-encoded.tex
\subsection{Proof for Theorem~\ref{thm:byolbad}}

\begin{theorem*}[Landscape of contrastive and non-contrastive loss (Theorem~\ref{thm:byolbad} restated)]
Let the data generating process and network architecture be specified as in Section \ref{sec:setup}, 
and consider the setting $d=p=m, \bm{M} = \mathbf{I}$, and $\sigma^2_0 = 0$.
Moreover, let the latent vectors $\{\bm{z}_j \in \mathbb{R}^{d \times 1}\}_{j=1}^d$ be chosen by a uniform distribution over 1-sparse vectors. 

Then, we have:
\begin{itemize}
    \item[(a)] 
    $\mathcal{U}_\ge \subseteq \argmin_{W \in \mathcal{U}} L_{\text{non-CL}}(W, W)$
    \item[(b)] $\argmin_{W \in \mathcal{U}} L_{\text{CL}}(W, W)$ 
    is the set of permutation matrices. 
\end{itemize}
\end{theorem*}

\label{sec:proof:non-CL}
\paragraph{Notation}: Recall that $\mathcal{U}:= \{V \in \mathbb{R}^{m \times p}: \|V_{*j}\|_2 = 1\}$; $\mathcal{U}_{\geq}:=\{V \in \mathcal{U}, V \geq 0\}$ 
We will also denote by $\bm{e}_j \in \mathbb{R}^{p \times 1}$ the vector which is one at $j$-th entry and zero elsewhere.

\begin{proof}
We proceed to claim (a) first. 
By the definition of $L_{non-CL}$, we have
\begin{align*}
    \mbox{argmin}_{\bm{W}\in\mathcal{U}} L_{non-CL}(\bm{W}, \bm{W}) &= \mbox{argmin}_{\bm{W}\in\mathcal{U}}    \sum_{j\in[d]} \left(2 - 2 \mathbb{E}_{\bm{D}_1, \bm{D}_2}\langle Relu(\bm{W} \bm{D}_1 \bm{e}_{j}) ,  Relu(\bm{W}\bm{D}_2 \bm{e}_{j})\rangle\right) \\ 
    &= \sum_{j\in[d]}\mbox{argmax}_{\bm{W}\in\mathcal{U}} \mathbb{E}_{\bm{D}_1, \bm{D}_2}\langle Relu(\bm{W} \bm{D}_1 \bm{e}_{j}) ,  Relu(\bm{W}\bm{D}_2 \bm{e}_{j}) \rangle \\
\end{align*}

Furthermore, we have 
\begin{align*}
    \langle Relu(\bm{W} \bm{D}_1 \bm{e}_{j}) ,  Relu(\bm{W}\bm{D}_2 \bm{e}_{j}) \rangle  &= \bm{D}_{1,jj} \bm{D}_{2,jj} \langle ReLU(\bm{W}_{*j}), ReLU(\bm{W}_{*j}) \rangle \\ 
    &\leq^{(1)} \bm{D}_{1,jj} \bm{D}_{2,jj} \|\bm{W}_{*j}\|^2 \\
    &= \bm{D}_{1,jj} \bm{D}_{2,jj} 
\end{align*}
where (1) follows since $(ReLU(x))^2 \leq x^2, \forall x \in \mathbb{R}$ and the last equality follows since for any $\bm{W} \in \mathcal{U}$, we have $\|\bm{W}_{*j}\|=1$.  
Hence, 
\begin{align*}
\max_{\bm{W}\in\mathcal{U}} \sum_{j\in[d]} \mathbb{E}_{\bm{D}_1, \bm{D}_2}\langle Relu(\bm{W} \bm{D}_1 \bm{e}_{j}) ,  Relu(\bm{W}\bm{D}_2 \bm{e}_{j}) \rangle &\leq \max_{\bm{W}\in\mathcal{U}} \sum_{j\in[d]} \mathbb{E}_{\bm{D}_1, \bm{D}_2} \bm{D}_{1,jj} \bm{D}_{2,jj} =
d \alpha^2
\end{align*}
Moreover, for any $\bm{W} \geq 0$, the inequality (1) is an equality, as $Relu(x) = x$ --- thus any such $\bm{W}$ is a maximum of the objective, which is what we wanted to show. 

\end{proof}

%% file: sections/appendices/contrastive-k-sparse-gen-M.tex
Next, we proceed to (b). 

Recall our definition of contrastive loss:
\begin{align}
    &L_{CL}(\bW) := - \mathbb{E}\left[ \log \frac{\exp\{\tau S^{+} \} }{ \exp\{\tau   S^{+} \} + \sum_{\bm{x}' \in \mathbb{B}} \exp\{\tau S^{-} \}} \right].\nonumber
\end{align}
Assume that $|\mathbb{B}|\to \infty$, according to law of large numbers, we have 
\begin{align}
    &L_{CL}(\bW) \approx - \mathbb{E}\left[ \log \frac{\exp\{\tau S^{+} \} }{ \exp\{\tau   S^{+} \} +  B\mathbb{E}\exp\{\tau S^{-} \}} \right],\nonumber
\end{align}
where $B:=|\mathbb{B}|$.

Plug in the definition of $S^{+}$ and $S^{-}$, we have
\begin{align}\label{simclrloss-old}
    &L_{CL}(\bW) \approx -\frac{1}{d}\sum_{i \in [d]} \mathbb{E}_{D_1, D_2} \left[ \log \left[ \frac{\exp\{\tau\langle ReLU(\bm{W} \bm{D}^i_1 \bm{e}_{i}),  ReLU(\bm{W} \bm{D}^i_2 \bm{e}_{i})\} \rangle }{ \sum_{j \in [d]}\exp\{\tau  \langle ReLU(\bm{W} \bm{D}^i_1 \bm{e}_{i}),  ReLU(\bm{W} \bm{D}^j_2 \bm{e}_j) \rangle \}} \right] \right] +\log{\frac{B}{d}}
\end{align}

We can drop the constant $\log{\frac{B}{d}}$ and $\frac{1}{d}$, and consider the surrogate loss function
\begin{equation} \label{simclrloss-genm}
    \widetilde{L}_{CL}(\bm{W}) := -\sum_{i \in [d]} \mathbb{E}_{D_1, D_2} \left[ \log \left[ \frac{\exp\{\tau\langle ReLU(\bm{W} \bm{D}^i_1 \bm{e}_{i}),  ReLU(\bm{W} \bm{D}^i_2 \bm{e}_{i})\} \rangle }{ \sum_{j \in [d]}\exp\{\tau  \langle ReLU(\bm{W} \bm{D}^i_1 \bm{e}_{i}),  ReLU(\bm{W} \bm{D}^j_2 \bm{e}_j) \rangle \}} \right] \right].
\end{equation}
Note that the minimizer of the $L_{CL}$ is exactly the same as those of $\widetilde{L}_{CL}$.

Denote 
\begin{align*}
    &\mathbb{A}^{+}_{i} : = \exp\{\tau  \langle ReLU(\bm{W} \bm{D}^i_1 \bm{e}_{i}),  ReLU(\bm{W} \bm{D}^i_2 \bm{e}_{i}) \rangle\}, \\
    &\mathbb{A}^{-}_i : = \sum_{j\neq i } \exp\{\tau  \langle ReLU(\bm{W} \bm{D}^i_1 \bm{e}_{i}),  ReLU(\bm{W} \bm{D}^j_2 \bm{e}_j) \rangle \},
\end{align*}
then the term inside the summation of \eqref{simclrloss-genm} corresponding to $i$ can be rewritten as
\begin{equation*}
  \mathbb{B}_i := \log \frac{1}{1+    \mathbb{A}^{-}_i \big / \mathbb{A}^{+}_i}.
\end{equation*}

First, we note that again we have: 
\begin{align}
    \langle Relu(\bm{W} \bm{D}_1 \bm{e}_{i}) ,  Relu(\bm{W}\bm{D}_2 \bm{e}_{i}) \rangle  &= \bm{D}_{1,ii} \bm{D}_{2,ii} \langle ReLU(\bm{W}_{*i}), ReLU(\bm{W}_{*i}) \rangle \nonumber \\ 
    &\leq^{(1)} \bm{D}_{1,ii} \bm{D}_{2,ii} \|\bm{W}_{*i}\|^2 \label{eq:nonnegative}\\
    &= \bm{D}_{1,ii} \bm{D}_{2,ii} \label{eq:mainineq} 
\end{align}
where (1) follows since $(ReLU(x))^2 \leq x^2, \forall x \in \mathbb{R}$ and the last equality follows since for any $\bm{W} \in \mathcal{U}$, we have $\|\bm{W}_{*i}\|=1$.

Additionally, since $ReLU(x) \geq 0$ we have 
\begin{equation}
     \langle ReLU(\bm{W} \bm{D}^i_1 \bm{e}_{i}),  ReLU(\bm{W} \bm{D}^j_2 \bm{e}_j) \rangle \geq 0 
     \label{eq:orthogonal}
\end{equation}

Let us denote $g(\bm{D}^i_1, \bm{D}^i_2):= \exp\{\tau \left(\sum_j (\bm{D}^i_1)_{jj}\right) \left(\sum_j (\bm{D}^i_2)_{jj} \right)\}$. We have by \eqref{eq:mainineq} 
\begin{equation}
\widetilde{L}_{CL}(\bm{W}) \geq -\sum_{i \in [d]}\mathbb{E}_{D_2, D_2}\log\left(\frac{1}{1 + (d-1)/g(\bm{D}^i_1, \bm{D}^i_2)}\right)
\label{eq:boundmin}
\end{equation}

We claim that the only $\bm{W}$ for which \eqref{eq:boundmin} is satisfied with an equality (thus, they are minima of $L_{CL}$) are $\bm{W} = \bm{P}$, for a permutation matrix $\bm{P}$. 

First, we show that such $\bm{W}$ have to satisfy $\bm{W}_{*i} \geq 0, \forall i$, where the inequality is understood to apply entrywise. 
Indeed, for the sake of contradiction, assume otherwise. Then, consider an $i \in [d]$, and a mask $\bm{D}^i_1$, s.t. $\bm{D}^i_{1,ii}=1$ . (Note, such a mask has a non-zero probability of occurring.) For this choice of $i, \bm{D}^i_1$, we have 
\begin{align}
    \|ReLU(\bm{W} \bm{D}^i_1 \bm{e}_i)\| = \|ReLU(\bm{W}_{*i})\| < \|\bm{W}_{*i}\|
\end{align}
since at least one element of $\bm{W}_{*i}$ is negative. Thus, \eqref{eq:nonnegative} is a strict inequality, and thus \eqref{eq:boundmin} cannot yield an equality. 

Next, we show that  $\forall i \neq j, \langle \bm{W}_{*i}, \bm{W}_{*j} \rangle = 0$ (i.e. the vectors $\bm{W}_{*i}$ are orthogonal). 
Again, we proceed by contradiction. Let us assume that there exist $i,j$, s.t. 
$\langle \bm{W}_{*i}, \bm{W}_{*j} \rangle > 0$. (Note, as we concluded before, all coordinates of $\bm{W}_{*i}$ have to be nonnegative, so if the above inner product is non-zero, it has to be nonnegative.) 
Consider a pair of $i\neq j \in [d]$, and two masks $\bm{D}^i_1, \bm{D}^j_2$, s.t.  $\bm{D}^i_{1,ii}=\bm{D}^j_{2,jj}=1$. (Again, such masks occur with non-zero probability.) 
For such choices of $i,j, \bm{D}^i_1,\bm{D}^j_2$, we have 
\begin{align*}
     \langle ReLU(\bm{W} \bm{D}^i_1 \bm{e}_{i}),  ReLU(\bm{W} \bm{D}^j_2 \bm{e}_j) \rangle =
    \langle \bm{W}_{*i}, \bm{W}_{*j} \rangle > 0
\end{align*}
 which implies that \eqref{eq:orthogonal} cannot be satisfied with an inequality. 
Thus, again \eqref{eq:boundmin} cannot yield an equality.

Thus, we concluded that, in order to achieve the minimum of $\widetilde{L}_{CL}$ over $\bm{W}\in\mathcal{U}$, the vectors $\bm{W}_{*i}$ have to be nonnegative and orthonormal for all $i$. Then, we claim this means $\bm{W}$ is a permutation matrix. Indeed, assume otherwise --- i.e. that there exists some row, for which two columns $i,i'$ has a non-zero element in that row. But then, these columns have a stricly positive inner product, so cannot be orthogonal. Thus, $\bm{W}$ has to be a permutation matrix. 

On the other hand, for $\bm{W}$ a permutation matrix, \eqref{eq:orthogonal} and \eqref{eq:nonnegative} is an equality, and consequently \eqref{eq:boundmin} is an equality. Thus, any permutation matrix is a minimum of $L_{CL}$. Altogether, this concludes the proof of the theorem.

%% file: sections/appendices/empirical_results_landscape.tex
\subsection{Additional Empirical results for Landscape analysis}
\begin{figure}[!h]
  \centering
  \begin{minipage}[b]{0.23\textwidth}
    \includegraphics[width=\textwidth]{sections/plots/min-max-cosine-k-sparse-bn-relu-cl-vs-ncl.png}
  \end{minipage}
  \hfill
  \begin{minipage}[b]{0.23\textwidth}
    \includegraphics[width=\textwidth]{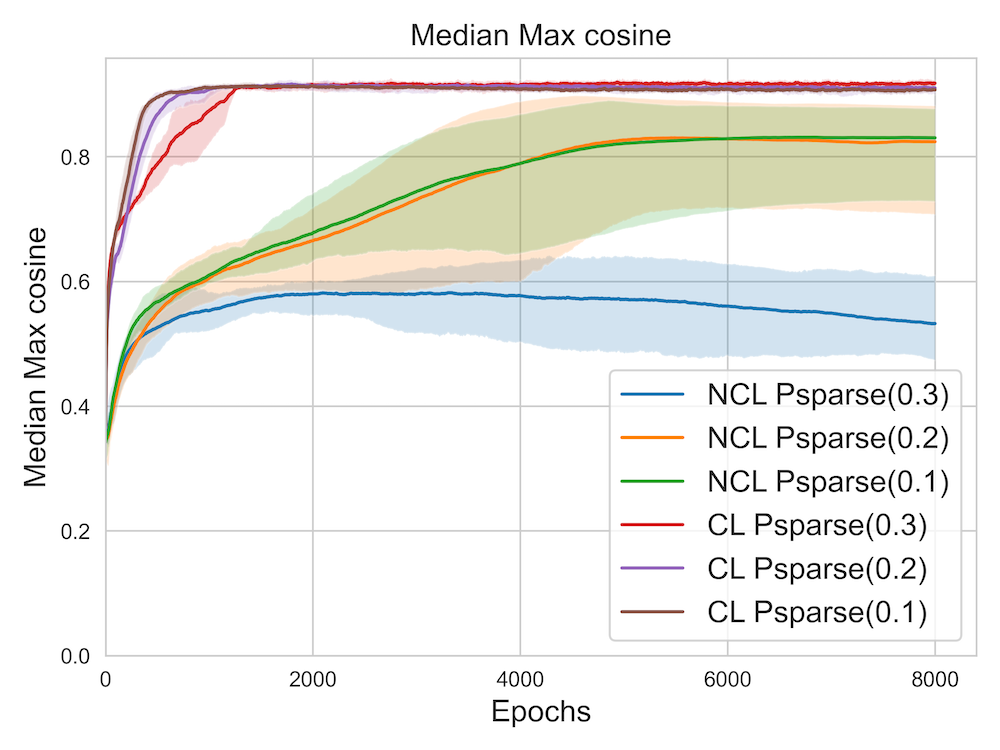}
  \end{minipage}
  \hfill
  \begin{minipage}[b]{0.23\textwidth}
    \includegraphics[width=\textwidth]{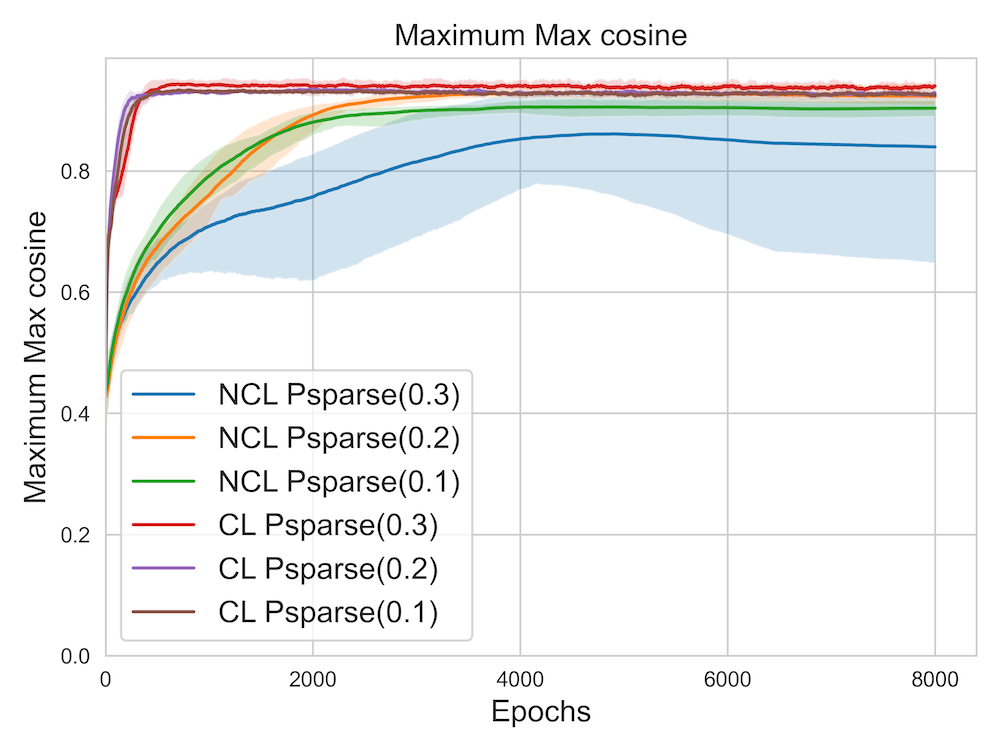}
  \end{minipage}
  \hfill
  \begin{minipage}[b]{0.23\textwidth}
    \includegraphics[width=\textwidth]{sections/plots/loss-k-sparse-bn-relu-cl-vs-ncl.png}
  \end{minipage}
  \\~\\
  \caption{(NCL vs CL with k-sparse latents - Overparametrized network) (left-to-right) Minimum Max-Cosine,  Median Max-Cosine, Maximum Max-Cosine, and Loss curves for non-contrastive loss (NCL) and contrastive loss (CL) on an architecture with a randomly initialized linear encoder. We normalize the representations before computing the loss and use a symmetric ReLU after the linear encoder. Reported numbers are averaged over 5 different runs. The shaded area represents the maximum and the minimum values observed across those 5 runs. We use p = 50, d=10, m=10.}  \label{fig:landscape-p-50-m10-d10-k-sparse-overparam}
\end{figure}

\begin{figure}[!h]
  \centering
  \begin{minipage}[b]{0.23\textwidth}
    \includegraphics[width=\textwidth]{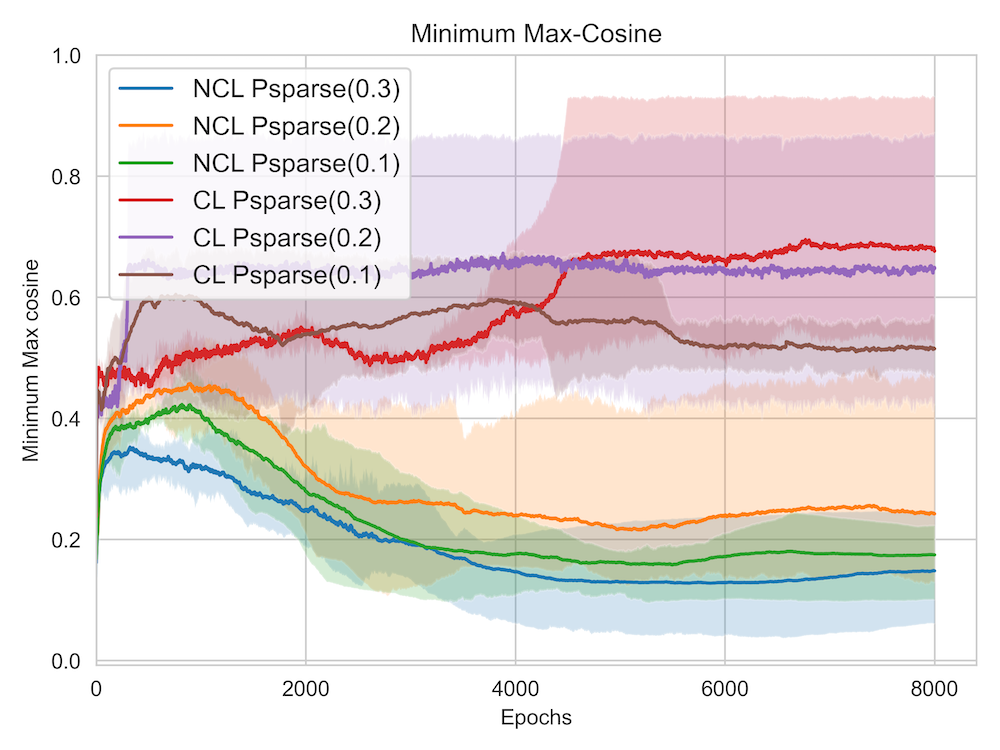}
  \end{minipage}
  \hfill
  \begin{minipage}[b]{0.23\textwidth}
    \includegraphics[width=\textwidth]{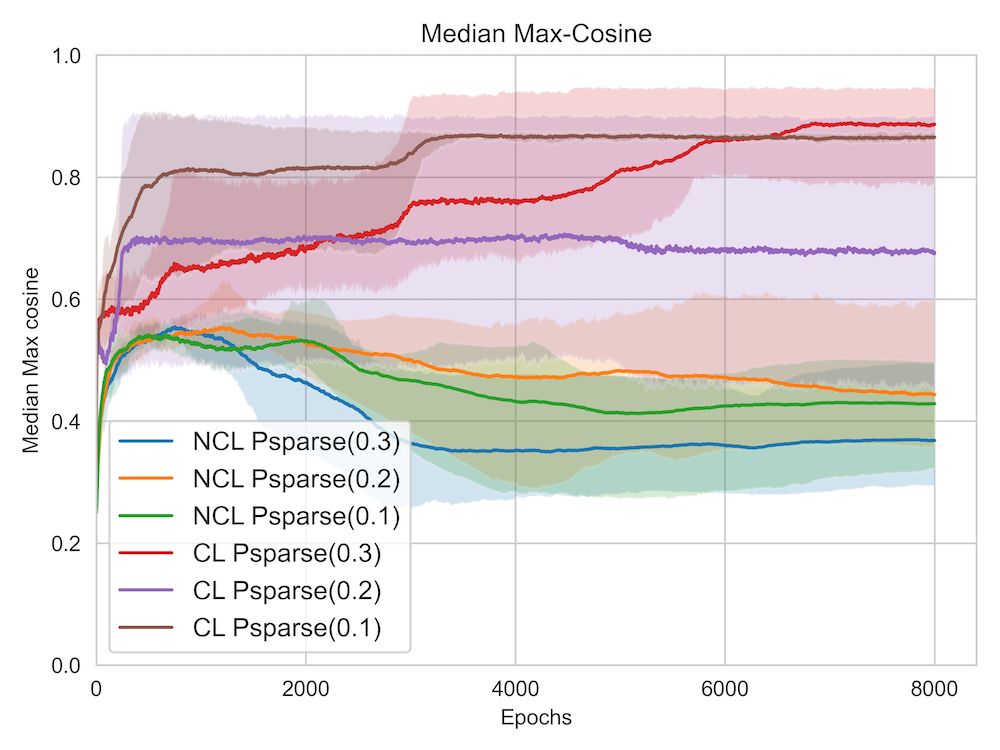}
  \end{minipage}
  \hfill
  \begin{minipage}[b]{0.23\textwidth}
    \includegraphics[width=\textwidth]{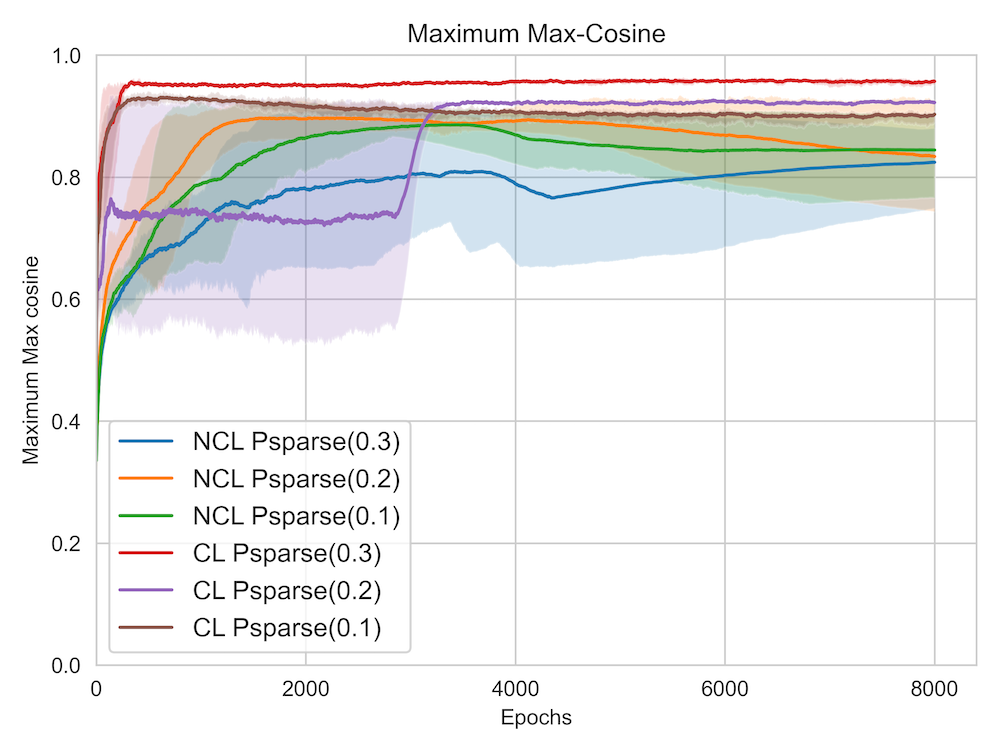}
  \end{minipage}
  \hfill
  \begin{minipage}[b]{0.23\textwidth}
    \includegraphics[width=\textwidth]{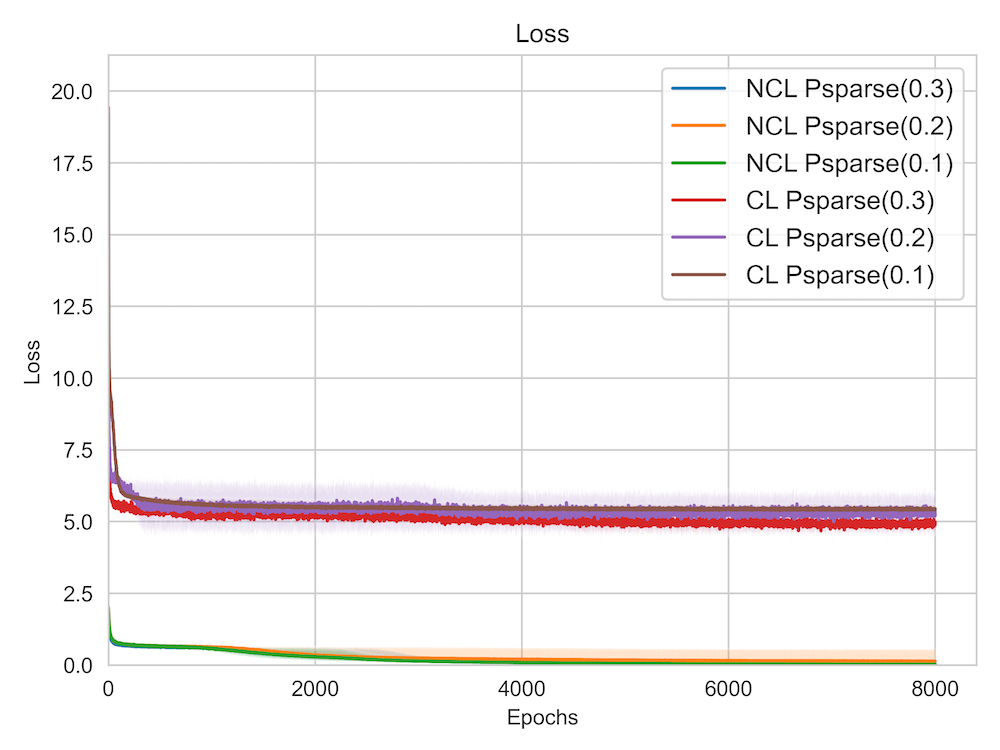}
  \end{minipage}
    \\~\\
  \caption{(NCL vs CL with k-sparse latents) (left-to-right) Minimum Max-Cosine,  Median Max-Cosine, Maximum Max-Cosine, and Loss curves for non-contrastive loss (NCL) and contrastive loss (CL) on an architecture with a randomly initialized linear encoder. We normalize the representations before computing the loss and use a symmetric ReLU after the linear encoder. Reported numbers are averaged over 5 different runs. The shaded area represents the maximum and the minimum values observed across those 5 runs. We use p=50, d=10, m=10.}  \label{fig:landscape-p-50-m10-d10-k-sparse-z}
\end{figure}

\begin{figure}[!h]
  \centering
  \begin{minipage}[b]{0.23\textwidth}
    \includegraphics[width=\textwidth]{sections/plots/ncl-ml-mmc-k-sparse.png}
  \end{minipage}
  \hfill
    \begin{minipage}[b]{0.23\textwidth}
    \includegraphics[width=\textwidth]{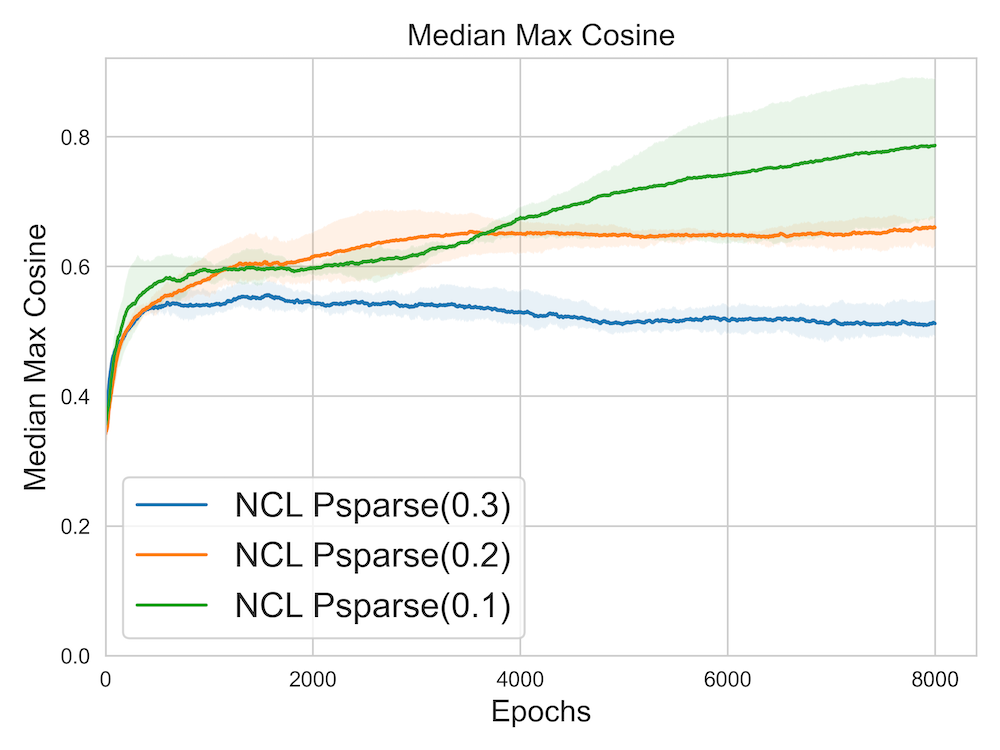}
  \end{minipage}
  \hfill
  \begin{minipage}[b]{0.23\textwidth}
    \includegraphics[width=\textwidth]{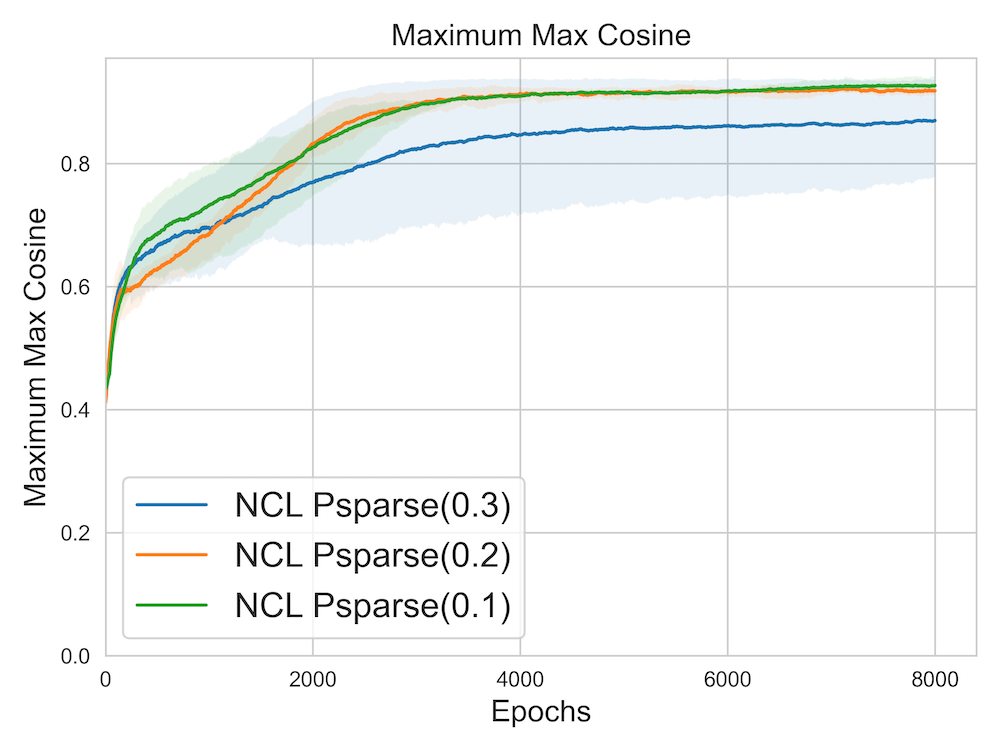}
  \end{minipage}
  \hfill
  \begin{minipage}[b]{0.23\textwidth}
    \includegraphics[width=\textwidth]{sections/plots/ncl-ml-loss-k-sparse.png}
  \end{minipage}
    \\~\\
  \caption{(NCL with 2 layered encoder does not work) (left to right) Minimum Max-Cosine, Median Max-Cosine, Maximum Max-Cosine and Loss curves for non-contrastive loss (NCL) with a two-layered linear encoder with batch-normalization and symmetric ReLU activation. Psparse indicates $Pr(\bm{z}_i = \pm 1), i \in [d]$ in the sparse coding vector $\bm{z}$. Reported numbers are averaged over 5 different runs. The shaded area represents the maximum and the minimum values observed across those 5 runs. We use p = 50, d=10, m=50.}
\end{figure}

\subsubsection{Non contrastive loss collapses to bad minima}

The results have been summarized in Table \ref{tab:table-basic-no-relu-no-bn}. We observe that the non-contrastive learning model \textit{NCL-basic} fails to achieve good values for the metric \textit{Minimum max-cosine} while \textit{NCL-linear} fails to achieve good values on all the three metrics. In practice, we observe a collapse in \textit{NCL-linear} as the weights of the encoder are driven to 0. This indicates that the simple encoder when training with non-CL loss objective fails to recover the correct ground-truth dictionary matrix $\bm{M}$. On the other hand, we observe that the encoders trained with CL loss objective  \textit{CL-linear} and \textit{CL-basic} have high values for all the three metrics which indicates that they succeed in recovering most of the ground truth support. This raises a question \textit{Does introducing more fully connected layers in the encoder help non-CL loss objective to learn better representations?} We test this in our model \textit{NCL-basic-2L} that uses a two-layered encoder with batch norm and symmetric ReLU. 

\begin{table*}[th!]
\centering
\begin{tabular}{cccccccc}
\toprule
Model & BN & SymReLU & Pr(sparse) & \multicolumn{1}{p{2.5cm}}{\centering Maximum \\ Max cosine $\uparrow$} & \multicolumn{1}{p{2.5cm}}{\centering Median \\ Max cosine $\uparrow$ }  & \multicolumn{1}{p{2.5cm}}{\centering Minimum \\ Max cosine $\uparrow$}  \\
\midrule
NCL-linear & $\times$ & $\times$ & 0.1 & $0.12 \pm 0.03$ & $0.10 \pm 0.01$ & $0.08 \pm 0.1$ \\
NCL-linear & $\times$ & $\times$ & 0.2 & $0.11 \pm 0.02$ & $0.09 \pm 0.03 $ & $0.08 \pm 0.3$ \\
NCL-linear & $\times$ & $\times$ & 0.3 & $0.31 \pm 0.006$ & $0.29 \pm 0.007 $ & $0.28 \pm 0.002$ \\
\midrule
CL-linear & $\times$ & $\times$ & 0.1 & $0.77 \pm 0.01$ & $0.69 \pm 0.01$ & $0.58 \pm 0.01$ \\
CL-linear & $\times$ & $\times$ & 0.2 & $0.78 \pm 0.06$ & $0.68 \pm 0.02$ & $0.59 \pm 0.01$ \\
CL-linear & $\times$ & $\times$ & 0.3 & $0.78 \pm 0.04$ & $0.67 \pm 0.01$ & $0.6 \pm 0.03$ \\
\midrule
NCL-basic & $\checkmark$ & $\checkmark$ & 0.1 & $0.9 \pm 0.01 $ & $0.83 \pm 0.06$ & $0.52 \pm 0.06$ \\
NCL-basic & $\checkmark$ & $\checkmark$ & 0.2 & $0.92 \pm 0.01$ & $0.82 \pm 0.07$  & $0.51 \pm 0.03$ \\
NCL-basic & $\checkmark$ & $\checkmark$ & 0.3 & $0.84 \pm 0.11$ & $0.58 \pm 0.04$ & $0.42 \pm 0.06$ \\
\midrule
NCL-basic-2L & $\checkmark$ & $\checkmark$ & 0.1 & $0.94 \pm 0.01$ & $0.91 \pm 0.01$ & $0.66 \pm 0.21$ \\
NCL-basic-2L & $\checkmark$ & $\checkmark$ & 0.2 & $0.94 \pm 0.009$ & $0.92 \pm 0.002$ & $0.51 \pm 0.11$ \\
NCL-basic-2L & $\checkmark$ & $\checkmark$ & 0.3 & $0.83 \pm 0.16$ & $0.48 \pm 0.04$ & $0.36 \pm 0.01$ \\
\midrule
CL-basic & $\checkmark$ & $\checkmark$ & 0.1 & $0.93 \pm 0.007$ & $0.91 \pm 0.005$  & $0.88 \pm 0.002$ \\
CL-basic & $\checkmark$ & $\checkmark$ & 0.2 & $0.93 \pm 0.002$  & $0.91 \pm 0.009$ & $0.79 \pm 0.14$ \\
CL-basic & $\checkmark$ & $\checkmark$ & 0.3 & $0.93 \pm 0.01$ & $0.86 \pm 0.07$ & $0.76 \pm 0.18$\\
\bottomrule
\end{tabular}
\caption{\label{tab:table-basic-no-relu-no-bn} Comparison of the cosine values learnt by contrastive vs non-contrastive losses. The column BN indicates the presence of a batch normalization layer after linear layer in the encoder. SymReLU indicates whether the encoder uses symmetric ReLU activation. Pr(sparse) indicates the probability $Pr(\bm{z}_i = \pm 1), i \in [d]$ in the sparse coding vector $\bm{z}$. We report mean $\pm$ std. deviation over 5 runs. $\uparrow$ symbol indicates that the higher value is better for the associated metric.} 
\end{table*}

%% file: sections/appendices/training_dynamics.tex
\section{TRAINING DYNAMICS}

In this section, we include additional empirical results that elucidate the training dynamics of contrastive and non-contrastive training. 

\label{sec:appendix-training-dynamics}
\input{sections/appendices/empirical_results_dynamics}

\section{THEORETICAL RESULTS FOR TRAINING DYNAMICS}

\input{sections/appendices/linear_network}

\subsection{Proof for ReLU network with normalization}
\label{sec:proof:relu_no_pred}

\begin{theorem*}[ReLU network with normalization, warm start, Theorem~\ref{thm:relu_no_pred} restated]
    Under Assumptions~\ref{assumption:symmetric_bernoulli_latent}, \ref{assumption:warm_start_symmetric_bernoulli}, \ref{assumption:bias_symmetric_bernoulli}, running:
    \item Repeat until both $\mW^o$ and $\mW^t$ converges:
    \begin{itemize}
        \item Repeat until $\mW^o$ converges: \\
            \[ \mW^o \leftarrow \text{normalize}(\mW^o - \eta \nabla_{\mW^o} L(\mW^o, \mW^t)) \]
        \item Repeat until $\mW^t$ converges: \\
            \[ \mW^t \leftarrow \text{normalize}(\mW^t - \eta \nabla_{\mW^t} L(\mW^o, \mW^t)) \]
    \end{itemize}
    will make $\mW^o$ and $\mW^t$ both converge to $I$
    in which \texttt{normalize} means row-normalization, i.e. 
    $\| \mW^o_{i*} \|_2 = 1$ and $\| \mW^t_{i*} \|_2 = 1$ for each $i$.
\end{theorem*}

\begin{proof}
    To start with, we simplify the expression of $L(\mW^o, \mW^t)$ under the above assumptions.
    \begin{align}
        L(\mW^o, \mW^t) &= 2 - 2 \mathbb{E}_{\vx, \mD_1, \mD_2} [\sum_{i=1}^d \text{SReLU}_{\vb^o}(\mW^o_{i*} \mD_1 \vx) \cdot  \text{SReLU}_{\vb^t}(\mW^t_{i*} \mD_2 \vx) ]  \\
        \label{eq:simplify_loss_symmetric_bernoulli}
        &= 2 - 2 \mathbb{E}_{\vx, \mD_1, \mD_2} [\sum_{i=1}^d \text{SReLU}_{b^o_i}(\mW^o_{i*} \mD_1 (\mM \vz + \bm\epsilon)) \cdot  \text{SReLU}_{b^t_i}(\mW^t_{i*} \mD_2 (\mM \vz + \bm\epsilon)) ]
    \end{align}
    
    Since $\mM = \mI$, the SReLU term becomes 
    \begin{align*}
        &\quad \text{SReLU}_{b^o_i}(\mW^o_{i*} \mD_1 (\mM \vz + \bm\epsilon)) \\
        &= \text{SReLU}_{b^o_i}(\mW^o_{i*} \mD_1 (\vz + \bm\epsilon)) \\
        &= \text{SReLU}_{b^o_i}(\sum_{j=1}^d \emW^o_{ij} \emD_{1,jj} (\evz_j + \bm\epsilon_j))
    \end{align*}
    
    By Assumption~\ref{assumption:warm_start_symmetric_bernoulli} (warm-start), 
    the above becomes 
    \begin{align*}
        &\quad \text{SReLU}_{b^o_i}(\sum_{j=1}^d (\emI_{ij} + \bm\Delta^o_{ij}) \emD_{1,jj} (\evz_j + \bm\epsilon_j)) \\
        &= \text{SReLU}_{b^o_i}((1 + \bm\Delta^o_{ii}) \emD_{1,ii} (\evz_i + \bm\epsilon_i) + \sum_{j \ne i} \bm\Delta^o_{ij} \emD_{1,jj} (\evz_j + \bm\epsilon_j)) 
    \end{align*}
    
    By Assumption~\ref{assumption:bias_symmetric_bernoulli} on the bias,
    \[ (1 + \bm\Delta^o_{ii}) \emD_{1,ii} (\evz_i + \bm\epsilon_i) + \sum_{j \ne i} \bm\Delta^o_{ij} \emD_{1,jj} (\evz_j + \bm\epsilon_j) \text{ is } \begin{cases}
            > b^o_i, \quad &\text{if } \emD_{1,ii} = 1 \text{ and } \evz_i = 1 \\ 
            < -b^o_i, \quad &\text{if } \emD_{1,ii} = 1 \text{ and } \evz_i = -1 \\
            \in (-b^o_i, b^o_i), \quad &\text{if } \emD_{1,ii} = 0 \text{ or } \evz_i = 0
        \end{cases}
    \]
    
    Therefore the SReLU term can be simplified to
    \begin{align*}
        &\quad \begin{cases}
            (1 + \bm\Delta^o_{ii}) (1 + \bm\epsilon_i) + \sum_{j \ne i} \bm\Delta^o_{ij} \emD_{1,jj} (\evz_j + \bm\epsilon_j) - b^o_i, \quad &\text{if } \emD_{1,ii} = 1 \text{ and }  \evz_i = 1 \\
            (1 + \bm\Delta^o_{ii}) (-1 + \bm\epsilon_i) + \sum_{j \ne i} \bm\Delta^o_{ij} \emD_{1,jj} (\evz_j + \bm\epsilon_j) + b^o_i, \quad &\text{if } \emD_{1,ii} = 1 \text{ and }  \evz_i = -1 \\
            0, \quad &\text{if } \emD_{1,ii} = 0 \text{ or } \evz_i = 0
            \end{cases} \\
        &=  \emD_{1,ii} \Big[\1_{\evz_i = 1} \big((1 + \bm\Delta^o_{ii}) (1 + \bm\epsilon_i) + \sum_{j \ne i} \bm\Delta^o_{ij} \emD_{1,jj} (\evz_j + \bm\epsilon_j) - b^o_i\big) \\ & \quad + \1_{\evz_i = -1} \big((1 + \bm\Delta^o_{ii}) (-1 + \bm\epsilon_i) + \sum_{j \ne i} \bm\Delta^o_{ij} \emD_{1,jj} (\evz_j + \bm\epsilon_j) + b^o_i \big) \Big]
    \end{align*}
    
    Note that the $\bm\epsilon$ terms of first order disappear after taking expectation over $\bm\epsilon$.
    Moreover, 
    The second-order terms in $\text{SReLU}_{b^o_i}(\mW^o_{i*} \mD_1 (\mM \vz + \bm\epsilon)) \cdot \text{SReLU}_{b^t_i}(\mW^t_{i*} \mD_2 (\mM \vz + \bm\epsilon))$ are contained in
    
    \begin{align*}
        &\quad \emD_{1,ii} \Big[\1_{\evz_i = 1} \big((1 + \bm\Delta^o_{ii}) \bm\epsilon_i + \sum_{j \ne i} \bm\Delta^o_{ij} \emD_{1,jj} \bm\epsilon_j) + \1_{\evz_i = -1} \big((1 + \bm\Delta^o_{ii}) \bm\epsilon_i + \sum_{j \ne i} \bm\Delta^o_{ij} \emD_{1,jj} \bm\epsilon_j \big) \Big] \\
        &\quad \cdot \emD_{2,ii} \Big[\1_{\evz_i = 1} \big((1 + \bm\Delta^t_{ii}) \bm\epsilon_i + \sum_{j \ne i} \bm\Delta^t_{ij} \emD_{2,jj} \bm\epsilon_j) + \1_{\evz_i = -1} \big((1 + \bm\Delta^t_{ii}) \bm\epsilon_i + \sum_{j \ne i} \bm\Delta^t_{ij} \emD_{2,jj} \bm\epsilon_j \big) \Big] \\
        &= \1_{\evz_i = 1} \Big[ \emD_{1,ii} \big((1 + \bm\Delta^o_{ii}) \bm\epsilon_i + \sum_{j \ne i} \bm\Delta^o_{ij} \emD_{1,jj} \bm\epsilon_j \big) \cdot \emD_{2,ii} \big((1 + \bm\Delta^t_{ii}) \bm\epsilon_i + \sum_{j \ne i} \bm\Delta^t_{ij} \emD_{2,jj} \bm\epsilon_j) \Big]  \\
        &\quad + \1_{\evz_i = -1} \Big[ \emD_{1,ii} \big((1 + \bm\Delta^o_{ii}) \bm\epsilon_i + \sum_{j \ne i} \bm\Delta^o_{ij} \emD_{1,jj} \bm\epsilon_j \big) \cdot \emD_{2,ii} \big((1 + \bm\Delta^t_{ii}) \bm\epsilon_i + \sum_{j \ne i} \bm\Delta^t_{ij} \emD_{2,jj} \bm\epsilon_j \big) \Big]
    \end{align*}
    
    After taking expectation over $\bm\epsilon$, the cross-terms $\bm\epsilon_i \bm\epsilon_j (i \ne j)$ also disappear, so the remaining terms are
    
    \begin{align}
        \label{eq:noise_term}
        &\quad \mathbb{E}_{\vz, \bm\epsilon, \mD_1, \mD_2} \Big[ \emD_{1,ii} \emD_{2,ii} \big( \1_{\evz_i = 1}  \big((1 + \bm\Delta^o_{ii}) (1 + \bm\Delta^t_{ii}) \bm\epsilon_i^2 + \sum_{j \ne i} \bm\Delta^o_{ij} \emD_{1,jj} \bm\Delta^t_{ij} \emD_{2,jj} \bm\epsilon_j^2 \big) \big) \nonumber \\
        &\quad + \1_{\evz_i = -1} \big((1 + \bm\Delta^o_{ii}) (1 + \bm\Delta^t_{ii}) \bm\epsilon_i^2 + \sum_{j \ne i} \bm\Delta^o_{ij} \emD_{1,jj} \bm\Delta^t_{ij} \emD_{2,jj} \bm\epsilon_j^2 \big) \Big] \nonumber \\
        &= \alpha^2 \mathbb{E}_{\vz, \bm\epsilon} \Big[ \1_{\evz_i = 1}  \big((1 + \bm\Delta^o_{ii}) (1 + \bm\Delta^t_{ii}) \bm\epsilon_i^2 + \alpha^2 \sum_{j \ne i} \bm\Delta^o_{ij} \bm\Delta^t_{ij} \bm\epsilon_j^2 \big) \nonumber  \\
        &\quad + \1_{\evz_i = -1} \big((1 + \bm\Delta^o_{ii}) (1 + \bm\Delta^t_{ii}) \bm\epsilon_i^2 + \alpha^2 \sum_{j \ne i} \bm\Delta^o_{ij} \bm\Delta^t_{ij} \bm\epsilon_j^2 \big) \Big] \nonumber \\
        &= \alpha^2 \kappa \mathbb{E}_{\bm\epsilon} \Big[ (1 + \bm\Delta^o_{ii}) (1 + \bm\Delta^t_{ii}) \bm\epsilon_i^2 + \alpha^2 \sum_{j \ne i} \bm\Delta^o_{ij} \bm\Delta^t_{ij} \bm\epsilon_j^2 \Big] \nonumber \\
        &= \alpha^2 \kappa \mathbb{E}_{\bm\epsilon}[\bm\epsilon_i^2] ( (1 + \bm\Delta^o_{ii}) (1 + \bm\Delta^t_{ii})  + \alpha^2 \sum_{j \ne i} \bm\Delta^o_{ij} \bm\Delta^t_{ij} ) \nonumber \\
        &= \alpha^2 \kappa \sigma_0^2 ( (1 + \bm\Delta^o_{ii}) (1 + \bm\Delta^t_{ii})  + \alpha^2 \sum_{j \ne i} \bm\Delta^o_{ij} \bm\Delta^t_{ij} ) \quad \text{(by \eqref{eq:data-generation})}
    \end{align}
    
    Meanwhile, a summand in \eqref{eq:simplify_loss_symmetric_bernoulli} without the noise term is
    \begin{align*}
        &\quad \mathbb{E}_{\vz, \bm\epsilon, \mD_1, \mD_2} \left[ \text{SReLU}_{b^o_i}(\mW^o_{i*} \mD_1 (\mM \vz + \bm\epsilon)) \cdot \text{SReLU}_{b^t_i}(\mW^t_{i*} \mD_2 (\mM \vz + \bm\epsilon)) \right] \\
        &= \mathbb{E}_{\vz, \bm\epsilon, \mD_1, \mD_2} \Big[ \emD_{1,ii} \Big(\1_{\evz_i = 1} (1 + \bm\Delta^o_{ii} + \sum_{j \ne i} \bm\Delta^o_{ij} \emD_{1,jj} \evz_j - b^o_i) + \1_{\evz_i = -1} (-1 - \bm\Delta^o_{ii} + \sum_{j \ne i} \bm\Delta^o_{ij} \emD_{1,jj} \evz_j + b^o_i)\Big) \\
        &\quad \cdot \emD_{2,ii} \Big(\1_{\evz_i = 1} (1 + \bm\Delta^t_{ii} + \sum_{j \ne i} \bm\Delta^t_{ij} \emD_{2,jj} \evz_j - b^t_i) + \1_{\evz_i = -1} (-1 - \bm\Delta^t_{ii} + \sum_{j \ne i} \bm\Delta^t_{ij} \emD_{2,jj} \evz_j + b^t_i)\Big) \Big] 
    \end{align*}
    
    which can be simplified to 
    \begin{align*}
        &\mathbb{E}_{\vz, \bm\epsilon, \mD_1, \mD_2} \Big[ \emD_{1,ii} \emD_{2,ii} \cdot \Big( \1_{\evz_i = 1} (1 + \bm\Delta^o_{ii} + \sum_{j \ne i} \bm\Delta^o_{ij} \emD_{1,jj} \evz_j - b^o_i ) (1 + \bm\Delta^t_{ii} + \sum_{j \ne i} \bm\Delta^t_{ij} \emD_{2,jj} \evz_j - b^t_i ) \\
        &\quad + \1_{\evz_i = -1} (-1 - \bm\Delta^o_{ii} + \sum_{j \ne i} \bm\Delta^o_{ij} \emD_{1,jj} \evz_j + b^o_i) (-1 - \bm\Delta^t_{ii} + \sum_{j \ne i} \bm\Delta^t_{ij} \emD_{2,jj} \evz_j + b^t_i) \Big) \Big] 
    \end{align*}
    
    Moreover, note that $\emD_{1,ii}, \emD_{2,ii}, \1_{\evz_i = 1}, \1_{\evz_i = -1}$ each appears only once in the expectation, 
    so by independence, we can further simplify the above to 
    
    \begin{align*}
        &\quad \alpha^2 \Big(\mathbb{E}_{\vz} [ \1_{\evz_i = 1}] \mathbb{E}_{\vz} \Big[ (1 + \bm\Delta^o_{ii} + \sum_{j \ne i} \bm\Delta^o_{ij} \mathbb{E}_{\mD_1} [\emD_{1,jj}] \evz_j - b^o_i ) (1 + \bm\Delta^t_{ii} + \sum_{j \ne i} \bm\Delta^t_{ij} \mathbb{E}_{\mD_2} [\emD_{2,jj}] \evz_j - b^t_i ) \Big] \\
        &\quad + \mathbb{E}_{\vz} [\1_{\evz_i = -1}] \mathbb{E}_{\vz}\Big[ (-1 - \bm\Delta^o_{ii} + \sum_{j \ne i} \bm\Delta^o_{ij} \mathbb{E}_{\mD_1} [\emD_{1,jj}] \evz_j + b^o_i) (-1 - \bm\Delta^t_{ii} + \sum_{j \ne i} \bm\Delta^t_{ij} \mathbb{E}_{\mD_2} [\emD_{2,jj}] \evz_j + b^t_i) \Big] \Big)  
    \end{align*}
    
    which can be written as
    \begin{align*}
        &\alpha^2 \frac{\kappa}{2} \mathbb{E}_{\vz} \Big[ (1 + \bm\Delta^o_{ii} - b^o_i) (1 + \bm\Delta^t_{ii} - b^t_i) + (1 + \bm\Delta^o_{ii} - b^o_i) \sum_{j \ne i} \bm\Delta^t_{ij} \alpha \evz_j + (1 + \bm\Delta^t_{ii} - b^t_i) \sum_{j \ne i} \bm\Delta^o_{ij} \alpha \evz_j\\
        &+ (\sum_{j \ne i} \bm\Delta^o_{ij} \alpha \evz_j) (\sum_{j \ne i} \bm\Delta^t_{ij} \alpha \evz_j) + (-1 - \bm\Delta^o_{ii} + b^o_i) (-1 - \bm\Delta^t_{ii} + b^t_i) \\ 
        & + (-1 - \bm\Delta^o_{ii} + b^o_i) \sum_{j \ne i} \bm\Delta^t_{ij} \alpha \evz_j + (-1 - \bm\Delta^t_{ii} + b^t_i) \sum_{j \ne i} \bm\Delta^o_{ij} \alpha \evz_j + (\sum_{j \ne i} \bm\Delta^o_{ij} \alpha \evz_j) (\sum_{j \ne i} \bm\Delta^t_{ij} \alpha \evz_j) \Big] 
    \end{align*}
    
    By Assumption~\ref{assumption:symmetric_bernoulli_latent} (symmetric Bernoulli latent), 
    $\mathbb{E}_{\vz} [\evz_j] = 0$ and 
    $\mathbb{E}_{\vz} [\evz_i \evz_j] = \begin{cases}
        0, \quad &\text{if } i \ne j \\ 
        \kappa, \quad &\text{if } i = j \\ 
        \end{cases}$ \\
    So the above can be simplified to
    
    \begin{align*}
        &\quad \alpha^2 \frac{\kappa}{2} \Big[ (1 + \bm\Delta^o_{ii} - b^o_i) (1 + \bm\Delta^t_{ii} - b^t_i) + \alpha^2 \kappa \sum_{j \ne i} \bm\Delta^o_{ij} \bm\Delta^t_{ij} + (-1 - \bm\Delta^o_{ii} + b^o_i) (-1 - \bm\Delta^t_{ii} + b^t_i) + \alpha^2 \kappa \sum_{j \ne i} \bm\Delta^o_{ij} \bm\Delta^t_{ij} \Big] \\
        &= \alpha^2 \kappa \Big[ (1 + \bm\Delta^o_{ii} - b^o_i) (1 + \bm\Delta^t_{ii} - b^t_i) + \alpha^2 \kappa \sum_{j \ne i} \bm\Delta^o_{ij} \bm\Delta^t_{ij} \Big]
    \end{align*}
    
    Plugging the above, as well as \eqref{eq:noise_term}, into \eqref{eq:simplify_loss_symmetric_bernoulli}, 
    \begin{align*}
        L(\mW^o, \mW^t) &= 2 - 2 \alpha^2 \kappa \sum_{i=1}^d \big( (1 + \bm\Delta^o_{ii} - b^o_i) (1 + \bm\Delta^t_{ii} - b^t_i) + \alpha^2 \kappa \sum_{j \ne i} \bm\Delta^o_{ij} \bm\Delta^t_{ij} \big) \\
        &\quad + \sigma_0^2 ( (1 + \bm\Delta^o_{ii}) (1 + \bm\Delta^t_{ii})  + \alpha^2 \sum_{j \ne i} \bm\Delta^o_{ij} \bm\Delta^t_{ij} ) 
    \end{align*}
    
    Hence
    \begin{align*}
        \nabla_{\mW^o_{ii}} L(\mW^o, \mW^t) &= -2 \alpha^2 \kappa ((1+\sigma_0^2)(1 + \bm\Delta^t_{ii}) - b^t_i)  \\
        \nabla_{\mW^o_{ij}} L(\mW^o, \mW^t) &= -2 \alpha^2 \kappa \cdot \alpha^2 (\kappa + \sigma_0^2) \bm\Delta^t_{ij} 
    \end{align*}
    
    Let $a = \frac{\alpha^2 (\kappa + \sigma_0^2)}{(1+\sigma_0^2)(1 + \bm\Delta^t_{ii}) - b^t_i}$, 
    then Assumption~\ref{assumption:bias_symmetric_bernoulli} and Assumption~\ref{assumption:warm_start_symmetric_bernoulli}, $a \in (0, 1)$ and
    \begin{equation}
        \label{eq:update_ratio}
        \nabla_{\mW^o_{ij}} L(\mW^o, \mW^t) = a \bm\Delta^t_{ij} \nabla_{\mW^o_{ii}} L(\mW^o, \mW^t) \quad \forall j \ne i
    \end{equation}
    
    After each update, with row-normalization, $\mW^o$ still satisfies the warm-start condition. 
    In the following, we characterize the convergence point.
    
    By Lemma~\ref{lemma:convergence_proportional_normalization}, \eqref{eq:update_ratio} implies that
    the convergence points $\mW^{o*}$ satisfy 
    \[ \mW^{o*}_{ij} = a \bm\Delta^t_{ij} \mW^{o*}_{ii} = a \mW^t_{ij} \mW^{o*}_{ii} \quad \forall j \ne i \]
    and moreover, since $\mW^{o*}_{ii} \in [0, 1]$,
    \[ |\mW^{o*}_{ij}| \le a |\mW^t_{ij}| \quad \forall j \ne i \]
    
    Next, while updating $\mW^t$ to convergence $\mW^{t*}$, the same argument implies 
    \[ \mW^{t*}_{ij} = a \mW^{o*}_{ij} \mW^{t*}_{ii} \quad \forall j \ne i \]
    and consequently
    \[ |\mW^{t*}_{ij}| \le a |\mW^{o*}_{ij}| \le a^2 |\mW^t_{ij}| \]
     
    Repeating the above alternating optimization process $T$ times, the resulting $\mW^{t*}$ satisfies
    \[ |\mW^{t*}_{ij}| \le a^{2T} |\mW^t_{ij}| \]
    since $a \in (0, 1)$, $\lim_{T \to \infty} a^{2T} = 0$. Hence, as $T \to \infty$,
    \[ |\mW^{o*}_{ij}|, |\mW^{t*}_{ij}| \to 0, \forall j \ne i \] 
    and 
    \[ |\mW^{o*}_{ii}|, |\mW^{t*}_{ii}| \to 1, \forall i \]
    Therefore both $\mW^o$ and $\mW^t$ converge to $I$. 
\end{proof}

\section{TECHNICAL LEMMAS}

\begin{lemma}[convergence of proportional update with normalization]
    \label{lemma:convergence_proportional_normalization}
    For any $\vv^{(0)} \in \R^d$, consider the following update 
    \begin{align*}
        \vv^{(t+1)}_1 &\leftarrow \vv^{(t)}_1 + c^{(t)} \\
        \vv^{(t+1)}_i &\leftarrow \vv^{(t)}_i + a c^{(t)} \quad \forall i = 2..d \\
        \vv^{(t+1)} &\leftarrow l_2\text{-normalize}(\vv^{(t+1)}) 
    \end{align*}
    in which $a \in \R$ and $\exists c > 0$ s.t. $\forall t, c^{(t)} > c$.
    Moreover, the initial $\vv^{(0)}$ satisfies 
    \begin{align*}
        \vv^{(0)}_1 + a \vv^{(0)}_2 + \cdots a \vv^{(0)}_d &> 0 \\
        \sum_{i=1}^d (\vv^{(0)}_i)^2 = 1
    \end{align*}
    Then, repeating the above update will cause $\vv^{(t)}$ to converge to
    \begin{align*}
        \vv^{(*)}_1 &\leftarrow \frac{1}{\sqrt{1 + (d-1) a^2}} \\
        \vv^{(*)}_i &\leftarrow \frac{a}{\sqrt{1 + (d-1) a^2}} \quad \forall i = 2..d 
    \end{align*}
\end{lemma}

\begin{proof}
    By the update rule,
    \begin{align*}
        \vv^{(t+1)}_1 &= \frac{\vv^{(t)}_1 + c^{(t)}}{\sqrt{(\vv^{(t)}_1 + c^{(t)})^2 + \sum_{i=2}^d (\vv^{(t)}_i + a c^{(t)})^2} } \\
        \vv^{(t+1)}_i &= \frac{\vv^{(t)}_i + a c^{(t)}}{\sqrt{(\vv^{(t)}_1 + c^{(t)})^2 + \sum_{i=2}^d (\vv^{(t)}_i + a c^{(t)})^2} } \quad \forall i = 2..d
    \end{align*}

    First, we prove by induction that $\forall t$,
    \begin{equation}
        \label{eq:induction_positive}
        \vv^{(t)}_1 + a \vv^{(t)}_2 + \cdots a \vv^{(t)}_d > 0
    \end{equation}
    It is given that \eqref{eq:induction_positive} holds for $t=0$. 
    Suppose it holds for $t=t_0$, then by the update rule
    \begin{align*}
        &\quad \vv^{(t_0+1)}_1 + a \vv^{(t_0+1)}_2 + \cdots a \vv^{(t_0+1)}_d \\
        &= \frac{\vv^{(t_0)}_1 + c^{(t_0)}}{\sqrt{(\vv^{(t_0)}_1 + c^{(t_0)})^2 + \sum_{i=2}^d (\vv^{(t_0)}_i + a c^{(t_0)})^2} } + a \sum_{i=2}^d \frac{\vv^{(t_0)}_i + a c^{(t_0)}}{\sqrt{(\vv^{(t_0)}_1 + c^{(t_0)})^2 + \sum_{i=2}^d (\vv^{(t_0)}_i + a c^{(t_0)})^2} } \\
        &= \frac{\vv^{(t_0)}_1 + c^{(t_0)} + a \sum_{i=2}^d (\vv^{(t_0)}_i + a c^{(t_0)}) }{\sqrt{(\vv^{(t_0)}_1 + c^{(t_0)})^2 + \sum_{i=2}^d (\vv^{(t_0)}_i + a c^{(t_0)})^2} } \\
        &= \frac{\vv^{(t_0)}_1 + a \sum_{i=2}^d \vv^{(t_0)}_i + (d-1) a^2 c^{(t_0)} }{\sqrt{(\vv^{(t_0)}_1 + c^{(t_0)})^2 + \sum_{i=2}^d (\vv^{(t_0)}_i + a c^{(t_0)})^2} } \\
        &> \frac{ (d-1) a^2 c^{(t_0)} }{\sqrt{(\vv^{(t_0)}_1 + c^{(t_0)})^2 + \sum_{i=2}^d (\vv^{(t_0)}_i + a c^{(t_0)})^2} } \quad \text{(since \eqref{eq:induction_positive} holds for $t=t_0$)} \\
        &\ge \frac{ (d-1) a^2 c }{\sqrt{(\vv^{(t_0)}_1 + c^{(t_0)})^2 + \sum_{i=2}^d (\vv^{(t_0)}_i + a c^{(t_0)})^2} } \quad \text{(since it is given that $\forall t, c^{(t)} > c$)} \\
        &\ge 0 \quad \text{(since it is given that $c > 0$)}
    \end{align*}
    Hence, \eqref{eq:induction_positive} also holds for $t=t_0+1$.
    Therefore, by induction, \eqref{eq:induction_positive} also holds for all $t$.
    
    Next, we use the above fact to prove that $\exists \gamma \in (0, 1)$ s.t. $\forall t, \forall i = 2..d$,
    \begin{equation}
        \label{eq:induction_shrink}
        | \vv^{(t+1)}_i - a \vv^{(t+1)}_1 | < \gamma | \vv^{(t)}_i - a \vv^{(t)}_1 |
    \end{equation}
    To show this, plugging in the update equation, we get $\forall t, \forall i = 2..d$,
    \begin{align*}
        &\quad | \vv^{(t+1)}_i - a \vv^{(t+1)}_1 | \\
        &= \left| \frac{\vv^{(t)}_i + a c^{(t)}}{\sqrt{(\vv^{(t)}_1 + c^{(t)})^2 + \sum_{i=2}^d (\vv^{(t)}_i + a c^{(t)})^2} } - a \frac{\vv^{(t)}_1 + c^{(t)}}{\sqrt{(\vv^{(t)}_1 + c^{(t)})^2 + \sum_{i=2}^d (\vv^{(t)}_i + a c^{(t)})^2} } \right| \\
        &= \left| \frac{\vv^{(t)}_i + a c^{(t)} - a (\vv^{(t)}_1 + c^{(t)}) }{\sqrt{(\vv^{(t)}_1 + c^{(t)})^2 + \sum_{i=2}^d (\vv^{(t)}_i + a c^{(t)})^2} } \right| \\
        &= \left| \frac{\vv^{(t)}_i - a \vv^{(t)}_1 }{\sqrt{(\vv^{(t)}_1 + c^{(t)})^2 + \sum_{i=2}^d (\vv^{(t)}_i + a c^{(t)})^2} } \right| \\
        &= \left| \frac{\vv^{(t)}_i - a \vv^{(t)}_1 }{\sqrt{\sum_{i=1}^d (\vv^{(t)}_i)^2 + 2 c^{(t)} \vv^{(t)}_1 + (c^{(t)})^2 + \sum_{i=2}^d 2 a c^{(t)} \vv^{(t)}_i + (d-1) a^2 (c^{(t)})^2} } \right| \\
        &= \left| \frac{\vv^{(t)}_i - a \vv^{(t)}_1 }{\sqrt{\sum_{i=1}^d (\vv^{(t)}_i)^2 + 2 c^{(t)} (\vv^{(t)}_1 + a \sum_{i=2}^d \vv^{(t)}_i ) + (1 + (d-1) a^2) (c^{(t)})^2} } \right| \\
        &= \left| \frac{\vv^{(t)}_i - a \vv^{(t)}_1 }{\sqrt{1 + 2 c^{(t)} (\vv^{(t)}_1 + a \sum_{i=2}^d \vv^{(t)}_i ) + (1 + (d-1) a^2) (c^{(t)})^2} } \right| \quad \text{(since each $\vv^{(t)}$ is $l_2$-normalized)} \\
        &\le \left| \frac{\vv^{(t)}_i - a \vv^{(t)}_1 }{\sqrt{1 + (1 + (d-1) a^2) (c^{(t)})^2} } \right| \quad \text{(since $c^{(t)} > c > 0$ and \eqref{eq:induction_positive})} \\
        &< \left| \frac{\vv^{(t)}_i - a \vv^{(t)}_1 }{\sqrt{1 + (1 + (d-1) a^2) c^2} } \right| \quad \text{(since $c^{(t)} > c > 0$)} 
    \end{align*}
    Hence, \eqref{eq:induction_shrink} holds with the constant
    \[ \gamma = \frac{1}{\sqrt{1 + (1 + (d-1) a^2) c^2} } \in (0, 1) \]
    Therefore, 
    \begin{align*}
        &\quad \lim_{t \to \infty} | \vv^{(t)}_i - a \vv^{(t)}_1 | \\
        &\le \lim_{t \to \infty} \gamma^t | \vv^{(0)}_i - a \vv^{(0)}_1 | \\
        &= 0 \quad \text{(since $\gamma \in (0, 1)$)}
    \end{align*}
    
    Finally, since each $\vv^{(t)}$ is $l_2$-normalized, 
    the above relation for $t \to \infty$ leads to the following equations 
    \begin{align*}
        \lim_{t \to \infty} \sum_{i=1}^d (\vv^{(t)}_i)^2 &= 1 \\
        \lim_{t \to \infty} \vv^{(t)}_i &= a \lim_{t \to \infty} \vv^{(t)}_1 \quad \forall i = 2..d
    \end{align*}
    whose unique solution is the one stated in the lemma.
\end{proof}

%% file: sections/appendices/empirical_results_dynamics.tex
\subsection{Additional Empirical results for Training dynamics}
\subsubsection{Non-contrastive model needs a warm start and a linear prediction head to learn better representations}

\begin{figure}[!ht]
  \centering
  \begin{minipage}[b]{0.23\textwidth}
    \includegraphics[width=\textwidth]{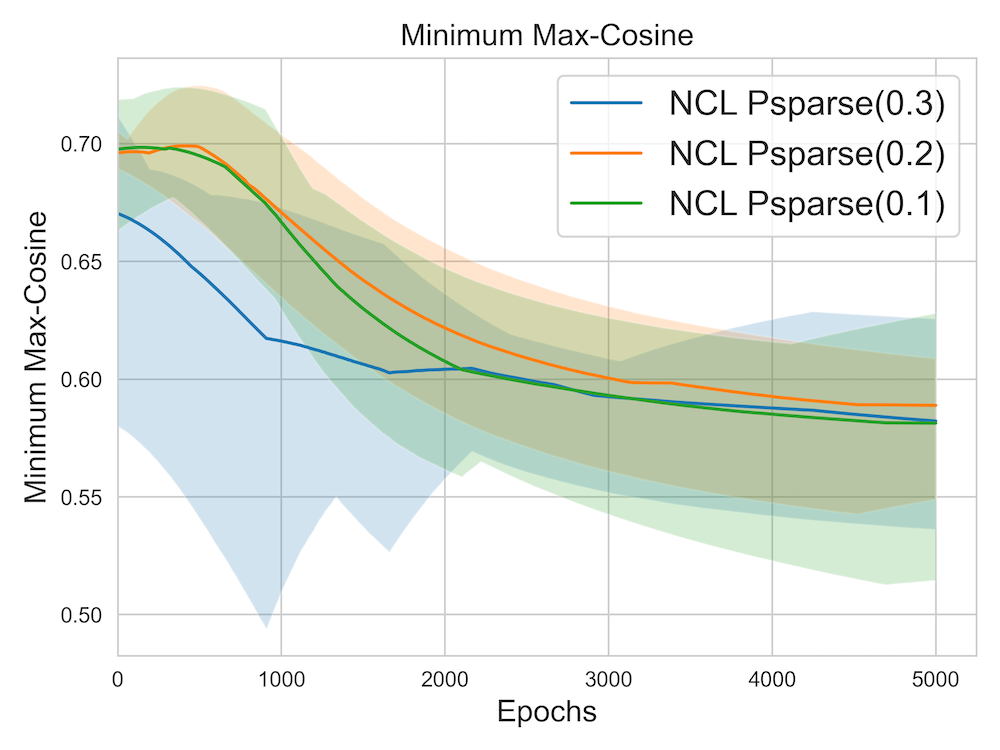}
  \end{minipage}
    \hfill
  \begin{minipage}[b]{0.23\textwidth}
    \includegraphics[width=\textwidth]{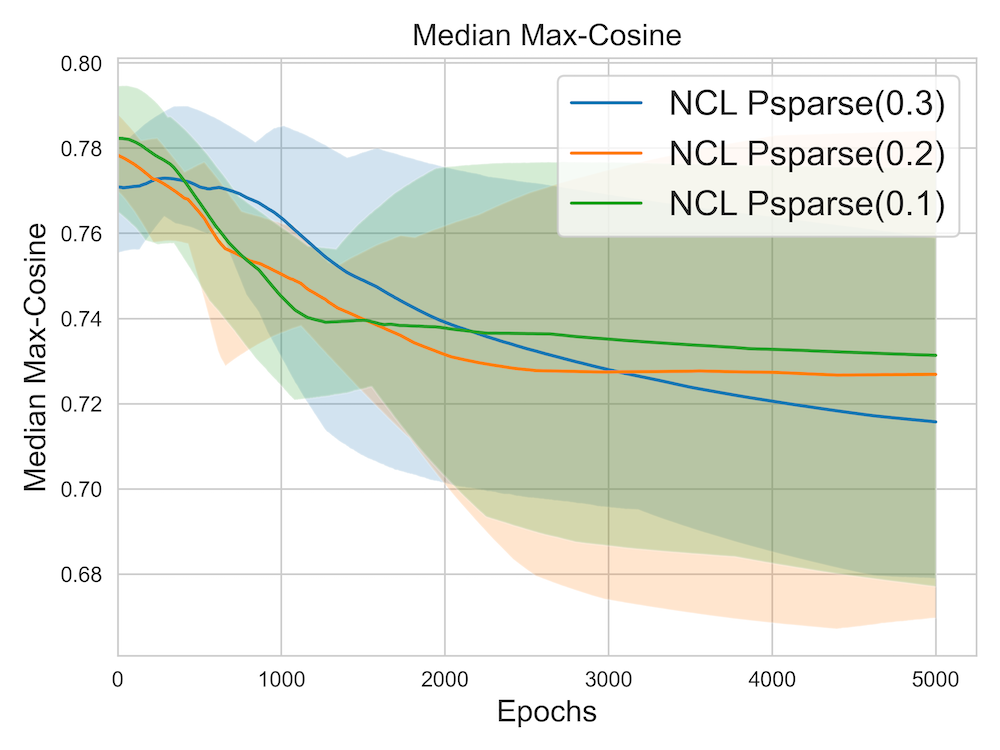}
  \end{minipage}
  \hfill
  \begin{minipage}[b]{0.23\textwidth}
    \includegraphics[width=\textwidth]{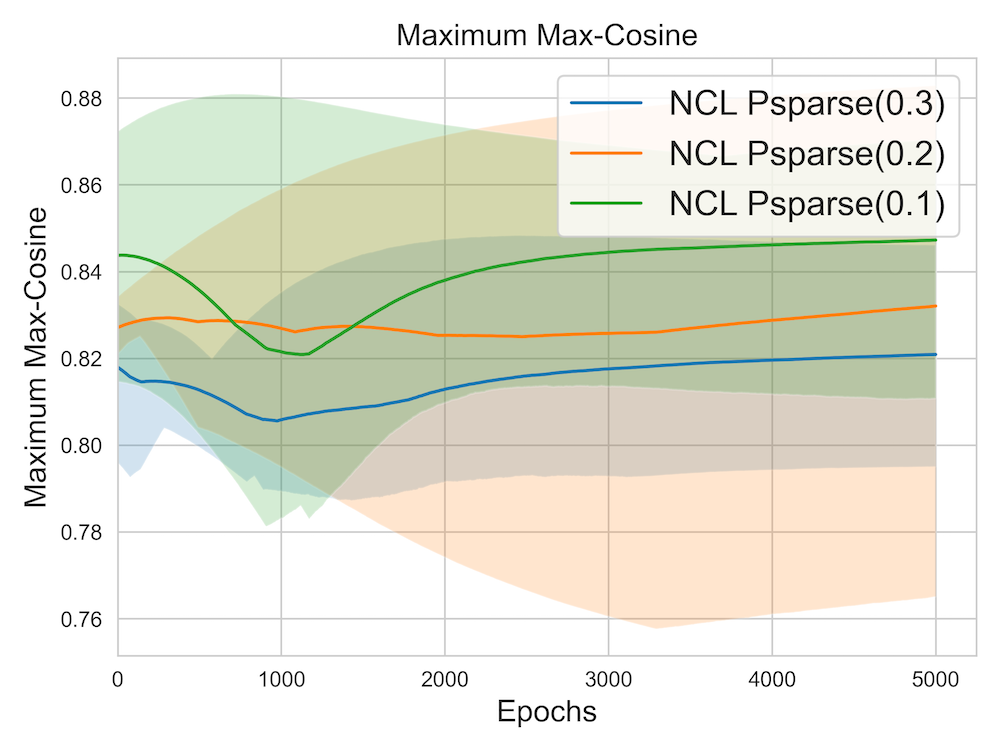}
  \end{minipage}
   \hfill
  \begin{minipage}[b]{0.23\textwidth}
    \includegraphics[width=\textwidth]{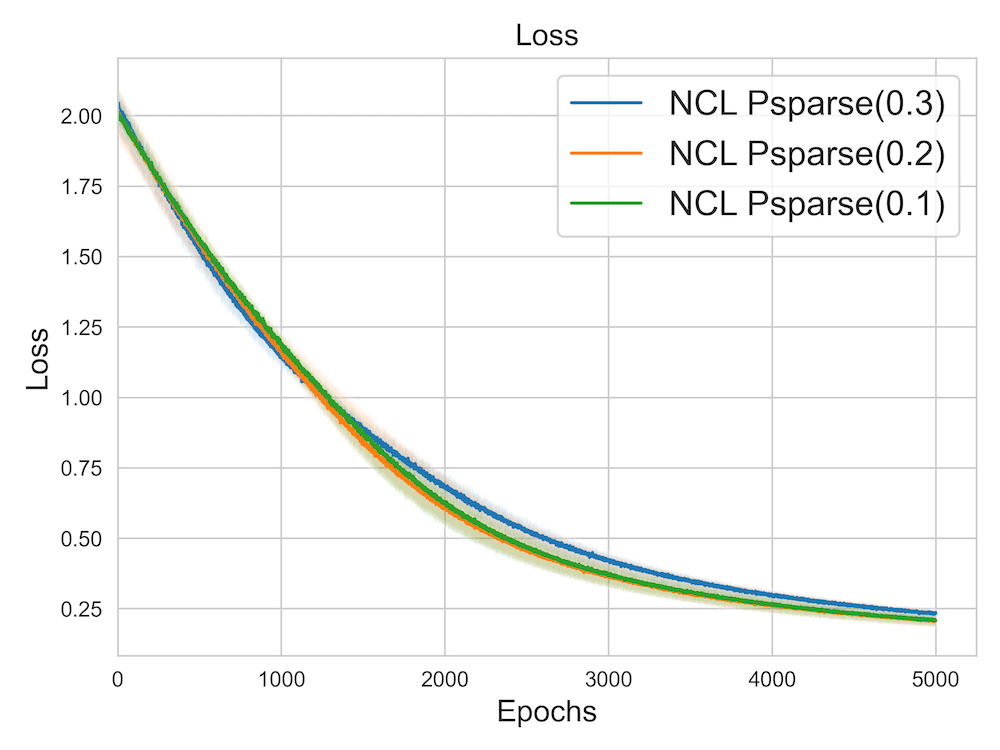}
  \end{minipage}
    \\~\\
  \caption{(NCL with warm start does not work) Minimum Max-Cosine (left) and Loss (right) curves for non-contrastive loss (NCL) with a warm-started linear encoder. Psparse indicates $Pr(\bm{z}_i = \pm 1), i \in [d]$ in the sparse coding vector $\bm{z}$. Reported numbers are averaged over 5 different runs. The shaded area represents the maximum and the minimum values observed across those 5 runs. We use p=50, m=50, d=10}
\end{figure}

\begin{figure}[!htb]
  \centering
  \begin{minipage}[b]{0.23\textwidth}
    \includegraphics[width=\textwidth]{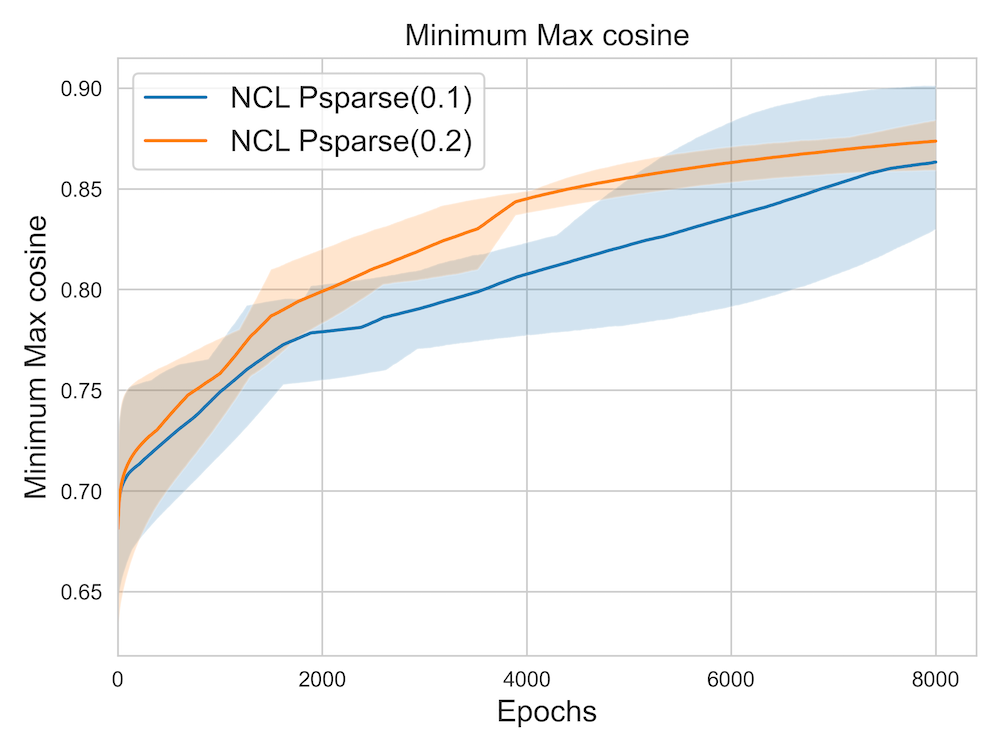}
  \end{minipage}
  \hfill
  \begin{minipage}[b]{0.23\textwidth}
    \includegraphics[width=\textwidth]{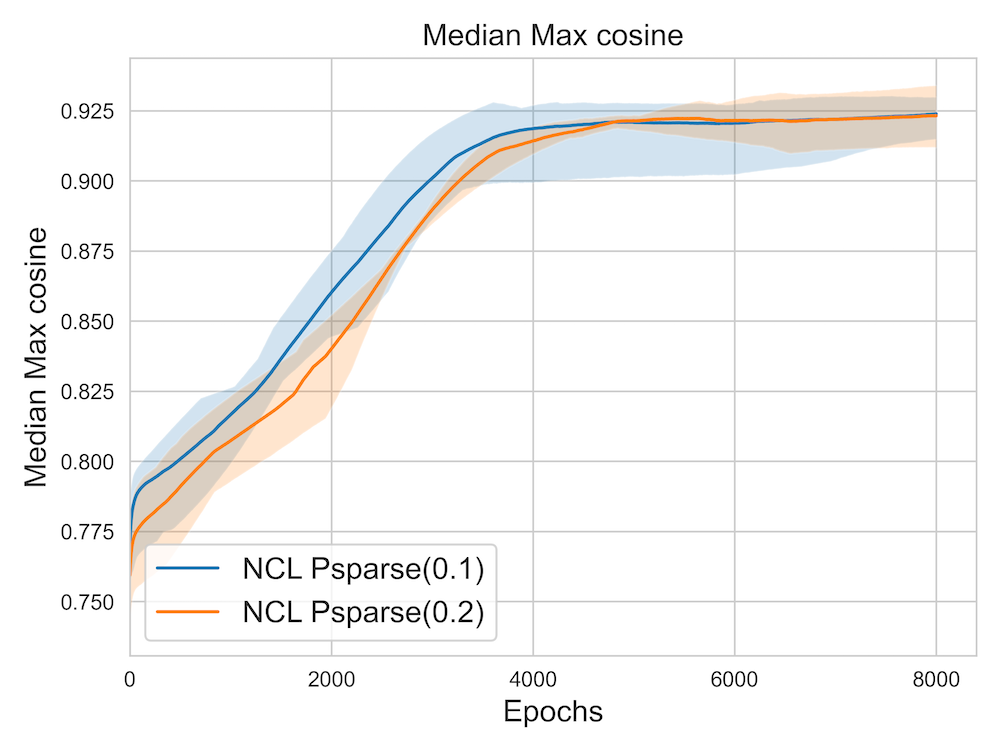}
  \end{minipage}
   \hfill
  \begin{minipage}[b]{0.23\textwidth}
    \includegraphics[width=\textwidth]{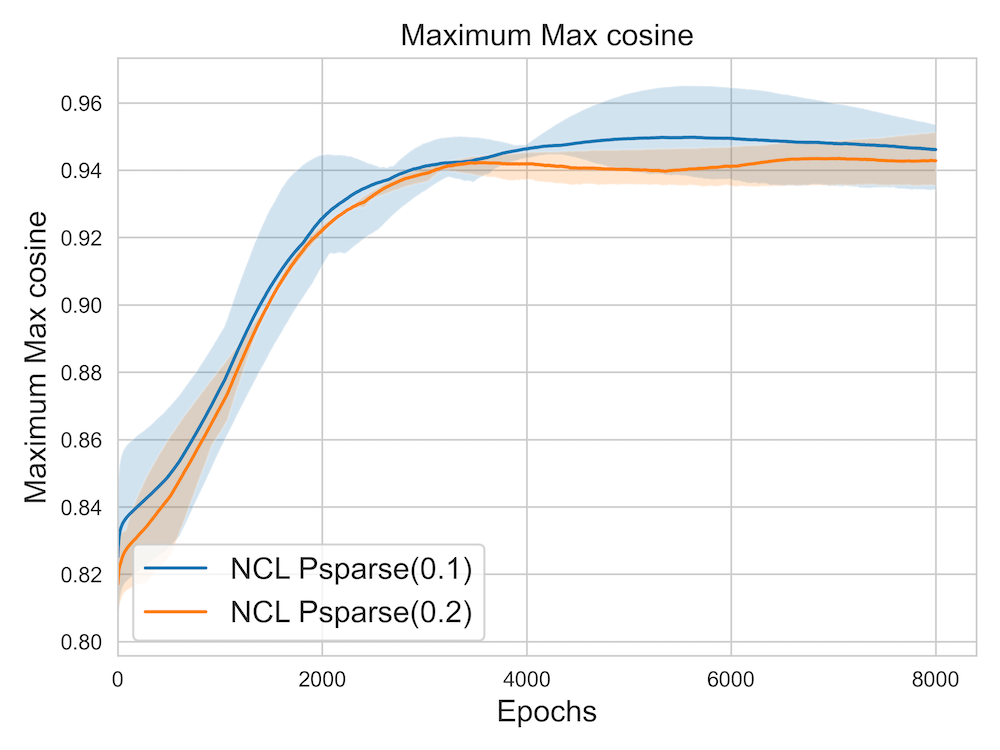}
  \end{minipage}
   \hfill
  \begin{minipage}[b]{0.23\textwidth}
    \includegraphics[width=\textwidth]{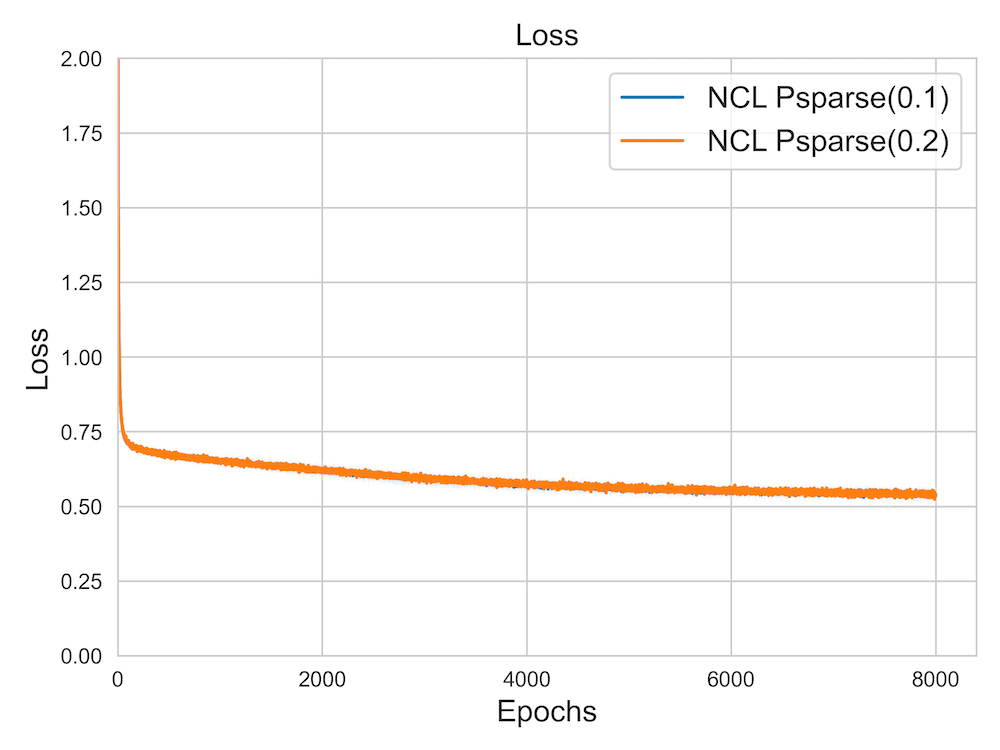}
  \end{minipage}
    \\~\\
  \caption{(NCL with a linear prediction head and warm-start works) (left-to-right) Minimum Max-Cosine,  Median Max-Cosine, Maximum Max-Cosine, and Loss curves for non-contrastive loss (NCL) with warm-started linear encoder and a linear predictor. Reported numbers are averaged over 3 different runs. The shaded area represents the maximum and the minimum values observed across those 3 runs. We use p=50, m=50, d=10}
\end{figure}

\newpage
\subsubsection{Non-contrastive model with warm start learns good  representations if weights of the encoder are row-normalized or column-normalized}

\begin{figure}[!htb]
  \centering
  \begin{minipage}[b]{0.23\textwidth}
    \includegraphics[width=\textwidth]{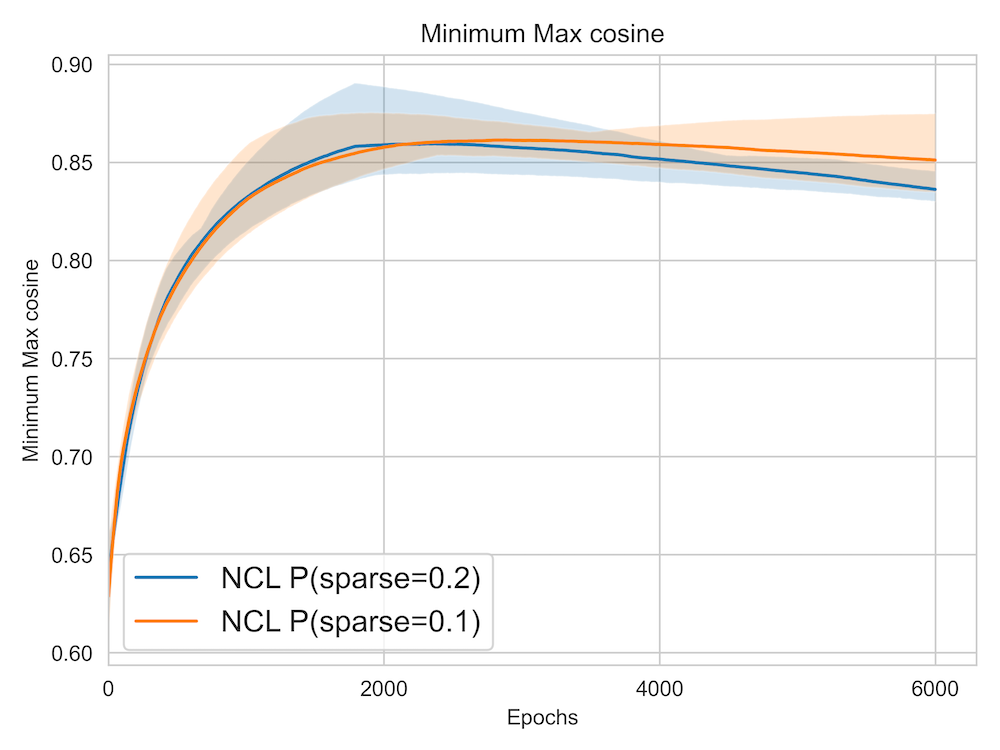}
  \end{minipage}
  \hfill
  \begin{minipage}[b]{0.23\textwidth}
    \includegraphics[width=\textwidth]{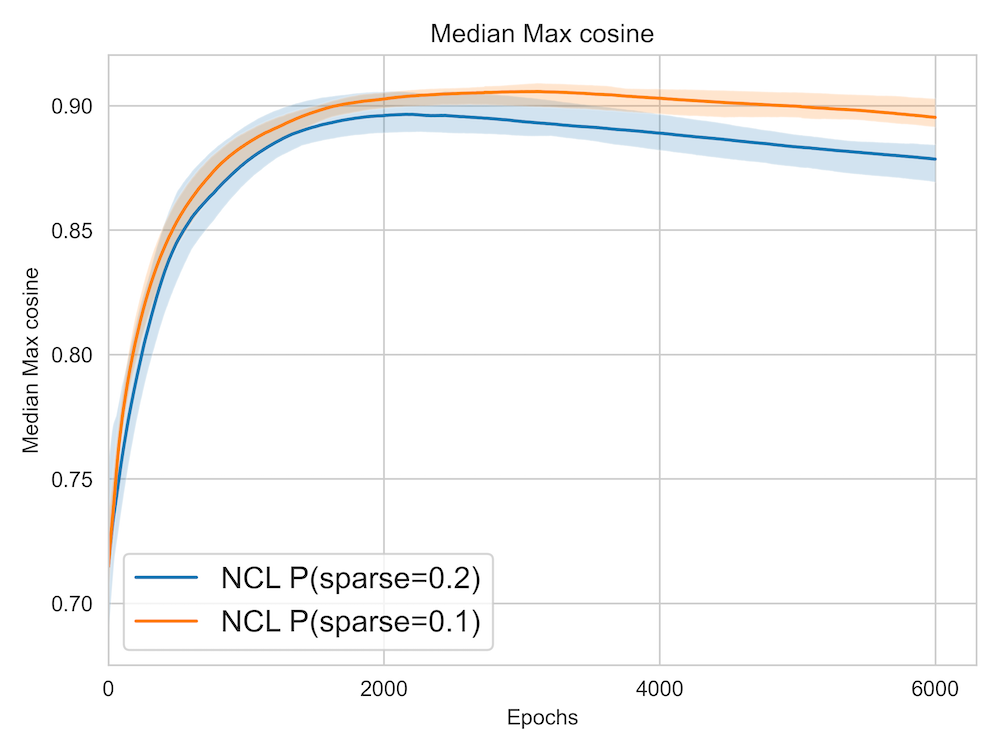}
  \end{minipage}
   \hfill
  \begin{minipage}[b]{0.23\textwidth}
    \includegraphics[width=\textwidth]{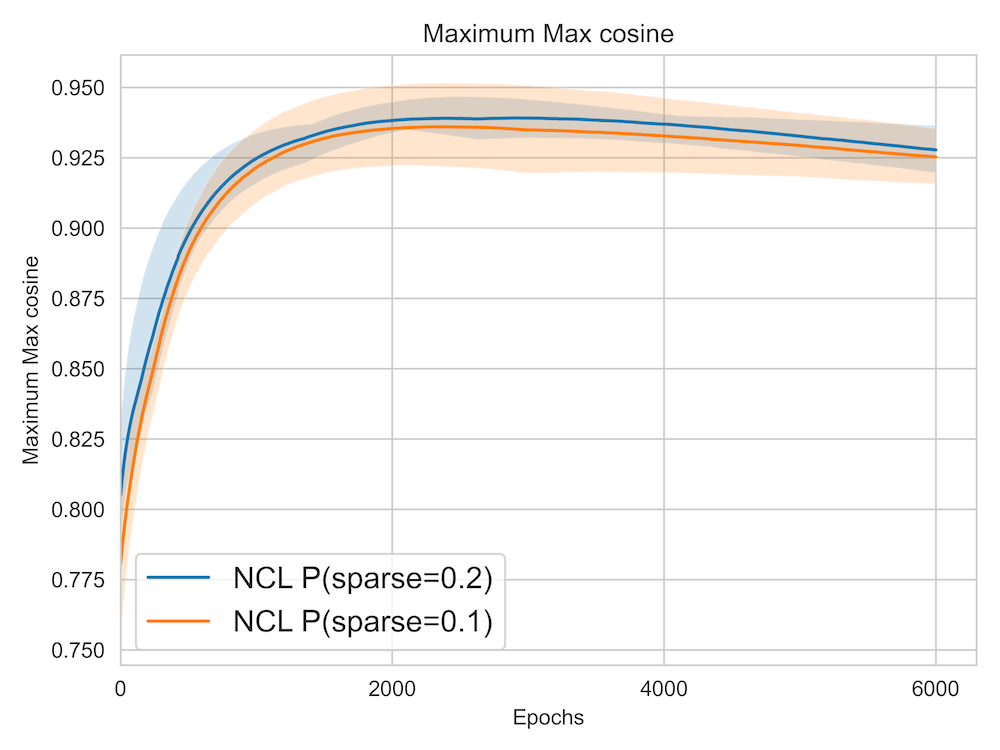}
  \end{minipage}
   \hfill
  \begin{minipage}[b]{0.23\textwidth}
    \includegraphics[width=\textwidth]{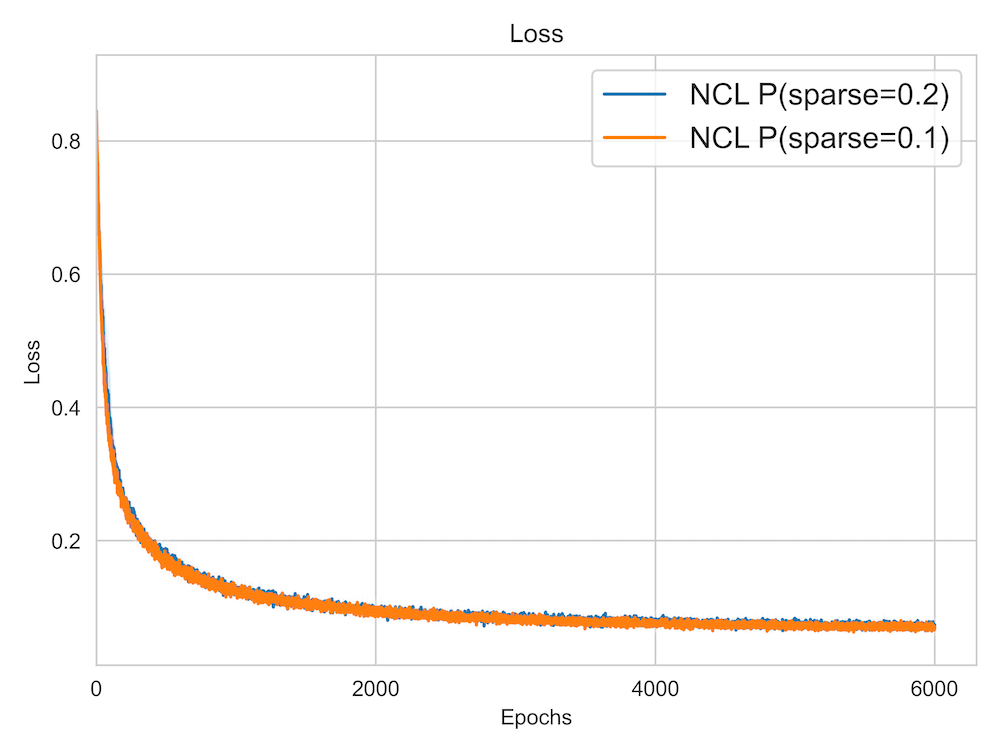}
  \end{minipage}
    \\~\\
  \caption{(NCL with warm-start and row-normalized encoder works) (left-to-right) Minimum Max-Cosine,  Median Max-Cosine, Maximum Max-Cosine, and Loss curves for  non-contrastive loss [BYOL \cite{grill2020bootstrap}] with warm-started linear encoder and row normalized encoder. Reported numbers are averaged over 3 different runs. The shaded area represents the maximum and the minimum values observed across those 3 runs. We use p=50, d=10, m=10. Warm start parameter $c=1$.}
\end{figure}

\begin{figure}[!htb]
  \centering
  \begin{minipage}[b]{0.23\textwidth}
    \includegraphics[width=\textwidth]{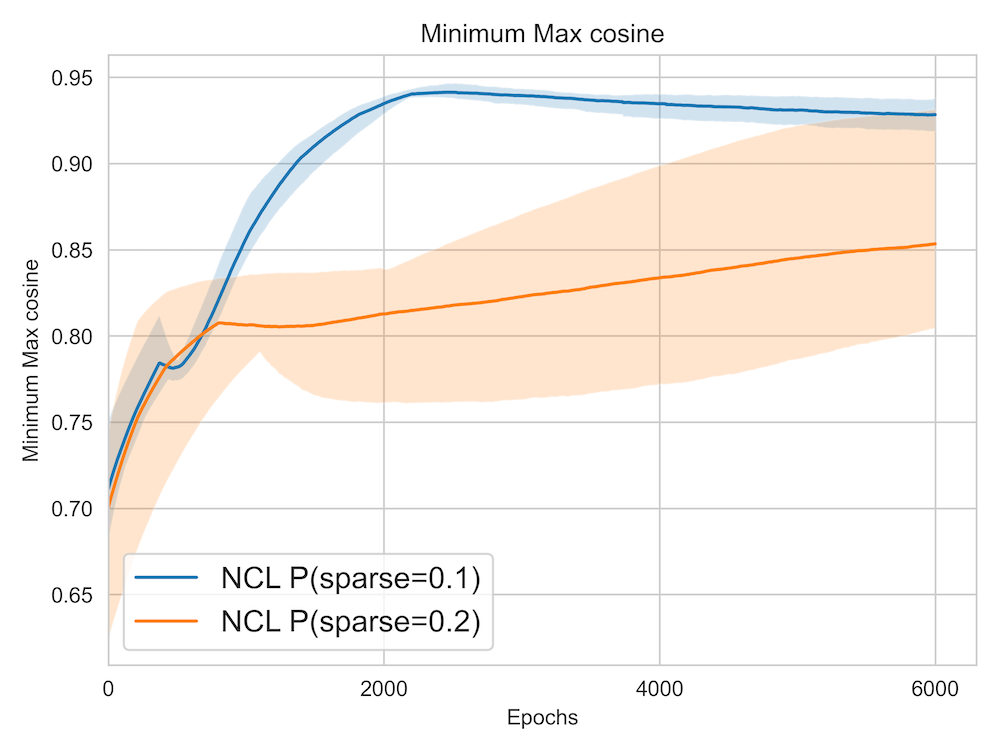}
  \end{minipage}
  \hfill
  \begin{minipage}[b]{0.23\textwidth}
    \includegraphics[width=\textwidth]{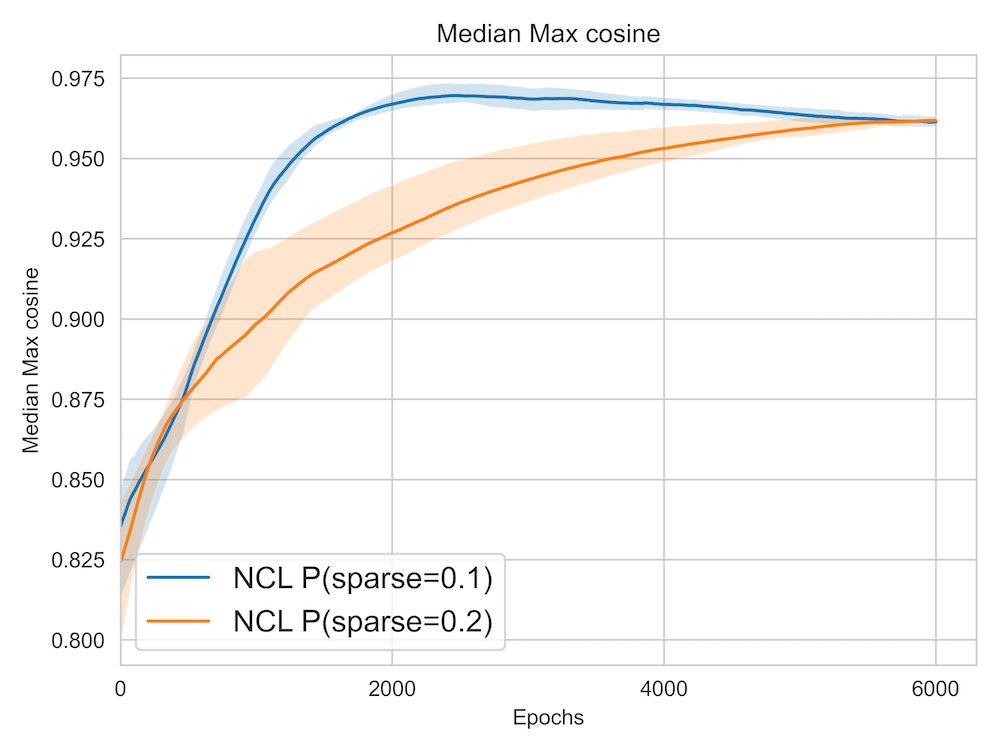}
  \end{minipage}
   \hfill
  \begin{minipage}[b]{0.23\textwidth}
    \includegraphics[width=\textwidth]{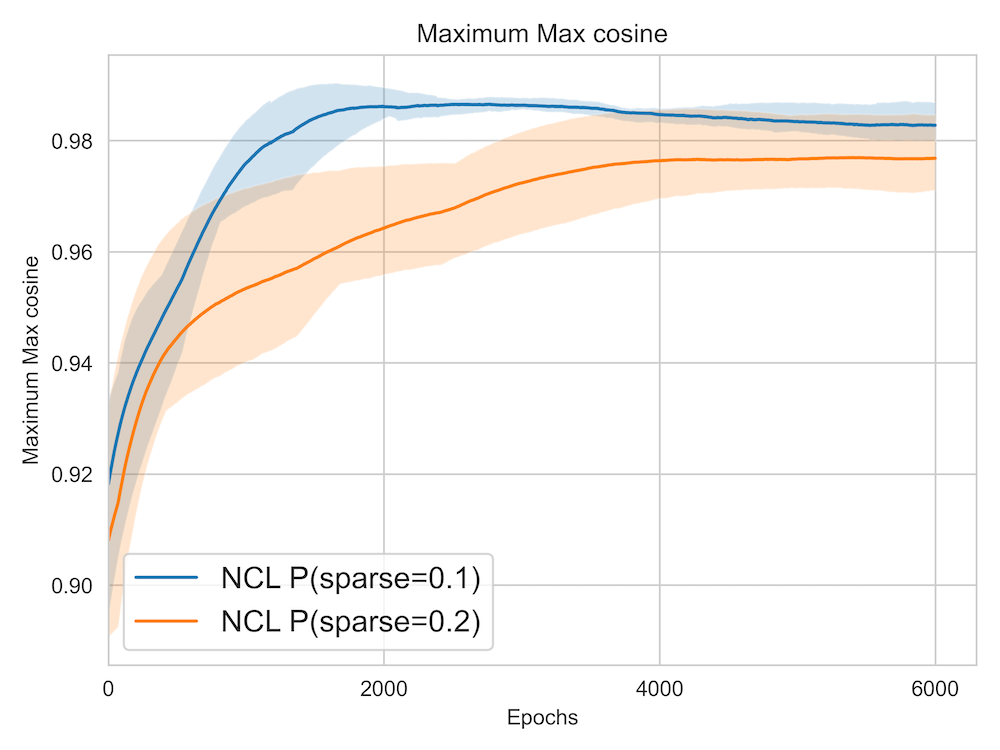}
  \end{minipage}
   \hfill
  \begin{minipage}[b]{0.23\textwidth}
    \includegraphics[width=\textwidth]{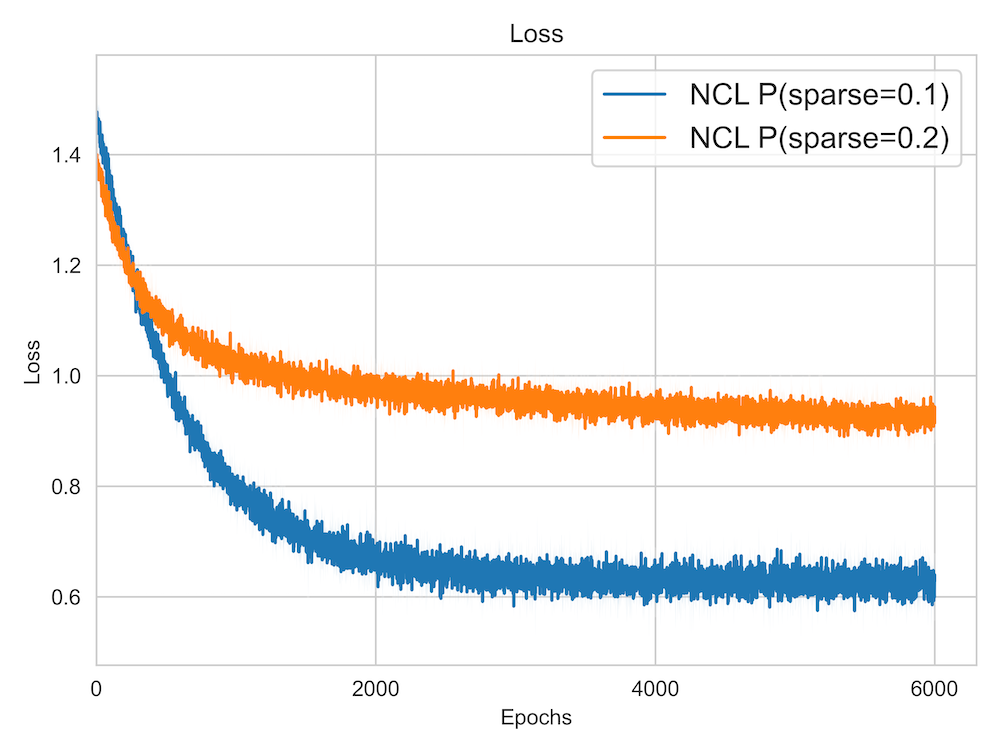}
  \end{minipage}
    \\~\\
  \caption{(NCL with warm-start and row-normalized encoder works) (left-to-right) Minimum Max-Cosine,  Median Max-Cosine, Maximum Max-Cosine, and Loss curves for  non-contrastive loss [BYOL \cite{grill2020bootstrap}] with warm-started linear encoder and row normalized encoder. Reported numbers are averaged over 3 different runs. The shaded area represents the maximum and the minimum values observed across those 3 runs. We use p=20, d=20, m=20. Warm start parameter $c=1.25$.}
\end{figure}

\begin{figure}[!htb]
  \centering
  \begin{minipage}[b]{0.23\textwidth}
    \includegraphics[width=\textwidth]{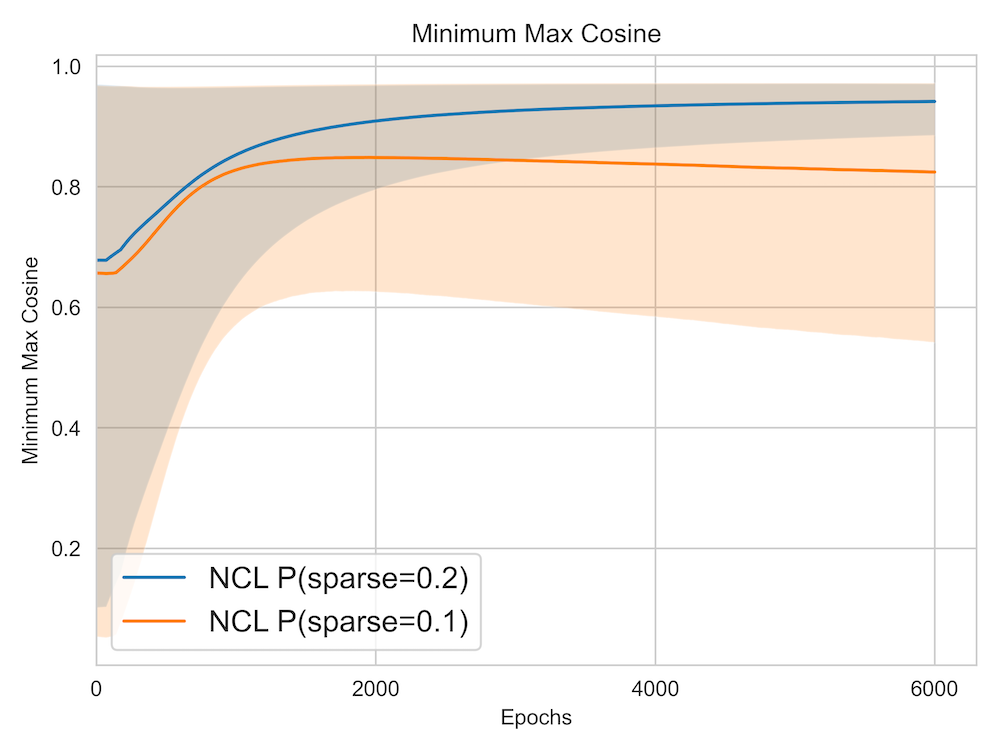}
  \end{minipage}
  \hfill
  \begin{minipage}[b]{0.23\textwidth}
    \includegraphics[width=\textwidth]{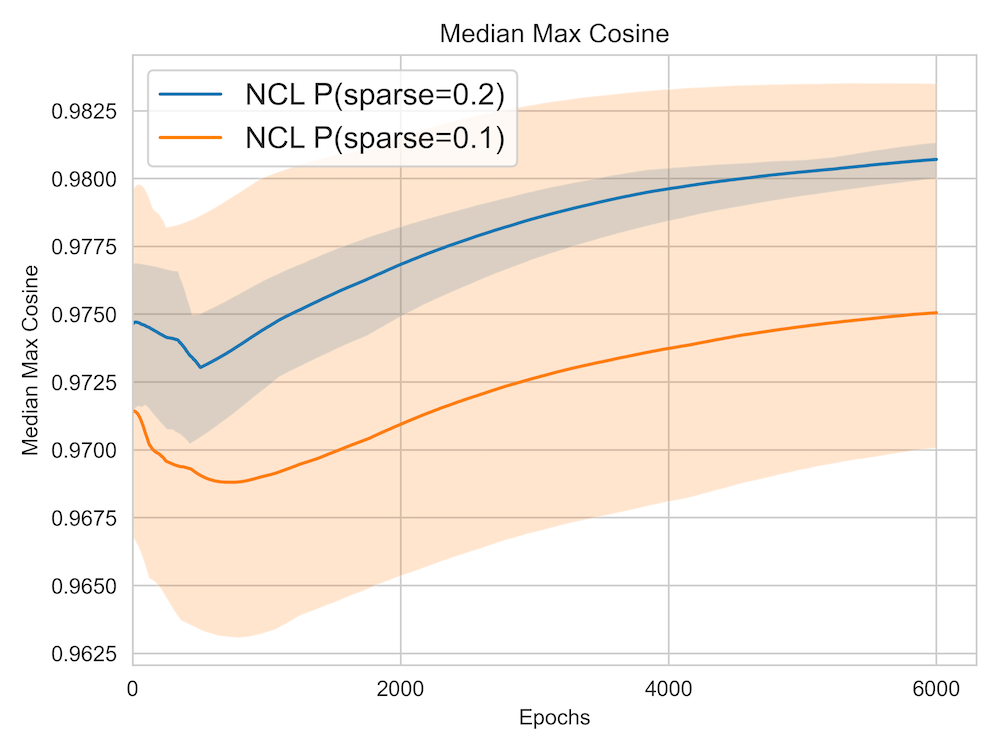}
  \end{minipage}
   \hfill
  \begin{minipage}[b]{0.23\textwidth}
    \includegraphics[width=\textwidth]{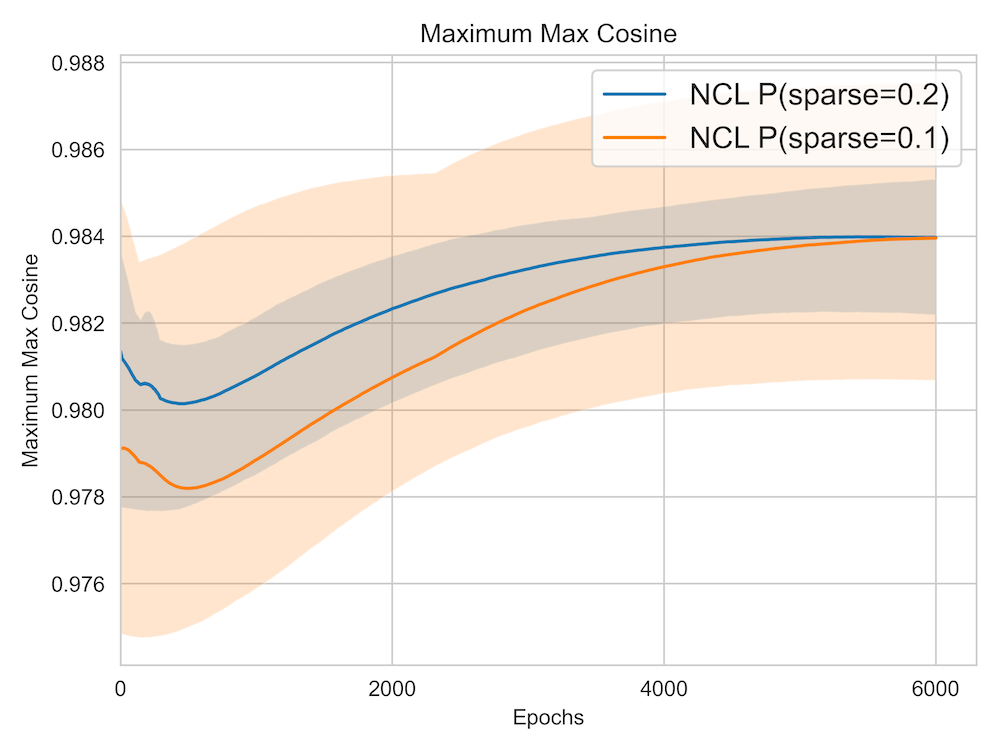}
  \end{minipage}
   \hfill
  \begin{minipage}[b]{0.23\textwidth}
    \includegraphics[width=\textwidth]{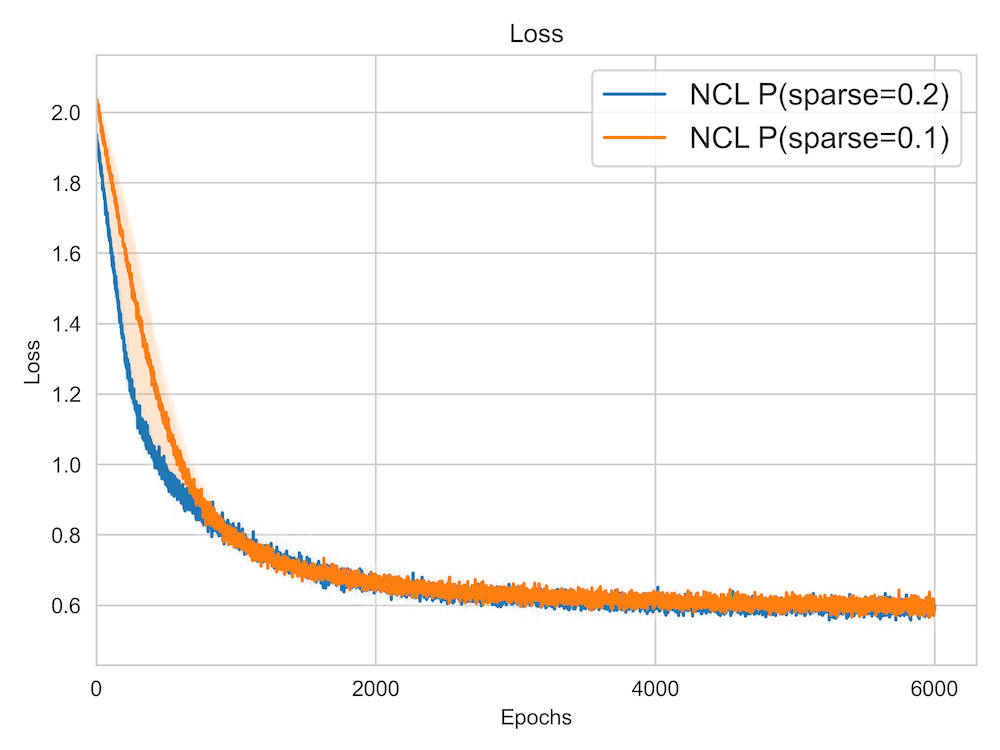}
  \end{minipage}
    \\~\\
  \caption{(NCL with warm-start and column-normalized encoder works) (left-to-right) Minimum Max-Cosine,  Median Max-Cosine, Maximum Max-Cosine, and Loss curves for  non-contrastive loss [BYOL \cite{grill2020bootstrap}] with warm-started linear encoder and column normalized encoder. Reported numbers are averaged over 3 different runs. The shaded area represents the maximum and the minimum values observed across those 3 runs. We use p=50, d=10, m=50. Warm start parameter $c=2$ and probability of random masking is 0.75.}
\end{figure}

\begin{figure}[!htb]
  \centering
  \begin{minipage}[b]{0.23\textwidth}
    \includegraphics[width=\textwidth]{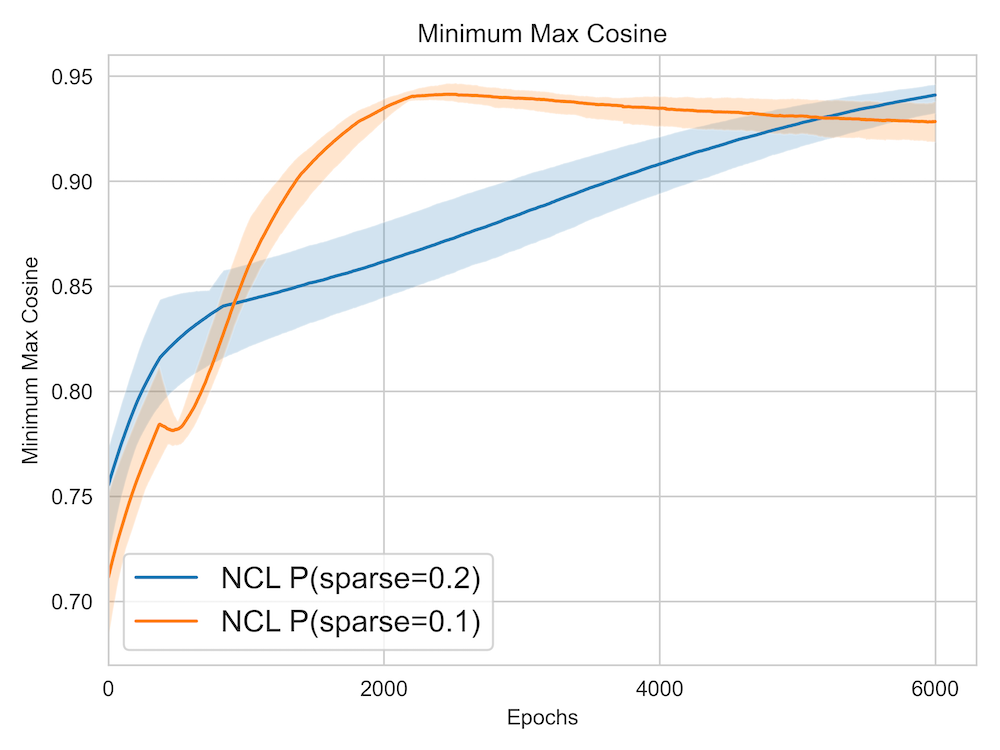}
  \end{minipage}
  \hfill
  \begin{minipage}[b]{0.23\textwidth}
    \includegraphics[width=\textwidth]{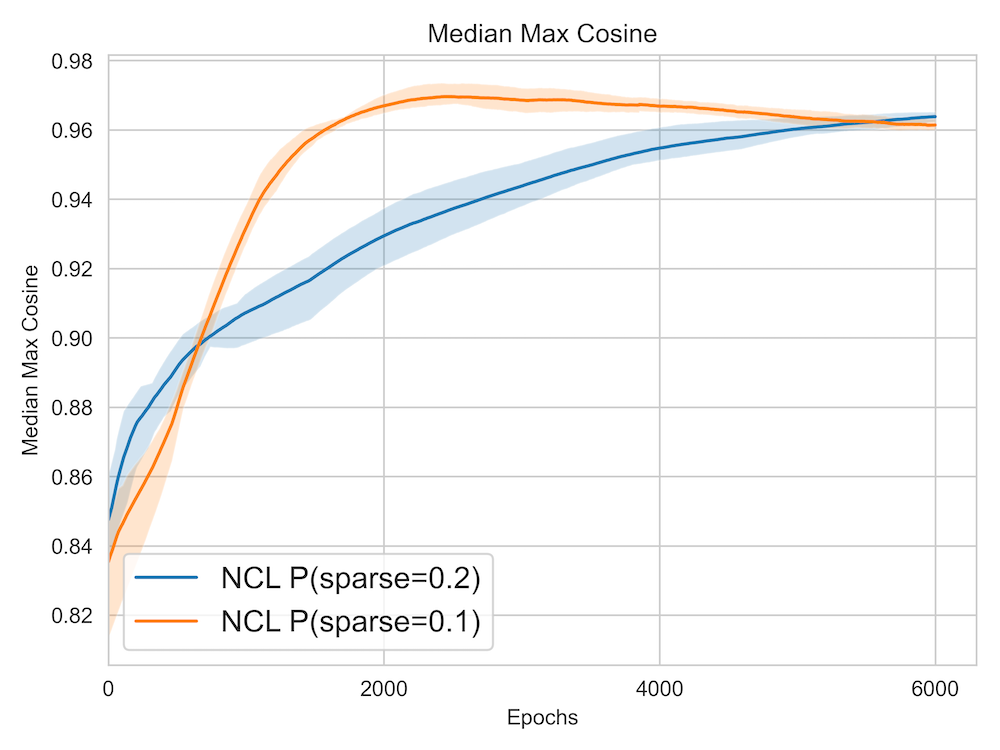}
  \end{minipage}
   \hfill
  \begin{minipage}[b]{0.23\textwidth}
    \includegraphics[width=\textwidth]{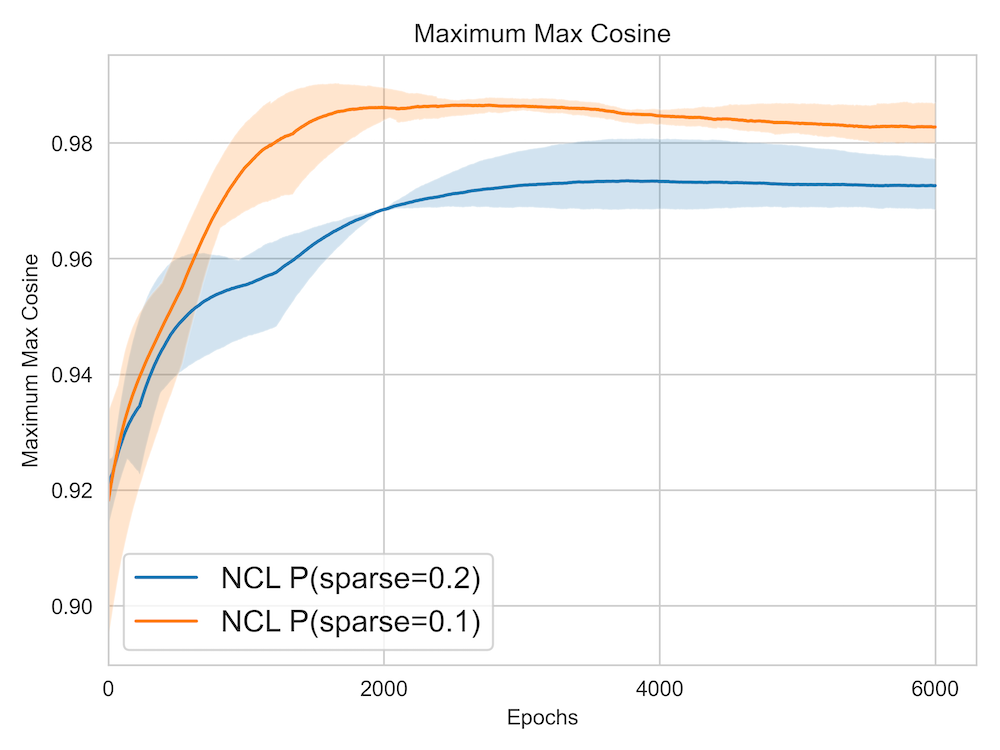}
  \end{minipage}
   \hfill
  \begin{minipage}[b]{0.23\textwidth}
    \includegraphics[width=\textwidth]{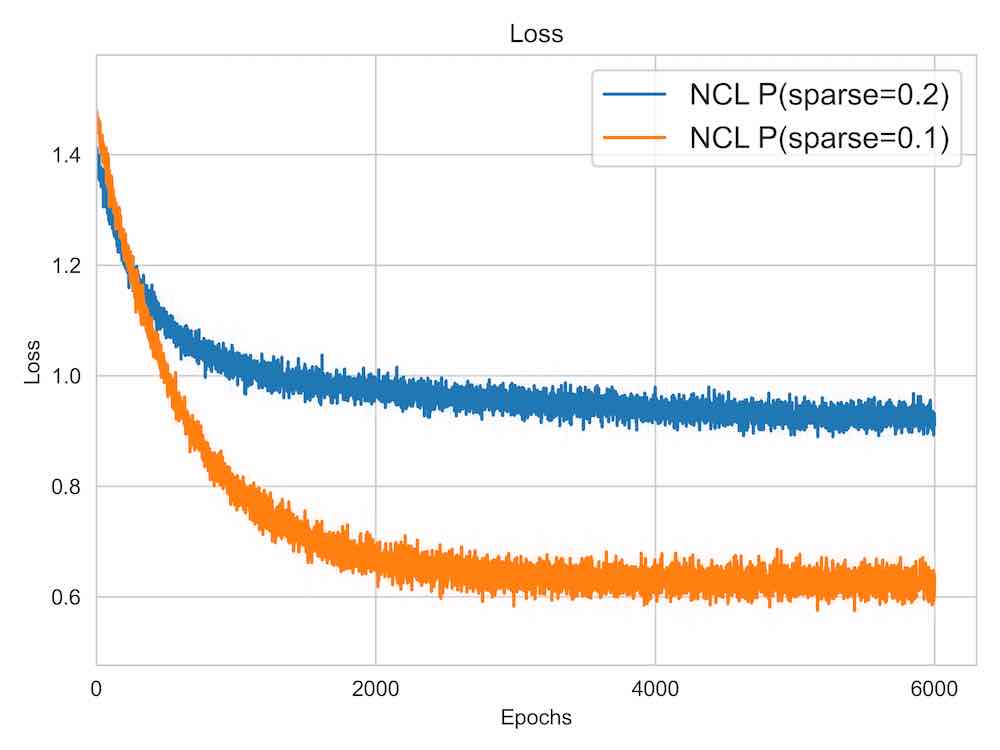}
  \end{minipage}
    \\~\\
  \caption{(NCL with warm-start and column-normalized encoder works) (left-to-right) Minimum Max-Cosine,  Median Max-Cosine, Maximum Max-Cosine, and Loss curves for  non-contrastive loss [BYOL \cite{grill2020bootstrap}] with warm-started linear encoder and column normalized encoder. Reported numbers are averaged over 3 different runs. The shaded area represents the maximum and the minimum values observed across those 3 runs. We use p=20, d=20, m=20. Warm start parameter $c=1.25$ and probability of random masking is 0.5.}
\end{figure}
\FloatBarrier

\subsubsection{Robustness across architectures and loss functions}
\begin{figure}[!h]
  \centering
  \begin{minipage}[b]{0.23\textwidth}
    \includegraphics[width=\textwidth]{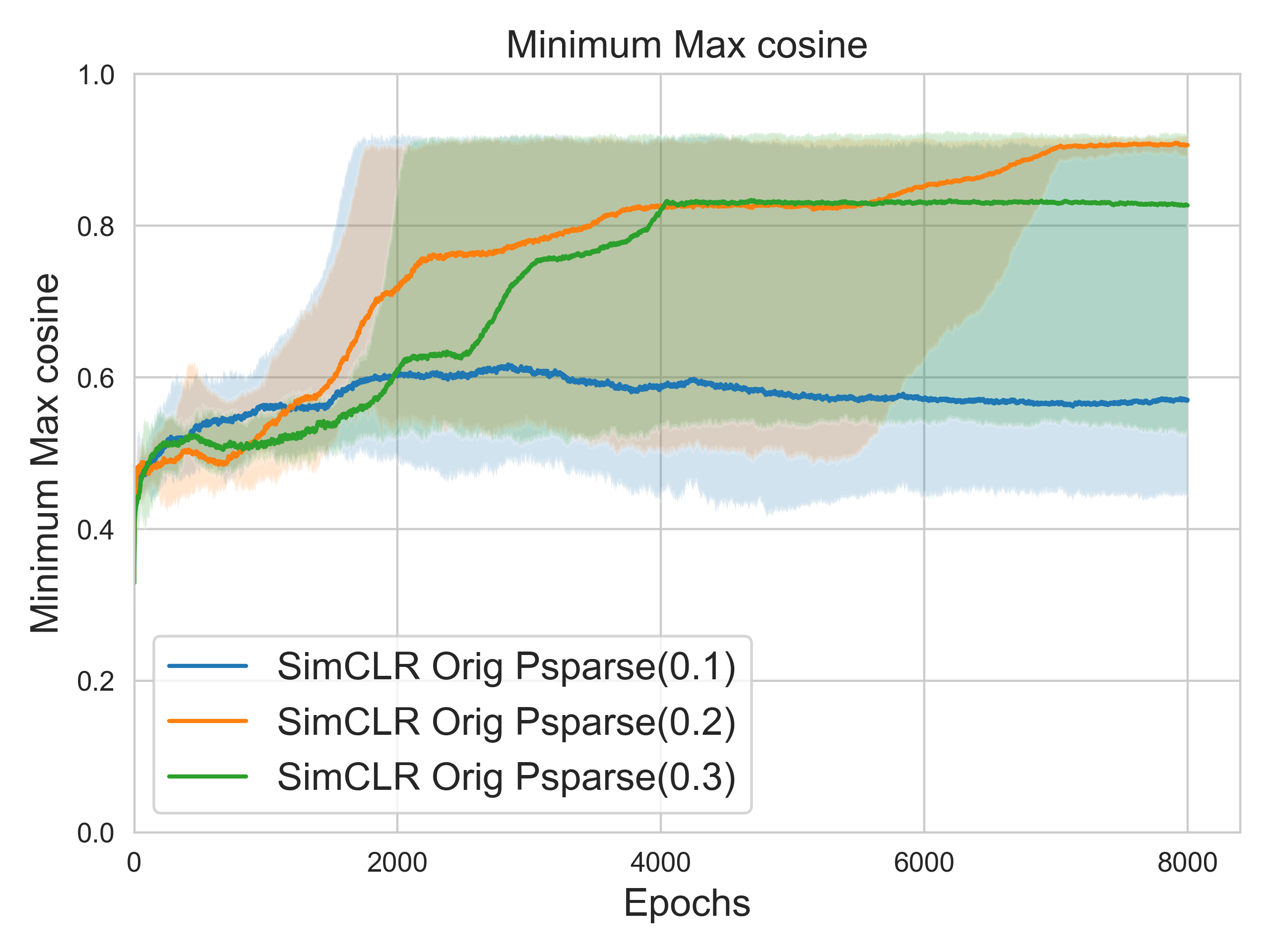}
  \end{minipage}
  \hfill
  \begin{minipage}[b]{0.23\textwidth}
    \includegraphics[width=\textwidth]{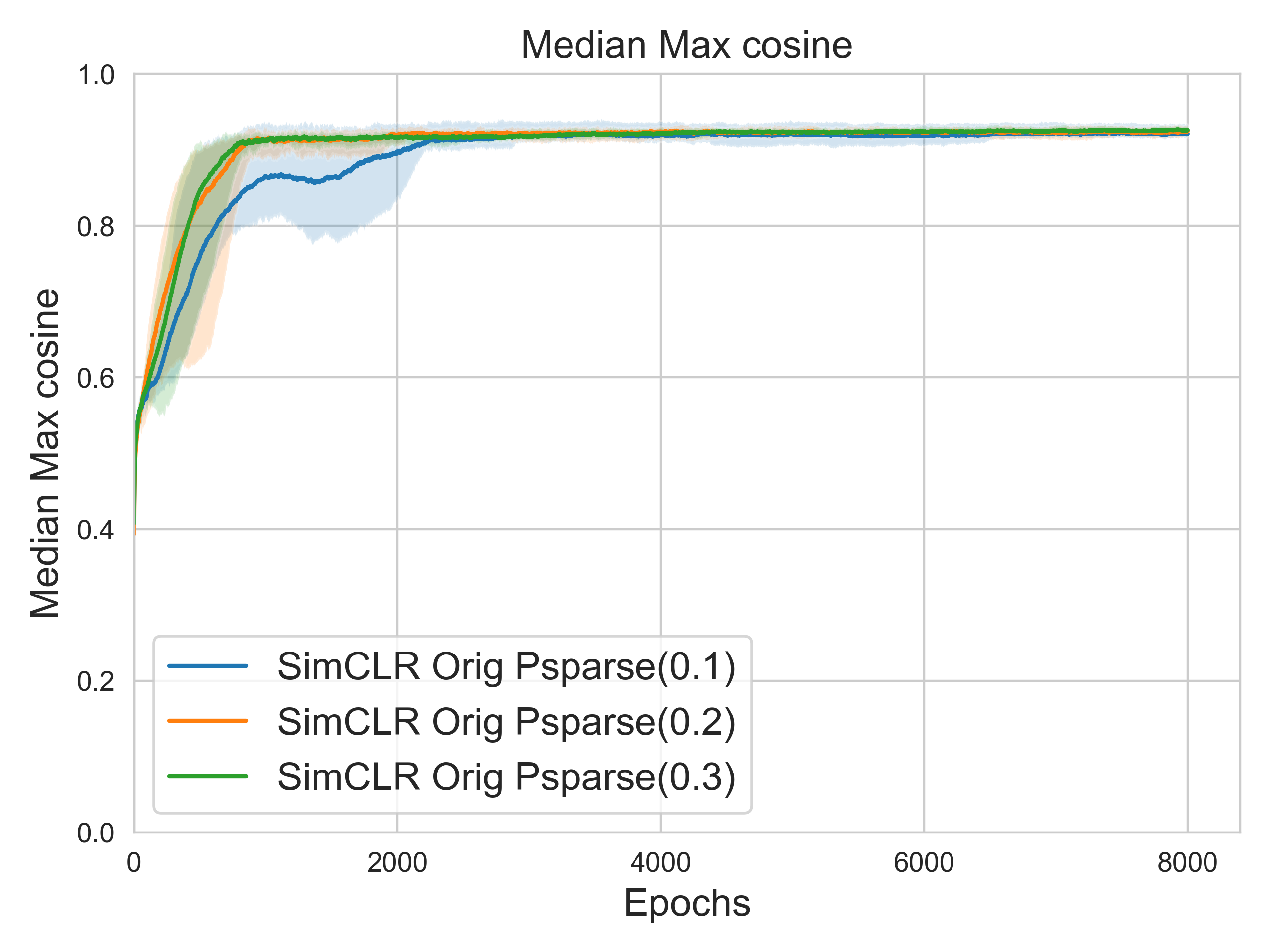}
  \end{minipage}
   \hfill
  \begin{minipage}[b]{0.23\textwidth}
    \includegraphics[width=\textwidth]{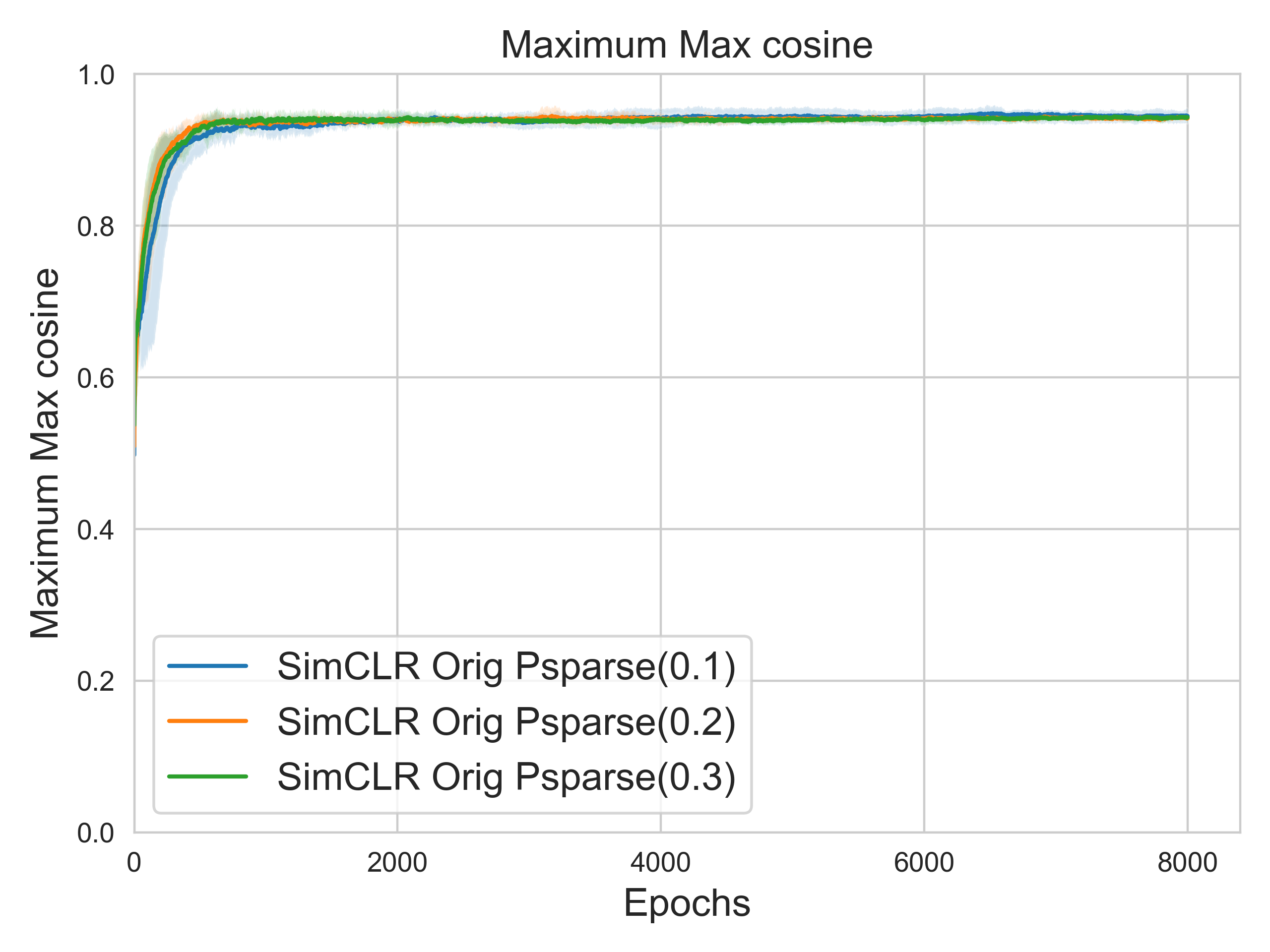}
  \end{minipage}
   \hfill
  \begin{minipage}[b]{0.23\textwidth}
    \includegraphics[width=\textwidth]{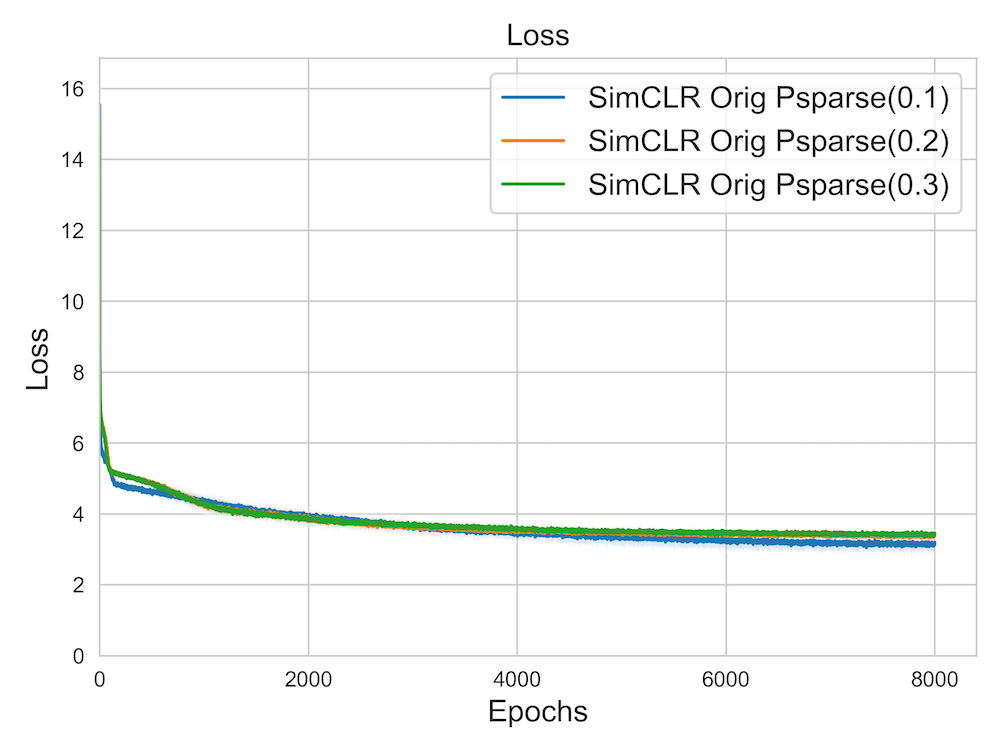}
  \end{minipage}
    \\~\\
  \caption{(Simplified SimCLR with random initialization learns good representations when the latents have low sparsity) (left-to-right) Minimum Max-Cosine,  Median Max-Cosine, Maximum Max-Cosine, and Loss curves for Simplified SimCLR model discussed in Table \ref{tab:table-simsiam-simclr}. Reported numbers are averaged over 5 different runs. The shaded area represents the maximum and the minimum values observed across those 5 runs. We use p = 50, d=10, m=50. Probability of random masking is 0.1. }
\end{figure}

\begin{figure}[!h]
  \centering
  \begin{minipage}[b]{0.23\textwidth}
    \includegraphics[width=\textwidth]{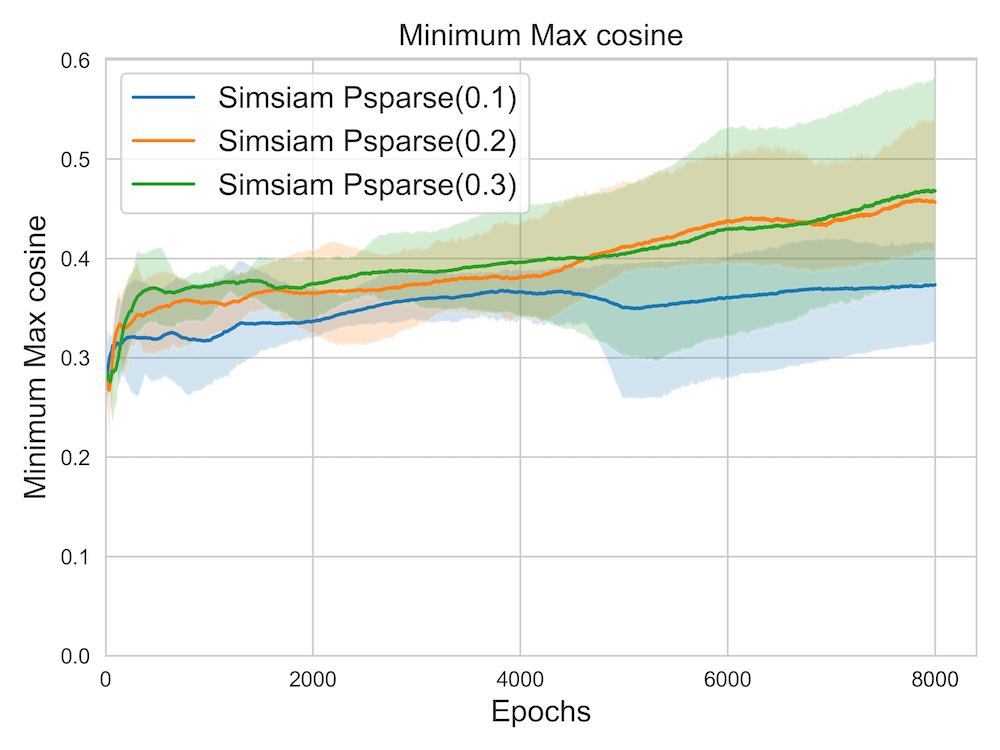}
  \end{minipage}
  \hfill
  \begin{minipage}[b]{0.23\textwidth}
    \includegraphics[width=\textwidth]{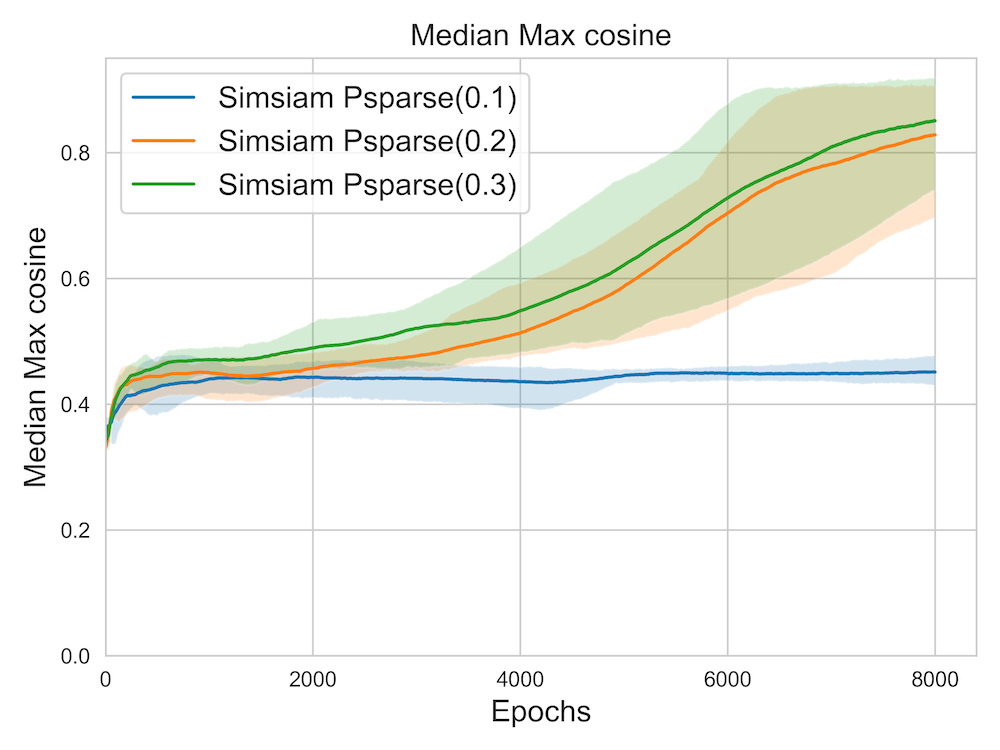}
  \end{minipage}
   \hfill
  \begin{minipage}[b]{0.23\textwidth}
    \includegraphics[width=\textwidth]{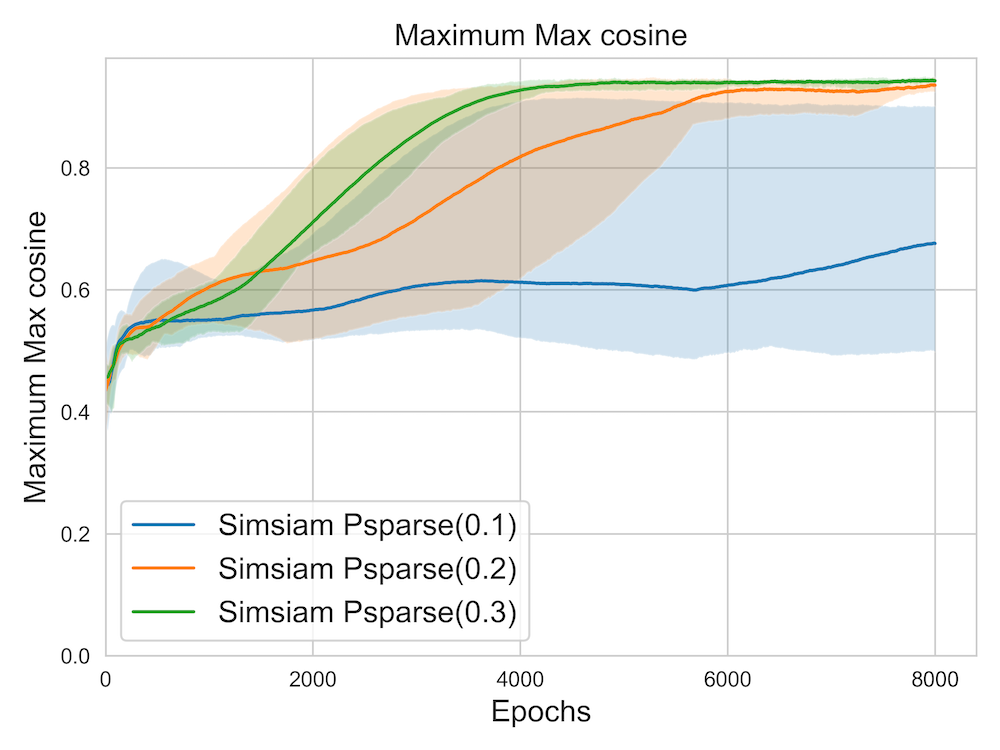}
  \end{minipage}
   \hfill
  \begin{minipage}[b]{0.23\textwidth}
    \includegraphics[width=\textwidth]{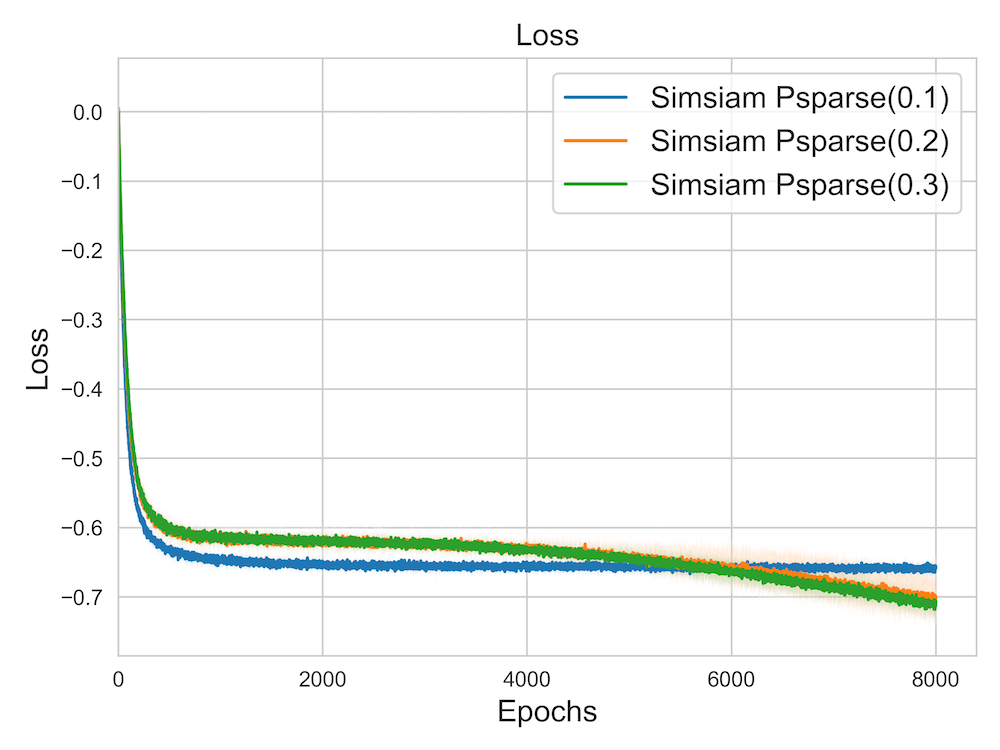}
  \end{minipage}
    \\~\\
  \caption{(Simplified Simsiam with random initialization fails to learn good representations) (left-to-right) Minimum Max-Cosine,  Median Max-Cosine, Maximum Max-Cosine, and Loss curves for Simplified Simsiam model discussed in Table \ref{tab:table-simsiam-simclr}. Reported numbers are averaged over 5 different runs. The shaded area represents the maximum and the minimum values observed across those 5 runs. We use p = 50, m=50, d=10. Probability of random masking is 0.1.}
\end{figure}

\subsubsection{Additional results for different levels of warm start}

In our experiments we find that initializing the weights of the encoder $\bm{W}$ close to the dictionary matrix $\bm{M}$ helps the non-contrastive learning algorithm to learn better cosine values. For warm start, we initialize the columns of the matrix $\bm{W}$ with random columns of the dictionary matrix $\bm{M}$. In addition, we add Gaussian noise $\mathcal{N}(0, \sigma^2)I$ to the matrix $\bm{W}$ where the choice of $\sigma$ determines the closeness of the matrix $\bm{W}$ to the dictionary matrix. The results of these experiments have been summarized in Table \ref{tab:table-warm-start}. 
As expected, a warmer start (corresponding to smaller $\sigma$) helps in all settings of Pr(sparse) that we tried.

\begin{table}[h!]
\centering
\begin{tabular}{cccccc}
\toprule
Model & Pr(sparse)& $\sigma$ & \multicolumn{1}{p{2.5cm}}{\centering Maximum \\ Max cosine $\uparrow$} & \multicolumn{1}{p{2.5cm}}{\centering Median \\ Max cosine $\uparrow$ }  & \multicolumn{1}{p{2.5cm}}{\centering Minimum \\ Max cosine $\uparrow$}  \\
\midrule
NCL-basic & 0.1 & 0.14 &$0.84 \pm 0.02$ & $0.73 \pm 0.04$ & $0.58 \pm 0.04$ \\
NCL-basic & 0.1 & 0.02 &$0.97 \pm 0.005$ & $0.77 \pm 0.01$ & $0.69 \pm 0.04$ \\
NCL-basic & 0.1 & 0.002 &$0.98 \pm 0.003$ & $0.78 \pm 0.04 $ & $0.69 \pm 0.02$ \\
\midrule
NCL-basic & 0.2 & 0.14  &$0.84 \pm 0.04$ & $0.73 \pm 0.05$ & $0.57 \pm 0.05$ \\
NCL-basic & 0.2 & 0.02 & $0.97 \pm 0.006$ & $0.79 \pm 0.07$ & $0.69 \pm 0.03$ \\
NCL-basic & 0.2 & 0.002  & $0.97 \pm 0.006$ & $0.81 \pm 0.08$ & $0.67 \pm 0.03$ \\
\midrule
NCL-basic & 0.3 & 0.14  &$0.82 \pm 0.02$ & $0.71 \pm 0.03$ & $0.58 \pm 0.03$ \\
NCL-basic & 0.3 & 0.02 & $0.95 \pm 0.009$ & $0.77 \pm 0.06$ & $0.66 \pm 0.01$ \\
NCL-basic & 0.3 & 0.002  & $0.96 \pm 0.02$ & $0.75 \pm 0.04$ & $0.66 \pm 0.02$ \\
\bottomrule
\end{tabular}
\caption{\label{tab:table-warm-start} Summary of cosine values learnt by the simple linear encoder by non contrastive algorithm with an overparameterized linear encoder (with batch normalization and symmetric ReLU) when the weights of the encoder $\bm{W}$ are initialized close to the ground truth dictionary $\bm{M}$. Pr(sparse) indicates the probability $Pr(\bm{z}_i = \pm 1), i \in [d]$ in the sparse coding vector $\bm{z}$. $\sigma$ denotes the std. deviation of the Gaussian noise added to  $\bm{W}$ that is initialized with random columns of the ground truth dictionary matrix $\bm{M}$. We report mean $\pm$ std. deviation over 5 runs.} 
\end{table}

%% file: sections/appendices/linear_network.tex
\subsection{Proof of the limitation of linear networks}
\label{sec:proof:linear}

\begin{theorem*}[non-contrastive loss on linear network, Theorem~\ref{thm:linear_network} restated]
    Suppose that the data generating process is specified as in Section~\ref{sec:setup} with the latent distribution changed to symmetric (Assumption~\ref{assumption:symmetric_bernoulli_latent}), 
    and the weights are initialized to $\mW^o_0$ and $\mW^t_0$, respectively.
    In step $t$, denote the learning rate as $\eta_t < \frac{\lambda}{2 \alpha^2 (2 \frac{\kappa}{2} + \sigma_0^2)}, \forall t$.
    Then, running gradient descent on the loss $L_{\text{linear-non-CL}}(\mW^o_t, \mW^t_t)$ in \eqref{eq:linear_network_weight_decay}
    will lead to
    \begin{align*}
        \mW^o_t \mM &= C_{1,t} \mW^o_0 \mM + C_{2,t} (\mW^o_0 \mM + \mW^t_0 \mM) \\
        \mW^t_t \mM &= C_{1,t} \mW^t_0 \mM + C_{2,t} (\mW^o_0 \mM + \mW^t_0 \mM)
    \end{align*}
    for some scalars $C_{1,t} := \prod_{i=0}^{t-1}(\lambda-c_i) \in (0, 1)$, and $C_{2,t} := \sum_{j=0}^{t-1} \left( c_i \prod_{i=j}^{t-1}(\lambda-c_i) \prod_{i=0}^{j-1}(\lambda+c_i) \right) > 0$,
    in which $c_i = 2 \eta_i \alpha^2 (2 p_z + \sigma_0^2) \in (0, \lambda)$.
\end{theorem*}

\begin{proof}
    Differentiating \eqref{eq:linear_network_weight_decay} gives 
    \begin{align*}
        \nabla_{\mW^o} L_{\text{linear-non-CL}}(\mW^o, \mW^t) &= -2 \mathbb{E}_{\vx, \mD_1, \mD_2} [(\mW^t \mD_2 \vx) (\mD_1 \vx)^\top] + (1-\lambda) \mW^o 
    \end{align*}
    By the setup in Section~\ref{sec:setup} and the latent distribution Assumption~\ref{assumption:symmetric_bernoulli_latent}, this expands to
    \[ -2 \alpha^2 \mW^t \left( \mM \mathbb{E}_{\vz} [ \vz \vz^\top ] \mM^\top + \mathbb{E}_{\vz, \bm\epsilon} [ \bm\epsilon \vz^\top \mM^\top ] + \mathbb{E}_{\vz, \bm\epsilon} [ \mM \vz \bm\epsilon^\top ] + \mathbb{E}_{\bm\epsilon} [ \bm\epsilon \bm\epsilon^\top ]  \right) + (1-\lambda) \mW^o \]
    By Assumption~\ref{assumption:symmetric_bernoulli_latent}
    and $\bm{\epsilon} \sim \mathcal{N}(0, \sigma_0^2 \mathbf{I}_p)$,
    we get
    \begin{align*}
        \mathbb{E}_{\vz} [ \vz \vz^\top ] &= 2 \frac{\kappa}{2} \mI \\
        \mathbb{E}_{\vz, \bm\epsilon} [ \bm\epsilon \vz^\top \mM^\top ] &= \bm 0 \\
        \mathbb{E}_{\vz, \bm\epsilon} [ \mM \vz \bm\epsilon^\top ] &= \bm 0 \\
        \mathbb{E}_{\bm\epsilon} [ \bm\epsilon \bm\epsilon^\top ] &= \sigma_0^2 \mathbf{I}_p
    \end{align*}
    Therefore
    \begin{align*}
        \nabla_{\mW^o} L_{\text{linear-non-CL}}(\mW^o, \mW^t) &= -2 \alpha^2 \mW^t \left( 2 \frac{\kappa}{2} \mM \mM^\top + \sigma_0^2 \mathbf{I}_p \right) + (1-\lambda) \mW^o \\
        &= -2 \alpha^2 (2 \frac{\kappa}{2} \mW^t \mM \mM^\top + \sigma_0^2 \mW^t) + (1-\lambda) \mW^o
    \end{align*}
    
    Then
    \begin{align*}
        \mW^o_{t+1} &= \mW^o_t - \eta_t \nabla_{\mW^o_t} L_{\text{linear-non-CL}}(\mW^o_t, \mW^t_t) \\
        &= \lambda \mW^o_t - \eta_t \left( -2 \alpha^2 (2 \frac{\kappa}{2} \mW^t_t \mM \mM^\top + \sigma_0^2 \mW^t_t) \right) \\
        &= \lambda \mW^o_t + 2 \eta_t \alpha^2 (2 \frac{\kappa}{2} \mW^t_t \mM \mM^\top + \sigma_0^2 \mW^t_t) \\
    \end{align*}
    Therefore 
    \begin{align*}
        \mW^o_{t+1} \mM &= \lambda \mW^o_t \mM + 2 \eta_t \alpha^2 (2 \frac{\kappa}{2} \mW^t_t \mM \mM^\top \mM + \sigma_0^2 \mW^t_t \mM) \\
        &= \lambda \mW^o_t \mM + 2 \eta_t \alpha^2 (2 \frac{\kappa}{2} \mW^t_t \mM + \sigma_0^2 \mW^t_t \mM) \\
        &= \lambda \mW^o_t \mM + 2 \eta_t \alpha^2 (2 \frac{\kappa}{2} + \sigma_0^2) \mW^t_t \mM 
    \end{align*}
    
    Since $L_{\text{non-contrastive}}(\mW^o, \mW^t)$ is symmetric in $\mW^o$ and $\mW^t$, we also have
    \[ \mW^t_{t+1} \mM = \lambda \mW^t_t \mM + 2 \eta_t \alpha^2 (2 \frac{\kappa}{2} + \sigma_0^2) \mW^o_t \mM \]
    Adding the above two equations gives
    \begin{align*}
        \mW^o_{t+1} \mM + \mW^t_{t+1} \mM &= \lambda (\mW^o_t \mM + \mW^t_t \mM) + 2 \eta_t \alpha^2 (2 \frac{\kappa}{2} + \sigma_0^2) (\mW^o_t \mM + \mW^t_t \mM) \\
        &= \left( \lambda + 2 \eta_t \alpha^2 (2 \frac{\kappa}{2} + \sigma_0^2) \right)  (\mW^o_t \mM + \mW^t_t \mM)
    \end{align*}
    Let constant $c_t = 2 \eta_t \alpha^2 (2 \frac{\kappa}{2} + \sigma_0^2)$,
    then by recursion 
    \[ \mW^o_t \mM + \mW^t_t \mM = \prod_{i=0}^{t-1}(\lambda+c_i) (\mW^o_0 \mM + \mW^t_0 \mM) \]
    Plugging into the expression for $\mW^o_{t+1} \mM$ above, we get
    \begin{align}
        \label{eq:Wot_M_recurrence}
        \mW^o_{t+1} \mM &= \lambda \mW^o_t \mM + 2 \eta_t \alpha^2 (2 \frac{\kappa}{2} + \sigma_0^2) (\mW^o_t \mM + \mW^t_t \mM - \mW^o_t \mM) \nonumber \\
        &= \lambda \mW^o_t \mM + c_t \left(\prod_{i=0}^{t-1}(\lambda+c_i) (\mW^o_0 \mM + \mW^t_0 \mM) - \mW^o_t \mM \right) \nonumber \\
        &= (\lambda -c_t) \mW^o_t \mM + c_t \prod_{i=0}^{t-1}(\lambda+c_i) (\mW^o_0 \mM + \mW^t_0 \mM) \\
    \end{align}
    
    From here, we will prove by induction that
    \begin{equation}
        \label{eq:Wot_M_IH}
        \mW^o_t \mM = \prod_{i=0}^{t-1}(\lambda-c_i) \mW^o_0 \mM + \sum_{j=0}^{t-1} \left( c_j \prod_{i=j+1}^{t-1}(\lambda-c_i) \prod_{i=0}^{j-1}(\lambda+c_i) \right) (\mW^o_0 \mM + \mW^t_0 \mM) 
    \end{equation}
    
    \underline{Base case}: when $t=0$, $\mW^o_0 \mM = \mW^o_0 \mM$ trivially satisfies \eqref{eq:Wot_M_IH}.
    
    \underline{Inductive step}: suppose \eqref{eq:Wot_M_IH} holds for $t$, and we will calculate $\mW^o_{t+1} \mM$.
    
    By \eqref{eq:Wot_M_recurrence} and \eqref{eq:Wot_M_IH}, we get
    \begin{align*}
        \mW^o_{t+1} \mM &= 
        (\lambda -c_t) \left( \prod_{i=0}^{t-1}(\lambda-c_i) \mW^o_0 \mM + \sum_{j=0}^{t-1} \left( c_j \prod_{i=j+1}^{t-1}(\lambda-c_i) \prod_{i=0}^{j-1}(\lambda+c_i) \right) (\mW^o_0 \mM + \mW^t_0 \mM) \right) \\
        &\quad + c_t \prod_{i=0}^{t-1}(\lambda+c_i) (\mW^o_0 \mM + \mW^t_0 \mM) \\
        &= \prod_{i=0}^{t}(\lambda-c_i) \mW^o_0 \mM + \sum_{j=0}^{t-1} \left( c_j \prod_{i=j+1}^{t}(\lambda-c_i) \prod_{i=0}^{j-1}(\lambda+c_i) \right) (\mW^o_0 \mM + \mW^t_0 \mM) \\
        &\quad + c_t \prod_{i=t+1}^{t}(\lambda-c_i) \prod_{i=0}^{t-1}(\lambda+c_i) (\mW^o_0 \mM + \mW^t_0 \mM) \\
        &= \prod_{i=0}^{t}(\lambda-c_i) \mW^o_0 \mM + \sum_{j=0}^{t} \left( c_j \prod_{i=j+1}^{t}(\lambda-c_i) \prod_{i=0}^{j-1}(\lambda+c_i) \right) (\mW^o_0 \mM + \mW^t_0 \mM) 
    \end{align*}
    implying that \eqref{eq:Wot_M_IH} holds for $t+1$.
    Therefore, by induction, \eqref{eq:Wot_M_IH} holds for $t \in \Z_+$.
    
    Since masking Bernoulli parameter $\alpha \in (0, 1)$, sparsity parameter $\frac{\kappa}{2} \in (0, 0.5)$, 
    and $\eta_t < \frac{\lambda}{2 \alpha^2 (2 \frac{\kappa}{2} + \sigma_0^2)}$,
    these conditions result in $c_t = 2 \eta_t \alpha^2 (2 \frac{\kappa}{2} + \sigma_0^2) \in (0,\lambda), \forall t$.
    Therefore, denote $C_{1,t} := \prod_{i=0}^{t-1}(\lambda-c_i) \in (0, 1)$. 
    
    \noindent
    Therefore we have showed that 
    \[ \mW^o_t \mM = C_{1,t} \mW^o_0 \mM + C_{2,t} (\mW^o_0 \mM + \mW^t_0 \mM) \]
    By symmetry,
    \[ \mW^t_t \mM = C_{1,t} \mW^t_0 \mM + C_{2,t} (\mW^o_0 \mM + \mW^t_0 \mM) \]
    for some $C_{1,t} \in (0, 1), C_{2,t} > 0$.
\end{proof}